%% file: main_arxiv.tex
\documentclass[10pt]{article}
\usepackage[margin=1in]{geometry}
\usepackage{parskip}

\usepackage[round,compress]{natbib}

\usepackage[toc,page,header]{appendix}
\usepackage{minitoc}
\usepackage{yub}
\usepackage{enumitem}

\allowdisplaybreaks

\usepackage{hyperref}
\hypersetup{
    colorlinks,
    linkcolor={blue!50!black},
    citecolor={blue!50!black},
}
\colorlet{linkequation}{blue}

\def\shownotes{0}  
\ifnum\shownotes=1
\newcommand{\authnote}[2]{{\scriptsize $\ll$\textsf{#1 notes: #2}$\gg$}}
\newcommand{\ziang}[1]{\textcolor{blue}{\small{[Ziang notes: #1]}}}
\newcommand{\sm}[1]{\textcolor{purple}{\small{[Song notes: #1]}}}
\else
\newcommand{\authnote}[2]{}
\newcommand{\ziang}[1]{}
\newcommand{\sm}[1]{}
\fi

\input{macros.tex}

\title{Sample-Efficient Learning of Correlated Equilibria in Extensive-Form Games}

\author{
  Ziang Song\thanks{Peking University. Email: {\tt songziang@pku.edu.cn}}
  \and
  Song Mei\thanks{UC Berkeley. Email: {\tt songmei@berkeley.edu}}
  \and
  Yu Bai\thanks{Salesforce Research. Email: {\tt yu.bai@salesforce.com}}
}

\begin{document}

\maketitle


\input{new_version/abstract.tex}

\input{new_version/intro.tex}
\input{new_version/related-work.tex}
\input{new_version/prelim.tex}

\input{new_version/def.tex}
\input{new_version/full-feedback.tex}
\input{new_version/bandit-feedback.tex}

\input{new_version/conclusion.tex}


\bibliographystyle{plainnat}
\bibliography{bib}

\clearpage
\appendix

\makeatletter
\def\renewtheorem#1{%
  \expandafter\let\csname#1\endcsname\relax
  \expandafter\let\csname c@#1\endcsname\relax
  \gdef\renewtheorem@envname{#1}
  \renewtheorem@secpar
}
\def\renewtheorem@secpar{\@ifnextchar[{\renewtheorem@numberedlike}{\renewtheorem@nonumberedlike}}
\def\renewtheorem@numberedlike[#1]#2{\newtheorem{\renewtheorem@envname}[#1]{#2}}
\def\renewtheorem@nonumberedlike#1{  
\def\renewtheorem@caption{#1}
\edef\renewtheorem@nowithin{\noexpand\newtheorem{\renewtheorem@envname}{\renewtheorem@caption}}
\renewtheorem@thirdpar
}
\def\renewtheorem@thirdpar{\@ifnextchar[{\renewtheorem@within}{\renewtheorem@nowithin}}
\def\renewtheorem@within[#1]{\renewtheorem@nowithin[#1]}
\makeatother

\renewtheorem{theorem}{Theorem}[section]

\doparttoc 
\faketableofcontents 

\part{} 
\parttoc 
\part{Appendix} 

\parttoc 

\input{new_version/tools.tex}

\input{new_version/time_selection}
\input{new_version/properties.tex}

\input{new_version/relation-proof.tex}

\input{new_version/regret-decomposition}

\input{new_version/proof-fullfeedback}
\input{new_version/proof-banditfeedback}

\end{document}

%% file: macros.tex
\newcommand{\ith}{$i^{\text{th}}$ }

\newcommand{\cO}{\mc{O}}
\newcommand{\tO}{\wt{\mc{O}}}

\renewcommand{\setto}{\leftarrow}

\renewcommand{\epsilon}{\varepsilon}
\renewcommand{\hat}{\what}
\renewcommand{\circ}{\diamond}

\newcommand{\ext}{{\rm ext}}
\renewcommand{\int}{{\rm int}}
\newcommand{\swap}{{\rm swap}}
\newcommand{\regmin}{\mc{R}}

\newcommand{\recommend}{{\textsc{Recommend}}}
\newcommand{\observeTSF}{{\textsc{Observe\_Time\_Selection}}}
\newcommand{\observeLoss}{{\textsc{Observe\_Loss}}}

\newcommand{\pure}{{\rm pure}}

\renewcommand{\fill}{{\sf fill}}
\newcommand{\ZZ}{\mathbb{Z}}

\newcommand{\efce}{{{\rm EFCE}}}
\newcommand{\nfce}{{{\rm NFCE}}}

\newcommand{\oneefce}{{1\mhyphen {\rm EFCE}}}

\newcommand{\zeroefce}{{0\mhyphen {\rm EFCE}}}
\newcommand{\kce}{{K\mhyphen {\rm EFCE}}}
\newcommand{\kefce}{{K\mhyphen {\rm EFCE}}}

\newcommand{\kpefce}{{(K+1)\mhyphen {\rm EFCE}}}
\newcommand{\inftyce}{{\infty\mhyphen {\rm EFCE}}}
\newcommand{\inftyefce}{{\infty\mhyphen {\rm EFCE}}}

\newcommand{\nfcce}{{\rm NFCCE}}

\newcommand{\zeroefcegap}{{0\mhyphen {\rm EFCEGap}}}
\newcommand{\oneefcegap}{{1\mhyphen {\rm EFCEGap}}}
\newcommand{\kefcegap}{{K\mhyphen {\rm EFCEGap}}}
\newcommand{\kpefcegap}{{(K+1)\mhyphen {\rm EFCEGap}}}
\newcommand{\Hefcegap}{{H\mhyphen {\rm EFCEGap}}}
\newcommand{\inftyefcegap}{{\infty\mhyphen {\rm EFCEGap}}}
\newcommand{\nfcegap}{{{\rm NFCEGap}}}
\newcommand{\nfccegap}{{{\rm NFCCEGap}}}

\newcommand{\triggergap}{{{\rm TriggerGap}}}

\def\triggerpolicy{{\sf trig}(\pi_i, \hat\pi_i, (x_i, a))}

\def\infoseq{{rechistory}}
\def\infoseqs{{rechistories}}
\def\rechist{{rechistory}}
\def\rechists{{rechistories}}

\newcommand{\kcfr}{{$K$-{\rm EFR}}}
\newcommand{\kscfr}{{$K$-{\rm EFR}}}
\newcommand{\kefr}{{$K$-{\rm EFR}}}
\newcommand{\balkefr}{{{\rm Balanced} $K$-{\rm EFR}}}
\newcommand{\wrhedge}{\textsc{WRHedge}}
\newcommand{\WRHEDGE}{\textsc{WRHedge}}
\newcommand{\wrhedgebandit}{\textsc{SWRHedge}}
\newcommand{\swrhedge}{\textsc{SWRHedge}}

\def\omegaih{{\Omega_{i,h}^{({\rm I}), K}}}
\def\omegai{{\Omega_{i}^{({\rm I}), K}}}
\def\omegaiih{{\Omega_{i,(h_\star,h)}^{({\rm II}), K}}}
\def\omegaii{{\Omega_{i}^{({\rm II}), K}}}

\def\muphi{\phi}

\newcommand{\proofsize}{\footnotesize}

\def\POMG{{\rm POMG}}
\def\cS{{\mathcal S}}
\def\cX{{\mathcal X}}
\def\cA{{\mathcal A}}
\def\cF{{\mathcal F}}
\def\ba{{\mathbf{a}}}
\def\bbb{{\mathbf{b}}}

\def\l{\ell}
\def\L{L}

\def\piL{\pi L}
\def\dev{{\mathsf{D}}}

\renewcommand{\equiv}{\defeq}

\def\indicDleK{\delta^{\le K-1}(a_{1:h-1},b_{1:h-1})}
\def\indicDeqK{\delta^{K}({a}_{1:h-1},{b}_{1:h-1})}

\def\cC{{\mathcal C}}

\def\cB{{\mathcal B}}

\def\cI{{\mathcal I}}

\def\wh{{W(h)}}
\def\Regalg{{\textsc{RegAlg}}}

\renewcommand{\bar}{\wb}

\mathchardef\mhyphen="2D

%% file: new_version/abstract.tex
\begin{abstract}
  Imperfect-Information Extensive-Form Games (IIEFGs) is a prevalent model for real-world games involving imperfect information and sequential plays. The Extensive-Form Correlated Equilibrium (EFCE) has been proposed as a natural solution concept for multi-player general-sum IIEFGs. However, existing algorithms for finding an EFCE require full feedback from the game, and it remains open how to efficiently learn the EFCE in the more challenging bandit feedback setting where the game can only be learned by observations from repeated playing.

  This paper presents the first sample-efficient algorithm for learning the EFCE from bandit feedback. We begin by proposing $K$-EFCE---a more generalized definition that allows players to observe and deviate from the recommended actions for $K$ times. The $K$-EFCE includes the EFCE as a special case at $K=1$, and is an increasingly stricter notion of equilibrium as $K$ increases. We then design an uncoupled no-regret algorithm that finds an $\varepsilon$-approximate $K$-EFCE within $\widetilde{\mathcal{O}}(\max_{i}X_iA_i^{K}/\varepsilon^2)$ iterations in the full feedback setting, where $X_i$ and $A_i$ are the number of information sets and actions for the $i$-th player. Our algorithm works by minimizing a wide-range regret at each information set that takes into account all possible recommendation histories. Finally, we design a sample-based variant of our algorithm that learns an $\varepsilon$-approximate $K$-EFCE within $\widetilde{\mathcal{O}}(\max_{i}X_iA_i^{K+1}/\varepsilon^2)$ episodes of play in the bandit feedback setting. When specialized to $K=1$, this gives the first sample-efficient algorithm for learning EFCE from bandit feedback.







\end{abstract}

%% file: new_version/intro.tex
\section{Introduction}



This paper is concerned with the problem of learning equilibria in Imperfect-Information Extensive-Form Games (IIEFGs)~\citep{kuhn201611}. IIEFGs is a general formulation for multi-player games with both imperfect information (such as private information) and sequential play, and has been used for modeling and solving real-world games such as Poker~\citep{heinrich2015fictitious,moravvcik2017deepstack,brown2018superhuman,brown2019superhuman}, Bridge~\citep{tian2020joint}, Scotland Yard~\citep{schmid2021player}, and so on. In a two-player zero-sum IIEFG, the standard solution concept is the celebrated notion of Nash Equilibrium (NE)~\citep{nash1950equilibrium}, that is, a pair of independent policies for both players such that no player can gain by deviating from her own policy. However, in multi-player general-sum IIEFGs, computing an (approximate) NE is PPAD-hard and unlikely to admit efficient algorithms~\citep{daskalakis2009complexity}. A more amenable class of solution concepts is the notion of \emph{correlated equilibria}~\citep{aumann1974subjectivity}, that is, a correlated policy for all players such that no player can gain by deviating from the correlated play using certain types of deviations.

The notion of Extensive-Form Correlated Equilibria (EFCE) proposed by~\citet{von2008extensive} is a natural definition of correlated equilibria in multi-player general-sum IIEFGs. An EFCE is a correlated policy that can be thought of as a ``mediator'' of the game who recommends actions to each player privately and sequentially (at visited information sets), in a way that disincentivizes any player to deviate from the recommended actions. Polynomial-time algorithms for computing EFCEs have been established, by formulating as a linear program and using the ellipsoid method~\citep{huang2008computing,papadimitriou2008computing,jiang2015polynomial}, min-max optimization~\citep{farina2019correlation}, or uncoupled no-regret dynamics using variants of counterfactual regret minimization~\citep{celli2020no,farina2021simple,morrill2021efficient,anagnostides2021faster}.



However, all the above algorithms require that the full game is known (or \emph{full feedback}). In the more challenging \emph{bandit feedback} setting where the game can only be learned by observations from repeated playing, it remains open how to learn EFCEs sample-efficiently. This is in contrast to other types of equilibria such as NE in two-player zero-sum IIEFGs, where sample-efficient learning algorithms under bandit feedback are known~\citep{lanctot2009monte,farina2020stochastic,kozuno2021model,bai2022near}.
A related question is about the structure of the EFCE definition: An EFCE only allows players to deviate \emph{once} from the observed recommendations, upon which further recommendations are no longer revealed and the player needs to make decisions on her own instead. This may be too restrictive to model situations where players can still observe the recommendations after deviating~\citep{morrill2021efficient}. It is of interest how we can extend the EFCE definition in a structured fashion to disincentivize such stronger deviations while still allowing efficient algorithms.






This paper makes steps towards answering both questions above, by proposing stronger and more generalized definition of EFCEs, and designing efficient learning algorithms under both full-feedback and bandit-feedback settings. We consider IIEFGs with $m$ players, $H$ steps, where each player $i$ has $X_i$ information sets and $A_i$ actions. Our contributions can be summarized as follows.
\begin{itemize}[itemsep=0pt, topsep=5pt, leftmargin=12pt]
    \item We propose $\kefce$, a natural generalization of EFCE, at which players have no gains when allowed to observe and deviate from the recommended actions for $K$ times (Section~\ref{section:def}). At $K=1$, the $\kefce$ is equivalent to the existing definition of $\efce$ based on trigger policies. For $K\ge 1$, the $\kefce$ are increasingly stricter notions of equilibria as $K$ increases.
    \item We design an algorithm $K$ Extensive-Form Regret Minimization (\kefr) which finds an $\eps$-approximate $\kefce$ within $\tO(\max_i X_iA_i^{K\wedge H}/\eps^2)$ iterations under full feedback (Section~\ref{section:full-feedback}). At $K=1$, our linear in $\max_i X_i$ dependence improves over the best known $\tO(\max_i X_i^2)$ dependence for computing $\eps$-approximate EFCE. At $K>1$, this gives a sharp result for efficiently computing (the stricter) $\kefce$, improving over the best known $\tO(\max_i X_i^2A_i^{K\wedge H}/\eps^2)$ iteration complexity of~\citet{morrill2021efficient}.
    
    \item We further design Balanced \kefr---a sample-based variant of \kefr---for the more challenging bandit-feedback setting (Section~\ref{section:bandit-feedback}). Balanced \kefr~ learns an $\eps$-approximate $\kefce$ within $\tO(\max_i X_iA_i^{K\wedge H+1}/\eps^2)$ episodes of play. This is the first line of results for learning $\efce$ and $\kefce$ from bandit feedback, and the linear in $X_i$ dependence matches informaiton-theoretic lower bounds. 
    
    \item Technically, our bandit-feedback result builds on a novel stochastic wide range regret minimization algorithm \swrhedge, as well as sample-based estimators of counterfactual loss functions using newly designed sampling policies, both of which may be of independent interest. 
    
    
\end{itemize}



%% file: new_version/related-work.tex
\subsection{Related work}

\paragraph{Computing Correlated Equilibria from full feedback}



The notion of Extensive-Form Correlated Equilibria (EFCE) in IIEFGs is introduced in \citet{von2008extensive}. \citet{huang2008computing} design the first polynomial time algorithm for computing EFCEs in multi-player IIEFGs from full feedback, using a variation of the \textit{Ellipsoid against hope} algorithm~\citep{papadimitriou2008computing, jiang2015polynomial}. Later, \citet{farina2019correlation}~propose a min-max optimization formulation of EFCEs which can be solved by first-order methods. 

\citet{celli2020no} and its extended version~\citep{farina2021simple} design the first uncoupled no-regret algorithm for computing EFCEs. Their algorithms are based on minimizing the trigger regret (first considered in \citet{dudik2012sampling,gordon2008no}) via counterfactual regret decomposition~\citep{zinkevich2007regret}. \citet{morrill2021efficient} propose the stronger definition of ``Behavioral Correlated Equilibria'' (BCE) using general ``behavioral deviations'', and design the Extensive-Form Regret Minimization (EFR) algorithm to compute a BCE by using a generalized version of counterfactual regret decomposition. They also propose intermediate notions such as a ``Twice Informed Partial Sequence'' (TIPS) (and its $K$-shot generalization) as an interpolation between the EFCE and BCE. Our definition of $\kefce$ offers a new interpolation between the EFCE and BCE that is different from theirs, as the deviating player in $\kefce$ keeps following recommended actions until $K$ deviations has happened, whereas the deviator in $K$-shot Informed Partial Sequence keeps deviating until $K$ deviations has happened, after which she resumes to following.


The iteration complexity for computing an $\eps$-approximate correlated equilibrium in both~\citep{farina2021simple,morrill2021efficient} scales quadratically in $\max_{i\in[m]} X_i$. Our \kefr~algorithm for the full feedback setting builds upon the EFR algorithm, but specializes to the notion of $\kefce$, and achieve an improved linear in $\max_{i\in[m]} X_i$ iteration complexity.



Apart from the EFCE, there are other notions of (coarse) correlated equilibria in IIEFGs such as Normal-Form Coarse-Correlated Equilibria (NFCCE)~\citep{zinkevich2007regret,celli2019learning,burch2019revisiting,farina2020faster,zhou2020lazy}, Extensive-Form Coarse-Correlated Equilibria (EFCCE) \citep{farina2020coarse}, and Agent-Form (Coarse-)Correlated Equilibria (AF(C)CE)~\citep{selten1975reexamination,von2008extensive}; see~\citep{morrill2020hindsight, morrill2021efficient} for a detailed comparison. All above notions are either weaker than or incomparable with the EFCE, and thus results there do not imply results for computing EFCE.

\paragraph{Learning Equilibria from bandit feedback}
A line of work considers learning Nash Equilibria (NE) in two-player zero-sum IIEFGs and NFCCE in multi-player general-sum IIEFGs from bandit feedback~\citep{lanctot2009monte,farina2020stochastic, farina2021model, farina2021bandit,zhou2019posterior,zhang2021finding,kozuno2021model,bai2022near}; Note that the NFCCE is weaker than (and does not imply results for learning) EFCE. \citet{dudik2009sampling} consider sample-based learning of EFCE in succinct extensive-form games; however, their algorithm relies on an approximate Markov-Chain Monte-Carlo (MCMC) sampling subroutine that does not lead to an end-to-end sample complexity guarantee. To our best knowledge, our results are the first for learning the EFCE (and $\kefce$) under bandit feedback.

A related line of work considers learning (Coarse) Correlated Equilibria in general-sum Markov games without tree structure but with perfect information~\citep{liu2021sharp, song2021can, jin2021v, mao2022provably}, which does not imply our results for learning in IIEFGs as their setting does not contain imperfect information. 





%% file: new_version/prelim.tex
\section{Preliminaries}
\label{section:prelim}

We formulate IIEFGs as partially-observable Markov games (POMGs) with tree structure and perfect recall, following~\citep{kozuno2021model,bai2022near}. For a positive integer $i$, we denote by $[i]$ the set $\{1, 2, \cdots, i\}$. For a finite set $\cA$, we let $\Delta(\cA)$ denote the probability simplex over $\cA$. Let ${n \choose m}$ denote the binomial coefficient (i.e. number of combinations) of choosing $m$ elements from $n$ different elements, with the convention that ${n \choose m} = 0$ if $m > n$.


\paragraph{Partially observable Markov game} We consider an episodic, tabular, $m$-player, general-sum, partially observable Markov game
\[
\POMG(m, \cS, \{ \cX_i \}_{i \in [m]}, \{ \cA_i \}_{i \in [m]}, H, p_0, \{ p_h \}_{h \in [H]}, \{ r_{i, h} \}_{i \in [m], h \in [H]}),
\]
where $\cS$ is the state space of size $\abs{\cS}=S$; $\cX_i$ is the space of information sets (henceforth \emph{infosets}) for the \ith player, which is a partition of $\cS$ (i.e., $x_i \subseteq \cS$ for all $x_i \in \cX_i$, and $\cS = \sqcup_{x_i \in \cX_i} x_i$ where $\sqcup$ stands for disjoint union) with size $\abs{\cX_i}=X_i$, and we also use $x_i:\cS\to\cX_i$ to denote the \ith player's emission (observation) function; $\cA_i$ is the action space for the \ith player with size $\abs{\cA_i}=A_i$, and we let $\cA = \cA_1 \times \cdots \cA_m$ denote the space of joint actions $\ba=(a_1,\dots,a_m)$; $H \in \ZZ_{\ge 1}$ is the time horizon; $p_0 \in \Delta(\cS)$ is the distribution of the initial state $s_1$; $p_h: \cS \times \cA \to \Delta(\cS)$ are transition probabilities where $p_h(s_{h+1}|s_h, \ba_h)$ is the probability of transiting to the next state $s_{h+1}$ from state-action $(s_h, \ba_h)\in\cS\times\cA$; and $r_{i, h}: \cS \times \cA \to [0, 1]$ is the deterministic\footnote{Our results can be directly extended to the case of stochastic rewards.} reward function for the \ith player at step $h$.  


\paragraph{Tree structure and perfect recall assumption} 
We use a POMG with tree structure and perfect recall to formulate imperfect information games, following~\citep{kozuno2021model,bai2022near}. We assume that the game has a tree structure: for any state $s \in \cS$, there is a unique step $h$ and history $(s_1, \ba_1 , \ldots ,s_{h-1}, \ba_{h-1}, s_h = s)$ to reach $s$. Precisely, for any policy of the players, for any realization of the game (i.e., trajectory) $(s_k', \ba_k')_{k \in [H]}$, conditionally to $s_l'= s$, it almost surely holds that $l = h$ and $(s_1', \ldots , s_h') = (s_1, \ldots , s_h )$. 
We also assume perfect recall, which means that each player remembers its past information sets and actions. In particular, for each information set (infoset) $x_i \in \cX_i$ for the \ith player, there is a unique history of past infosets and actions $(x_{i, 1} , a_{i, 1} , \ldots , x_{i, h-1}, a_{i, h-1}, x_{i, h} = x_i)$ leading to $x_i$. This requires that $\cX_i$ can be partitioned to $H$ subsets $(\cX_{i, h})_{h \in [H]}$ such that $x_{i, h} \in \cX_{i, h}$ is reachable only at time step $h$. We define $X_{i,h}\defeq \abs{\cX_{i,h}}$. Similarly, the state set $\cS$ can be also partitioned into $H$ subsets $(S_h)_{h \in [H]}$. As we mostly focus on the \ith player, we use $x_h$ to also denote $x_{i,h}$ and use them interchangeably.




For $s \in \cS$ and $x_i \in \cX_i$, we write $s \in x_i$ if infoset $x_i$ contains the state $s$. With an abuse of notation, for $s \in \cS$, we let $x_i(s)$ denote the \ith player's infoset that $s$ belongs to. 
For any $h < h'$, $x_{i, h} \in \cX_{i, h}, x_{i, h'} \in \cX_{i, h'}$, we write $x_{i, h} \prec x_{i, h'}$ if the information set $x_{i, h'}$ can be reached from information set $x_{i, h}$ by \ith player's actions; we write $(x_{i, h}, a_{i, h}) \prec x_{i, h'}$ if the infoset $x_{i, h'}$ can be reached from infoset $x_{i, h}$ by \ith player's action $a_{i, h}$. For any $h < h'$ and $x_{i, h} \in \cX_{i, h}$, we let $\cC_{h'}(x_{i, h}, a_{i, h})\equiv \{ x \in \cX_{i, h'}: (x_{i,h}, a_{i, h}) \prec x\}$ and $\cC_{h'}(x_{i, h}) \equiv \{ x \in \cX_{i, h'}: x_{i,h} \prec x\} = \cup_{a_{i,h} \in \cA_i} \cC_{h'}(x_{i, h}, a_{i, h})$ denote the infosets within the $h'$-th step that are reachable from (i.e. childs of) $x_{i,h}$ or $(x_{i,h}, a_{i,h})$, respectively. For shorthand, let $\cC(x_{i,h}, a_{i, h}) \defeq \cC_{h+1}(x_{i,h}, a_{i,h})$ and $\cC(x_{i,h}) \defeq \cC_{h+1}(x_{i,h})$ denote the set of immediate childs.

\paragraph{Policies} 
We use $\pi_i=\set{\pi_{i, h}(\cdot|x_{i,h})}_{h \in [H], x_{i,h}\in\cX_{i,h}}$ to denote a policy of the \ith player, where each $\pi_{i,h}(\cdot|x_{i,h})\in\Delta(\cA_i)$ is the action distribution at infoset $x_{i,h}$. We say $\pi_i$ is a pure policy if $\pi_{i, h}(\cdot|x_{i,h})$ takes some single action deterministically for any $(h,x_{i,h})$; in this case we let $\pi_i(x_{i,h}) = \pi_{i,h}(x_{i,h})$ denote the action taken at infoset $x_{i,h}$ for shorthand. We use $\pi = \{ \pi_i\}_{i \in [m]}$ to denote a product policy for all players, and call $\pi$ a pure product policy if $\pi_i$ is a pure policy for all $i\in[m]$. Let $\Pi_i$ denote the set of all possible policies for the \ith player and $\Pi=\prod_{i\in[m]}\Pi_i$ denote the set of all possible product policies.  Any probability measure $\wb{\pi}$ on $\Pi$  induces a \emph{correlated policy}, which executes as first sampling a  product policy $\pi = \{\pi_i\}_{i \in [m]} \in\Pi$ from probability measure $\wb{\pi}$ and then playing the product policy $\pi$. We also use $\wb{\pi}$ to denote this policy. A correlated policy $\wb{\pi}$ can be viewed as a \emph{mediator} of the game which samples $\pi\sim \wb{\pi}$ before the game starts, and privately \emph{recommends} action sampled from $\pi_i(\cdot | x_i)$ to the \ith player when infoset $x_i\in\cX_i$ is visited during the game.






\paragraph{Reaching probability} 
With the tree structure assumption, for any state $s_h\in \cS_h$ and actions $\ba \in \cA$, there exists a unique history $(s_1, \ba_1, \ldots , s_h = s, \ba_h = \ba)$ ending with $(s_h =s,\ba_h = \ba)$. Given any product policy $\pi$, the probability of reaching $(s_h, \ba_h)$ at step $h$ can be decomposed as
\begin{align}
  \label{equation:reaching-prob}
  \textstyle
p^\pi_h(s_h, \ba) = p_{1:h}(s_h) \prod_{i \in [m]} \pi_{i, 1:h}(s_h, a_{i,h}),
\end{align}
where we have defined the \emph{sequence-form transitions} $p_{1:h}$ and \emph{sequence-form policies} $\pi_{i,1:h}$ as 
\begin{align}
  & \textstyle
  p_{1:h}(s_h) \defeq ~ p_0(s_1) \prod_{h' = 1}^{h-1} p_{h'}(s_{h'+1} \vert s_{h'}, \ba_{h'} ), \label{eqn:sequence-form-transition} \\
  & \textstyle \pi_{i, 1:h}(s_h, a_{i,h}) \defeq ~ \pi_{i, 1:h}(x_{i,h}, a_{i,h}) \defeq \prod_{h' = 1}^h \pi_{i,h'}(a_{i,h'} \vert x_{i,h'}), \label{eqn:sequence-form-policy}
\end{align}
where $(s_{h'}, \ba_{h'})_{h'\le h-1}$ is the unique history of states and actions that leads to $s_h$ by the tree structure; $x_{i,h}=x_i(s_{h})$ is the \ith player's infoset at the $h$-th step, and $(x_{i,h'}, a_{i,h'})_{h'\le h-1}$ is the unique history of infosets and actions that leads to $x_{i,h}$ by perfect recall. We also define $\pi_{i,h:h'}(x_{i,h'}, a_{i,h'})\defeq \prod_{h''=h}^{h'} \pi_{i,h''}(a_{i,h''}|x_{i,h''})$ for any $1\le h\le h'\le H$. 



\paragraph{Value functions and counterfactual loss functions}
Let $V^{\pi}_{i}\defeq \E_{\pi}[\sum_{h=1}^H r_{i,h}]$ denote the value function (i.e. expected cumulative reward) for the \ith player under policy $\pi$. By the product form of the reaching probability in~\eqref{equation:reaching-prob}, the value function $V^{\pi}_i$ admits a multi-linear structure over the sequence-form policies. Concretely, fixing any sequence of product policies $\{\pi^t\}_{t=1}^T$ where each $\pi^t= \{ \pi_i^t\}_{i\in[m]}$, we have
\begin{align*}
    V^{\pi^t}_i = \sum_{h=1}^H \sum_{(s_h, \ba_h = ( a_{j, h})_{j \in [m]})\in\mc{S}_h\times\mc{A}} p_{1:h}(s_h) \prod_{j=1}^m \pi^t_{j, 1:h}(x_{j}(s_h), a_{j,h}) r_{i, h}(s_h, \ba_h).
\end{align*}

For any sequence of policies $\{\pi^t\}_{t=1}^T$, we also define the \emph{counterfactual loss function} (for the $t$-th round) $\{L_{i,h}^t(x_{i,h}, a_{i,h})\}_{i,h,x_{i,h},a_{i,h}}$ as~\citep{zinkevich2007regret}: 
\begin{align}
    & \l_{i,h}^t(x_{i,h}, a_{i,h}) \defeq \sum_{\substack{s_h\in x_{i,h}, \\ \ba_{-i,h}\in\mc{A}_{-i}}} p_{1:h}(s_h) \prod_{j\neq i} \pi_{j, 1:h}^t(x_{j}(s_h), a_{j,h}) [1 - r_{i,h}(s_h, \ba_h)], \label{equation:l-definition} \\
    & \L_{i,h}^{t}(x_{i,h}, a_{i,h}) \defeq \l_{i,h}^t(x_{i,h}, a_{i,h}) + \sum_{h'=h+1}^H \sum_{\substack{x_{h'}\in \cC_{h'}(x_{i, h}, a_{i,h}), \\ a_{h'}\in\mc{A}_i}} \pi_{i,(h+1):h'}^t(x_{h'}, a_{h'}) \l_{i,h'}^t(x_{h'}, a_{h'}). \label{equation:L-definition}
\end{align}
Intuitively, $\L_{i,h}^{t}(x_{i,h}, a_{i,h})$ measures the \ith player's expected loss (one minus reward) conditioned on reaching $(x_{i,h}, a_{i,h})$, weighted by the (environment) transitions and all other players' policies $\pi_{-i}^t$ at all time steps, and the \ith player's own policy $\pi_i^t$ from step $h+1$ onward. We will omit the $i$ subscript and use $L_h^t$ to denote the above when clear from the context.


\paragraph{Feedback protocol}
We consider two standard feedback protocols for our algorithms: \emph{full feedback}, and \emph{bandit feedback}. In the full feedback case, the algorithm can query a product policy $\pi^t=\{ \pi_i^t\}_{i \in [m]}$ in each iteration and observe the counterfactual loss functions $\{\L_{i,h}^t(x_{i,h}, a_{i,h})\}_{i, h, x_{i,h}, a_{i,h}}$ \emph{exactly}\footnote{This is implementable (and slightly more general than) when the full game (transitions and rewards) is known.}. In the bandit feedback case, the players can only play repeated episodes with some policies and observe the transitions and rewards from the environment. Specifically, the $t$-th episode (overloading the $t$ notation briefly) proceeds as follows: Before the episode starts, each player chooses some policy $\pi_i^t\in\Pi_i$. At the beginning of an episode, an initial state $s^t_1\sim p_0(\cdot)$ is sampled. Then at each step $h\in[H]$, the \ith player observes her infoset $x^t_{i, h} \equiv x_i(s^t_h)$ and plays actions $a^t_{i, h} \sim \pi^t_{i, h}(\cdot|x^t_{i, h})$ for all $i\in[m]$. The underlying state then transitions to a next state $s^t_{h+1} \sim p_h(\cdot |s^t_h, \ba^t_h)$, and the \ith player receives reward $r^t_{i, h} \equiv r_{i, h}(s^t_h, \ba^t_h)$. The trajectory observed by the \ith player in this episode is $(x_{i,1}^t, a_{i,1}^t,r_{i,1}^t,\dots,x_{i,H}^t,  a_{i,H}^t, r_{i,H}^t)$.

%% file: new_version/def.tex
\section{$K$ Extensive-Form Correlated Equilibria}
\label{section:def}


We now introduce the definition of $K$ Extensive-Form Correlated Equilibria ($\kefce$) and establish its relationship with existing notions of correlated equilibria in IIEFGs.

\subsection{Definition of $\kefce$}



Roughly speaking, a $\kefce$ is a correlated policy in which no player can gain if allowed to deviate from the observed recommended actions $K$ times, and forced to choose her own actions without observing further recommendations afterwards.

To state its definition formally, we categorize all possible \emph{recommendation histories} (henceforth \emph{\rechists}) at each infoset $x_{i,h}\in\cX_{i,h}$ (for the \ith player) into two types, based on whether the player has already deviated $K$ times from past recommendations:
\begin{enumerate}[label=(\arabic*), leftmargin=20pt, topsep=5pt, itemsep=0pt]
\item A \emph{Type-I \rechist} ($\le K-1$ deviations happened) at $x_{i,h}$ is any action sequence $b_{1:h-1}\in\cA_i^{h-1}$ such that $\sum_{k=1}^{h-1} \indic{a_{k}\neq b_{k}}\le K-1$, where $(a_1,\dots,a_{h-1})$ is the unique sequence of actions leading to $x_{i,h}$. Let $\omegai(x_{i,h})$ denote the set of all Type-I \rechists~at $x_{i,h}$.
\item A \emph{Type-II \rechist} ($K$ deviations happened) at $x_{i,h}$ is any action sequence $b_{1:h'}\in\cA_i^{h'}$ with length $h' < h$ such that $\sum_{k=1}^{h'-1} \indic{a_{k}\neq b_{k}}=K-1$ and $a_{k}\neq b_{k}$, where $(a_1,\dots,a_{h-1})$ is the unique sequence of actions leading to $x_{i,h}$. Let $\omegaii(x_{i,h})$ denote the set of all Type-II \rechists~at $x_{i,h}$.
\end{enumerate}



We now define a $\kefce$ strategy modification ($0\le K\le \infty$) for the \ith player. 

\begin{definition}[$\kce$ strategy modification]
\label{definition:strategy-modification}
A $\kce$ strategy modification $\phi$ (for the \ith player) is a mapping
$\phi$ of the following form: At any infoset $x_{i,h}\in\cX_{i,h}$, for any Type-I \rechist~$b_{1:h-1}\in\omegai(x_{i,h})$, $\phi$ swaps any recommended action $b_h$ into $\phi(x_{i,h}, b_{1:h-1}, b_{h})\in\cA_i$; for any Type-II \rechist~$b_{1:h'}\in\omegaii(x_{i,h})$, $\phi$ directly takes action $\phi(x_{i,h}, b_{1:h'})\in \cA_i$.

Let $\Phi_i^K$ denote the set of all possible $\kce$ strategy modifications for any $0\le K\le \infty$. Formally, for any $\phi\in\Phi_i^K$ and any pure policy $\pi_i\in\Pi_i$, we define the modified policy $\phi\circ \pi_i$ as in Algorithm~\ref{algorithm:phi-circ-pi}.
\end{definition}


We parse the modified policy $\phi\circ \pi_i$ (Algorithm~\ref{algorithm:phi-circ-pi}) as follows. Upon receiving the infoset $x_{i,h}$ at each step $h$, the player has the \rechist~$\bbb$ containing all past observed recommended actions. Then, if $\bbb$ is Type-I, i.e. at most $K-1$ deviations have happened (Line~\ref{line:devif}), then the player observes the current recommended action $b_h=\pi_{i,h}(x_{i,h})$, takes a potentially swapped action $a_h=\phi(x_{i,h},\bbb,b_h)$ (Line~\ref{line:swap}), and appends $b_h$ to the recommendation history (Line~\ref{line:memory}). Otherwise, $(x_{i,h}, \bbb)$ is Type-II, i.e. $K$ deviations have already happened. In this case, the player does not observe the recommended action, and instead takes an action $a_h=\phi(x_{i,h}, \bbb)$, and does not update $\bbb$ (Line~\ref{line:cce}).

\begin{algorithm}[t]
  \small
   \caption{Executing modified policy $\phi\circ \pi_i$}
   \label{algorithm:phi-circ-pi}
   \begin{algorithmic}[1]
     \REQUIRE $\kce$ strategy modification $\phi\in\Phi_i^K$ ($0\le K\le \infty$), policy $\pi_i\in\Pi_i$ for the \ith player. 
     \STATE Initialize recommendation history $\bbb=\emptyset$.
     \FOR{$h=1,\dots,H$}
     \STATE Receive infoset $x_{i,h}\in\cX_{i,h}$.
     \IF{$\bbb\in\omegai(x_{i,h})$} \label{line:devif}
     \STATE Observe recommendation $b_h\sim \pi_{i,h}(\cdot | x_{i,h})$.
     \STATE Take swapped action $a_h=\phi(x_{i,h}, \bbb, b_h)$. \label{line:swap}
     \STATE Update recommendation history $\bbb\setto (\bbb, b_h)\in \cA_i^h$. \label{line:memory}
     \ELSE{} \label{line:develse}
     \STATE {\color{blue} // Must have $\bbb\in\omegaii(x_{i,h})$, do not observe recommendation from $\pi_i$}
     \STATE Take action $a_h=\phi(x_{i,h}, \bbb)$. \label{line:cce} 
     \ENDIF
     \ENDFOR
 \end{algorithmic}
\end{algorithm}

We now define $\kefce$ as the equilibrium induced by the $\kefce$ strategy modification set $\Phi_i^K$. With slight abuse of notation, we define $\phi\circ \wb{\pi}$ for any \emph{correlated policy} $\wb{\pi}$ to be the policy $(\phi\circ \pi_i)\times \pi_{-i}$ where $\pi\sim \wb{\pi}$ is the product policy sampled from  $\wb{\pi}$.

\begin{definition}[$\kefce$]
\label{definition:kefce}
A correlated policy $\bar\pi$ is an \emph{$\epsilon$-approximate $K$ Extensive-Form Correlated Equilibrium} ($\kefce$) if
\begin{align*}
    \kefcegap(\wb{\pi}) \defeq \max_{i\in[m]} \max_{\phi \in \Phi_{i}^K}
    \paren{ 
    V_i^{\phi\circ\wb{\pi}} - V_i^{\wb{\pi}}
    }
    \le \epsilon.
\end{align*}
We say $\bar\pi$ is an (exact) $\kce$ if $\kefcegap(\wb{\pi})=0$.
\end{definition}


\subsection{Properties of $\kefce$}


\paragraph{Equivalence between $\oneefce$ and trigger definition of EFCE}
At the special case $K=1$, our (exact) $\oneefce$ is equivalent to the existing definition of EFCE based on \emph{trigger policies}~\citep{gordon2008no,celli2020no}, which defines an $\epsilon$-approximate EFCE as any correlated policy $\wb{\pi}$ such that the following trigger gap is at most $\eps$:
\begin{align}
\label{equation:trigger-def}
    \triggergap(\wb{\pi}) \defeq \max_{i\in[m]} \max_{(x_i,a) \in \cX_i \times \cA_i}
    \max_{\hat\pi_i \in \Pi_i}\paren{ 
    \E_{\pi\sim\bar\pi} V_i^{\triggerpolicy\times\pi_{-i}} - \E_{\pi\sim\bar\pi} V_i^{\pi}
    }
    \le \epsilon.
\end{align}
Here the trigger policy $\triggerpolicy\in\Pi_i$ (with triggering sequence $(x_i, a)$) is the unique policy that plays $\pi_i\in\Pi_i$, unless infoset $x_i$ is visited and action $a$ is recommended, in which case the sequence $(x_i, a)$ is ``triggered'' and the player plays $\hat\pi_i\in\Pi_i$ thereafter.

\begin{proposition}[Equivalence of $\oneefce$ and trigger definition]
\label{proposition:trigger}
For any correlated policy $\wb{\pi}$, we have
\begin{align*}
    \textstyle
    \triggergap(\wb{\pi}) \le \oneefcegap(\wb{\pi}) \le (\max_{i\in[m]}X_iA_i) \cdot \triggergap(\wb{\pi}),
\end{align*}
In particular, $\oneefcegap(\wb{\pi})=0$ if and only if $\triggergap(\wb{\pi})=0$.
\end{proposition}
The proof can be found in Appendix~\ref{appendix:proof-trigger}. Proposition~\ref{proposition:trigger} has two main implications:
(1) An exact EFCE defined by the trigger gap is equivalent to an exact $\oneefce$ (cf. Definition~\ref{definition:kefce}). Therefore the two definitions yields the same set of exact equilibria. (2) For $\eps>0$, $\oneefcegap(\wb{\pi})\le \eps$ implies $\triggergap(\wb{\pi})\le \eps$, but the converse only holds with an extra $\max_{i\in[m]}X_iA_i$ factor, and thus $\oneefcegap$ is a stricter metric for approximate equilibria than $\triggergap$. This distinction is inherent instead of a proof artifact: Our $\oneefce$ strategy modification (Algorithm~\ref{algorithm:phi-circ-pi}) is able to \emph{implement multiple trigger policies simultaneously}, as long as their triggering sequences are not ancestors or descendants of each other.


\paragraph{Containment relationship}
We next show that $\kefce$ are indeed stricter equilibria as $K$ increases, i.e. any (approximate) $\kpefce$ is also an (approximate) $\kefce$, but not the converse. This justifies the necessity of considering $\kefce$ for all values of $K\ge 1$ and shows that they are strict strengthenings of the $\oneefce$.
Note that as we consider games with a finite horizon $H$, we have $\kefcegap = \Hefcegap$ for all $K\ge H$ (including $K=\infty$). The proof of Proposition~\ref{proposition:containment} can be found in Appendix~\ref{appendix:proof-containment}. 
\begin{proposition}[Containment relationship]
\label{proposition:containment}
    For any correlated policy $\wb{\pi}$, we have
    \begin{align*}
        & \zeroefcegap(\wb{\pi}) \le \oneefcegap(\wb{\pi}) \le \cdots \le \kefcegap(\wb{\pi}) \le \kpefcegap(\wb{\pi}) \le \cdots \le \inftyefcegap(\wb{\pi}).
    \end{align*}
    In other words, $\kefce$ are stricter equilibria as $K$ increases: Any $\eps$-approximate $\kpefce$ is also an $\eps$-approximate $\kefce$ for any $\eps\ge 0$ and $K\ge 0$.
    
    Moreover, the converse bounds do not hold, even if multiplicative factors are allowed: For any $0\le K<\infty$, there exists a game with $H=K+1$ and a correlated policy $\wb{\pi}$ for which 
    \begin{align*}
        \kefcegap(\wb{\pi})=0~~~{\rm but}~~~\kpefcegap(\wb{\pi})\ge 1/3>0.
    \end{align*}
\end{proposition}


\paragraph{Relationship with other correlated equilibria}
We also remark that the two endpoints $K=0$ and $K=\infty$ of $\kefce$ are closely related to other existing definitions of correlated equilibria in IIEFGs. Concretely, $\zeroefce$ is equivalent to Normal-Form Coarse Correlated Equilibria (NFCCE), whereas $\inftyefce$ is equivalent to using the ``Behavioral Correlated Equilibria'' considered in~\citep{morrill2021efficient}, which is strictly weaker than Normal-Form Correlated Equilibria (NFCE) that is more computationally challenging to learn~\citep{farina2020coarse,celli2020no}. See Appendix~\ref{appendix:relationship} and Proposition~\ref{prop:nfce-nfcce-relation} for the definitions of these equilibria and the formal statement of the above relationships.

%% file: new_version/full-feedback.tex
\section{Computing $\kefce$ from full feedback}
\label{section:full-feedback}



We first present our algorithm for computing $\kefce$ in the full-feedback setting.

\begin{algorithm}[t]
  \small
   \caption{\kefr~with full feedback (\ith player's version)}
   \label{algorithm:kefr-fullfeedback}
   \begin{algorithmic}[1]
     \REQUIRE Algorithm \Regalg~for minimizing wide range regret; learning rates $\{\eta_{x_{i,h}}\}_{x_{i,h}\in\cX_i}$.
     \STATE Initialize regret minimizers $\{\regmin_{x_{i,h}}\}_{x_{i,h}\in\cX_i}$ with $\Regalg$ and learning rate $\eta_{x_{i,h}}$.
     \FOR{iteration $t=1,\dots,T$}
     \FOR{$h=1,\dots,H$} 
     \FOR{$x_{i,h}\in\mc{X}_{i,h}$}
     \STATE Compute $S^t_{b_{1:h-1}} = \prod_{k = 1}^{h-1} \pi_{i, k}^t(b_k \vert x_k)$ for all $b_{1: h - 1} \in \omegai(x_{i, h})$. \label{line:Sb-for-typeI} 
     \STATE Compute $S^t_{b_{1:h'}} = \prod_{k = 1}^{h'} \pi_{i, k}^t(b_k \vert x_k)$ for all $b_{1: h'} \in \omegaii(x_{i, h})$.  \label{line:Sb-for-typeII} 
     \STATE $\regmin_{x_{i,h}}.\observeTSF( \{ S^t_{b_{1:h - 1}} \}_{b_{1:h - 1} \in \omegai(x_{i, h})} \cup \{ S^t_{b_{1:h'}} \}_{b_{1:h'} \in \omegaii(x_{i, h})})$. \label{line:observe-Sb}
     \STATE Set policy $\pi_i^t(\cdot | x_{i, h}) \setto \regmin_{x_{i,h}} .\recommend()$.   \label{line:set-policy}
     \ENDFOR
     \ENDFOR
     \STATE\label{line:observe-loss}  Observe counterfactual losses $\set{L_h^t(x_{i,h}, a_h)}_{h, x_{i,h}, a_h}$ (depending on $\pi_i^t$ and $\pi_{-i}^t$; cf.~\eqref{equation:L-definition}).
     \FOR{all $x_{i,h}\in\cX_i$}\label{line:begin-update}
     \STATE $\regmin_{x_{i,h}}.\observeLoss( \set{ L_h^t(x_{i,h}, a)}_{a\in\cA_i} )$. \label{line:regmin}
     \ENDFOR
     \ENDFOR
     \ENSURE Policies $\{ \pi_i^t \}_{t=1}^T$. 
 \end{algorithmic}
\end{algorithm}



\paragraph{Algorithm description}
Our algorithm $K$ Extensive-Form Regret Minimization (\kefr), described in Algorithm~\ref{algorithm:kefr-fullfeedback}, is an uncoupled no-regret learning algorithm that aims to minimize the following \emph{$K$-EFCE regret}
\begin{align}
  \label{equation:kefce-regret}
  R_{i,K}^T \defeq \max_{\phi \in \Phi_i^K} \sum_{t = 1}^T \Big( V_{i}^{\phi\circ \pi_i^t \times \pi_{-i}^t} - V_{i}^{\pi^t} \Big).
\end{align}
By standard online-to-batch conversion, achieving sublinear $\kefce$ regret for every player implies that the average joint policy over all players is an approximate $\kefce$ (Lemma~\ref{lem:online-to-batch}).

At a high level, our Algorithm~\ref{algorithm:kefr-fullfeedback} builds upon the EFR algorithm of~\citet{morrill2021efficient} to minimize the $\kefce$ regret~$R_{i,K}^T$, by maintaining a \emph{regret minimizer} $\regmin_{x_{i,h}}$ (using algorithm \Regalg) at each infoset $x_{i,h}\in\cX_u$ that is responsible for outputting the policy $\pi_i^t(\cdot|x_{i,h})\in\Delta_{\cA_i}$ (Line~\ref{line:set-policy}) which combine to give the overall policy $\pi_i^t$ for the $t$-th iteration.

Core to our algorithm is the requirement that $\regmin_{x_{i,h}}\sim \Regalg$ should be able to minimize regrets with \emph{time-selection functions and strategy modifications} (also known as the \emph{wide range regret})~\citep{lehrer2003wide,blum2007external}. Specifically, $\regmin_{x_{i,h}}$ needs to control the regret
\begin{align}
\label{equation:wide-range-regret}
  \max_{\vphi\in\Psi^s} \sum_{t = 1}^T \underbrace{\prod_{k = 1}^{h-1} \pi_i^t (b_{k}|x_{k})}_{\defeq S^t_{b_{1:h-1}}} \Big( \< \pi_{i, h}^t(\cdot \vert x_{i, h}) - \vphi \circ \pi_{i, h}^t(\cdot \vert x_{i, h}), \L_{i,h}^{t}(x_{i, h},\cdot) \> \Big)
\end{align}
for all possible Type-I \rechists~$b_{1:h-1}\in\omegai(x_{i,h})$ simultaneously, where $S^t_{b_{1:h-1}}$ is the \emph{time-selection function} (i.e. a weight function) associated with this $b_{1:h-1}$ (cf. Line~\ref{line:Sb-for-typeI}), and $\Psi^s=\set{\psi:\cA_i\to\cA_i}$ is the set of all \emph{swap modifications} from the action set $\cA_i$ onto itself. (An analogous regret for Type-II \rechists~is also controlled by $\regmin_{x_{i,h}}$.) Controlling these ``local" regrets at each $x_{i,h}$ guarantees that the overall $\kefce$ regret is bounded, by the $\kefce$ regret decomposition (cf. Lemma~\ref{lem:k-ce-regret_decomposition}).


To control this wide range regret, we instantiate \Regalg~as \wrhedge~(Algorithm~\ref{algorithm:time-selection-swap}; cf. Appendix \ref{appendix:bandit-algorithm}), which is similar as the wide regret minimization algorithm in~\citep{khot2008minimizing}, with a slight modification of the initial weights suitable for our purpose (cf.~\eqref{equation:wrhedge-initial-weight}).
The learning rate is set as
\begin{align}\label{equation:learning-rate-full-feedback}
  \textstyle
  \eta_{x_{i,h}} = \sqrt{{H \choose K\wedge H } X_i  A_i^{K \wedge
 H } \log A_i / (H^2T) }
\end{align}
for all $x_{i,h}\in \cX_i$. With this algorithm in place, at each iteration, $\regmin_{x_{i,h}}$ observes all time selection functions (Line~\ref{line:observe-Sb}), computes the policy for the current iteration (Line~\ref{line:set-policy}), and then observes the loss vector $L_{i,h}^t(x_{i,h},\cdot)$ (Line~\ref{line:observe-loss}) that is useful for updating the policy in the next iteration.

\paragraph{Theoretical guarantee}
We are now ready to present the theoretical guarantee for \kcfr.
\begin{theorem}[Computing $\kefce$ from full feedback]
\label{theorem:kcfr-full-feedback}
For any $0\le K\le \infty$, $\epsilon \in (0,H]$, let all players run Algorithm \ref{algorithm:kefr-fullfeedback} together in a self-play fashion where \Regalg~is instantiated as Algorithm \ref{algorithm:time-selection-swap}~with learning rates specified in (\ref{equation:learning-rate-full-feedback}). Let $\pi^t = \{\pi_i^t\}_{i \in [m]}$ denote the joint policy of all players at the $t$'th iteration. Then the average policy $\bar\pi={\rm Unif}(\{\pi^t\}_{t=1}^T)$ satisfies $\kefcegap(\wb{\pi})\le \eps$, as long as the number of iterations
 \begin{align*}
   \textstyle
    T \ge \cO\paren{ {H  \choose K\wedge H } \paren{\max_{i\in[m]} X_i A_i^{K\wedge H}} \iota / \epsilon^2}, 
  \end{align*}
  where $\iota=\log(\max_{i\in[m]}A_i)$ is a log factor and $\cO(\cdot)$ hides ${\rm poly}(H)$ factors.
\end{theorem}


In the special case of $K=1$, Theorem~\ref{theorem:kcfr-full-feedback} shows that \kscfr~can compute an $\eps$-approximate $\oneefce$ within $\tO(\max_{i\in[m]} X_iA_i/\eps^2)$ iterations. This improves over the existing $\tO(\max_{i\in[m]}X_i^2A_i^2/\eps^2)$ iteration complexity of~\citet{celli2020no,farina2021simple} by a factor of $X_iA_i$. Also, compared with the iteration complexity of the optimistic algorithm of~\citep{anagnostides2022faster} which is at least $\tO(\max_{i\in[m]} X_i^{4-\delta}A_i^{4/3}/\eps^{4/3})$\footnote{More precisely, \citet[Corollary 4.17]{anagnostides2022faster} proves an $\tO((X_i \max_{\pi_i\in\Pi_{\max}} \|\pi_i\|_1^2 A / \eps)^{4/3})$ iteration complexity, which specializes to the above rate, as for any $\delta>0$ a game with $\max_{\pi_i\in\Pi_{\max}} \|\pi_i\|_1 \ge X_i^{1-\delta}$ can be constructed.}, we achieve lower $X_i$ dependence (though worse $\eps$ dependence).


For $1<K\le\infty$, Theorem~\ref{theorem:kcfr-full-feedback} gives a sharp $\tO({H  \choose K\wedge H } (\max_{i\in[m]} X_i A_i^{K\wedge H})/\epsilon^2)$ iteration complexity for computing $\kefce$. This improves over the $\tO({H  \choose K\wedge H } (\max_{i\in[m]} X_i^2 A_i^{(K\wedge H)})/\epsilon^2)$ rate of EFR~\citep{morrill2021efficient} instantiated to the $\kefce$ problem. Also, note that although the term $A_i^{K \wedge H}$ is exponential in $K$ (for $K\le H$), this is sensible since it is roughly the same scale as the number of possible recommendation histories, which is also the ``degree of freedom" within a $\kefce$ strategy modification.




\paragraph{$\kefce$ regret guarantee} While Theorem~\ref{theorem:kcfr-full-feedback} requires all players running \kefr~together via self-play, our algorithm also achieves a low $\kefce$ regret when controlling the \ith player only and facing potentially adversarial opponents. Concretely, Algorithm~\ref{algorithm:kefr-fullfeedback} for the \ith player achieves $R_{i,K}^T\le \tO(\sqrt{{H\choose K\wedge H}X_iA_i^{K\wedge H}T})$ with high probability (Corollary~\ref{corollary:kscfr-regret}). In particular, the $\tO(\sqrt{X_iT})$ scaling is optimal up to log factors, due to the fact that $R_{i,K}^T\ge R_{i,0}^T$ (i.e. the vanilla regret) and the known lower bound $R_{i,0}^T\ge \Omega(\sqrt{X_iT})$ in IIEFGs~\citep{zhou2018lazy}.


\paragraph{Proof overview}
Our Theorem~\ref{theorem:kcfr-full-feedback} follows from a sharp analysis on the $\kefce$ regret of Algorithm~\ref{algorithm:kefr-fullfeedback}, by incorporating (i) a decomposition of the $\kefcegap$ into local regrets at each infoset with tight leading coefficients (Lemma \ref{lem:k-ce-regret_decomposition}), and (ii) loss-dependent upper bounds for the wide range regret of \wrhedge~(Lemma~\ref{lemma:regret-with-time-selection}), which when plugged into the aforementionted regret decomposition yields the improved dependence in $(X_iA_i^{K\wedge H})$ over the analysis of~\citet{morrill2021efficient} (Lemma~\ref{lemma:bound-G_h1-fullfeedback} \&~\ref{lemma:bound-G_h2-fullfeedback}). The full proof can be found in Appendix \ref{appendix:proof-full-feedback}.

%% file: new_version/bandit-feedback.tex
\section{Learning $\kefce$ from bandit feedback}
\label{section:bandit-feedback}

We now present Balanced \kefr, a sample-based variant of \kefr~that achieves a sharp sample complexity in the more challenging bandit feedback setting. 

Our algorithm relies on the following \emph{balanced exploration policy}~\citep{farina2020stochastic,bai2022near}. Recall that $\abs{\cC_h(x_{i,h'}, a_{i,h'})}$ is the number of descendants of $(x_{i,h'}, a_{i,h'})$ within the $h$-th layer of the \ith player's game tree (cf. Section~\ref{section:prelim}).
\begin{definition}[Balanced exploration policy]
  \label{definition:balanced-policy}
  For any $1\le h\le H$, we define $\pi_i^{\star, h}$, the (\ith player's) \emph{balanced exploration policy for layer $h$} as
  \begin{align}\label{eqn:balanced-strategy}
    \pi^{\star, h}_{i,h'}(a_{h'} | x_{h'}) \defeq \frac{\vert \cC_h(x_{i,h'}, a_{i,h'}) \vert }{\vert \cC_h(x_{i,h'}) \vert} ~~~\textrm{for all}~~(x_{i,h'}, a_{i,h'})\in \mc{X}_{i, h'}\times \mc{A}_i,~~1\le h'\le h-1,
  \end{align}
  and $\pi^{\star, h}_{i,h'}(a_{i,h'} | x_{i,h'}) \defeq 1/A_i$ for $h'\ge h$. 
\end{definition}
Note that there are $H$ such policies, one for each layer $h$. We remark that the construction of $\pi^{\star, h}_{i}$ requires knowledge about the descendant relationships among the \ith player's infosets, which is a mild requirement (e.g. can be efficiently obtained from one traversal of the \ith player's game tree).


\paragraph{Algorithm description}
Our Balanced \kefr~(sketched in Algorithm~\ref{algorithm:sample-based-kscfr-sketch} and fully described in Algorithm~\ref{algorithm:kefr-sampled-bandit}) builds upon the full feedback version of \kefr~(Algorithm~\ref{algorithm:kefr-fullfeedback}). For each infoset $x_{i,h}\in\cX_i$, the algorithm computes time selection functions $\{S^t_{\bbb}\}_{\bbb\in\omegai(x_{i,h})\cup\omegaii(x_{i,h})}$, based on the corresponding $M^t$ (defined in Line~\ref{line:Sb-for-typeI-balanced} \&~\ref{line:Sb-for-typeII-balanced}), as well as the following additional weighting function (below $W\defeq\{k\in[h-1]:a_k\neq b_k\}$, and $\fill(\cdot,\cdot)$ is defined in~\eqref{equation:fill}) 
\begin{align}
    & w_{b_{1:h-1}}(x_{i,h}) = \prod_{k \in  \fill(W, (h-1)\wedge(K-1)) \cup \{h\}} \pi^{\star, h}_{i,k}(a_k \vert x_k),~~\textrm{for all}~b_{1: h - 1} \in \omegai(x_{i, h}), \label{equation:wb-i} \\
    & w_{b_{1:h'}}(x_{i,h}) = \prod_{k \in W \cup \{h'+1, \cdots, h\}} \pi^{\star, h}_{i,k}(a_k \vert x_k),~~\textrm{for all}~b_{1: h'} \in \omegaii(x_{i, h}). \label{equation:wb-ii}
\end{align}
The resulting choice of time selection functions, $S^t_{\bbb}=M^t_{\bbb}w^t_{\bbb}(x_{i,h})$, is different from Algorithm~\ref{algorithm:kefr-fullfeedback}, and is needed for this sampled case. 

The main new ingredient within Algorithm~\ref{algorithm:sample-based-kscfr-sketch} is to use sample-based loss estimators obtained by two \emph{balanced sampling} algorithms (Algorithm~\ref{algorithm:kcfr-estimator-ii} \&~\ref{algorithm:kcfr-estimator-i}), one for each type of \rechists. Here we present the sampling algorithm for Type-II \rechists~in Algorithm~\ref{algorithm:kcfr-estimator-ii}; The sampling algorithm for Type-I \rechists~(Algorithm~\ref{algorithm:kcfr-estimator-i}) is designed similarly and deferred to Appendix~\ref{appendix:sample-based-estimator-i} due to space limit. Algorithm~\ref{algorithm:kcfr-estimator-ii} performs two main steps:
\begin{itemize}[topsep=0pt, itemsep=0pt, leftmargin=12pt]
\item Line~\ref{line:sampling-start}-\ref{line:sampling-end} (Sampling): Construct policies $\{\pi_i^{t, (h, h', W)}\}$ that are \emph{interlaced concatenations} of the current $\pi_i^t$ and the balanced policy $\pi^{\star, h}_i$, and play one episode using each policy against $\pi_{-i}^t$. 
\item Line~\ref{line:estimation-ii}: Construct loss estimators $\{\wt{L}_{x_{i,h}, b_{1:h'}}(a)\}_{x_{i,h}, b_{1:h'}, a}$ by~\eqref{equation:pil-estimator-ii}, where for each $x_{i,h}$ and $b_{1:h'}\in\omegaii(x_{i,h})$ this is an unbiased estimator of counterfactual losses $\set{ L_h^t(x_{i,h}, a)}_{a\in\cA_i}$. These unbiased estimators will be used by Algorithm~\ref{algorithm:sample-based-kscfr-sketch} to be fed into the regret minimization algorithm \Regalg.
\end{itemize}


\begin{algorithm}[t]
\small
   \caption{\balkefr~for the \ith player (sketch; detailed description in Algorithm~\ref{algorithm:kefr-sampled-bandit})}
   \label{algorithm:sample-based-kscfr-sketch}
   \begin{algorithmic}[1]
   \REQUIRE Weights $\{w_{b_{1:h-1}}(x_{i,h})\}_{x_{i,h},b_{1:h-1}\in\omegai(x_{i,h})}$ and  $\{w_{b_{1:h'}}(x_{i,h})\}_{x_{i,h},b_{1:h'}\in\omegaii(x_{i,h})}$ defined in~\eqref{equation:wb-i},~\eqref{equation:wb-ii}, learning rates $\{\eta_{x_{i,h}}\}_{x_{i,h}\in\cX_i}$, loss upper bound $\wb{L}>0$.
   \STATE Initialize regret minimizers $\{\regmin_{x_{i,h}}\}_{x_{i,h}\in\cX_i}$ with $\Regalg$, learning rate $\eta_{x_{i,h}}$, and loss upper bound $\wb{L}$.
   \STATE The algorithm is almost the same as the full feedback case (Algorithm~\ref{algorithm:kefr-fullfeedback}), with the following changes:
   \STATE For all $b_{1: h - 1} \in \omegai(x_{i, h})$, replace Line \ref{line:Sb-for-typeI} by $S^t_{b_{1:h-1}}\defeq M^t_{b_{1:h-1}} w_{b_{1:h-1}}(x_{i,h})$ where $M^t_{b_{1:h-1}} \defeq \prod_{k = 1}^{h-1} \pi_{i, k}^t(b_k \vert x_k)$. \label{line:Sb-for-typeI-balanced}
   \STATE For all $b_{1: h'} \in \omegaii(x_{i, h})$, replace Line \ref{line:Sb-for-typeII} by $S^t_{b_{1:h'}} \defeq M^t_{b_{1:h'}} w_{b_{1:h'}}(x_{i,h})$ where $M^t_{b_{1:h'}} \defeq  \prod_{k = 1}^{h'} \pi_{i, k}^t(b_k \vert x_k)$. \label{line:Sb-for-typeII-balanced}
   \STATE Replace Line~\ref{line:observe-loss} by obtaining sample-based loss estimators $\{\wt L^t_{(x_{i,h}, b_{1:h'})}(\cdot)\}_{(x_{i,h}, b_{1:h'})\in \omegaii}$ and $\{\wt L^t_{(x_{i,h}, b_{1:h-1})}(\cdot)\}_{(x_{i,h}, b_{1:h-1})\in \omegai}$ from Algorithm~\ref{algorithm:kcfr-estimator-ii} \&~\ref{algorithm:kcfr-estimator-i} respectively.
   \STATE Replace Line~\ref{line:regmin} by feeding the above loss estimators instead of true counterfactual losses to $\regmin_{x_{i,h}}$.
   \ENSURE Policies $\set{\pi_i^t}_{t=1}^T$.
 \end{algorithmic}
\end{algorithm}

\begin{algorithm}[t]
  \small
   \caption{Loss estimator for Type-II \infoseqs~via Balanced Sampling (\ith player's version)}
   \label{algorithm:kcfr-estimator-ii}
   \begin{algorithmic}[1]
   \REQUIRE Policy $\pi_i^t$, $\pi_{-i}^t$. Balanced exploration policies $\{\pi^{\star, h}_i\}_{h\in[H]}$.
   \FOR{$K\le h'<h\le H$, $W\subseteq[h']$ with $|W|=K$ and ending in $h'$} \label{line:typeII-samples}
   \STATE Set policy $\pi_i^{t, (h, h', W)}\setto (\pi^{\star, h}_{i, k})_{k\in W\cup\set{h'+1,\dots,h}}\cdot (\pi^t_{i, k})_{k\in [h']\setminus W} \cdot \pi^t_{i, (h+1):H}$.
   \STATE Play $\pi_i^{t, (h, h', W)}\times \pi_{-i}^t$ for one episode, observe trajectory
   \begin{align*}
      ( x_{i,1}^{t, (h, h', W)}, a_{i,1}^{t, (h, h', W)}, r_{i,1}^{t, (h, h', W)}, \dots, x_{i,H}^{t, (h, h', W)}, a_{i,H}^{t, (h, h', W)}, r_{i,H}^{t, (h, h', W)} ).
   \end{align*}
   \label{line:sampling-end}
   \ENDFOR
   \vspace{-1em}
   \FOR{all $(x_{i,h}, b_{1:h'})\in \omegaii$} 
  \STATE Find $(x_{i, 1}, a_1) \prec \cdots \prec (x_{i, h-1}, a_{h-1}) \prec x_{i, h}$.
  \STATE Set $W\setto \set{k\in[h']: b_{k}\neq a_{k}}$~
  \STATE Construct loss estimator for all $a\in\cA_i$ (below $a_h\in\cA_i$ is arbitrary):
  {\proofsize \begin{align}
                \label{equation:pil-estimator-ii}
                \hspace{-2em}
      \wt L^t_{(x_{i,h}, b_{1:h'})}(a) \setto 
      \frac{\indic{(x_{i,h}^{t, (h, h', W)}, a_{i,h}^{t, (h, h', W)}) = (x_{i,h}, a)}}{ \pi_{i,1:h}^{t, (h, h', W)}(x_{i,h}, a) } \cdot \sum_{h''=h}^H \paren{1 - r_{i,h''}^{t, (h, h', W)}}.
  \end{align}}
  \label{line:estimation-ii}
  \ENDFOR
  \ENSURE Loss estimators $\set{\wt L^t_{(x_{i,h}, b_{1:h'})}(\cdot)}_{(x_{i,h}, b_{1:h'})\in \omegaii}$.
 \end{algorithmic}
\end{algorithm}


\paragraph{Stochastic wide-range regret minimization}
Algorithm~\ref{algorithm:sample-based-kscfr-sketch} requires the wide-range regret minimization algorithm \Regalg~to additionally handle the stochastic setting, i.e. minimize the wide-range regret (e.g.~\eqref{equation:wide-range-regret}) when fed with our sample-based loss estimators. Here, we instantiate \Regalg~to be \wrhedgebandit~(Algorithm~\ref{algorithm:time-selection-sample}), a stochastic variant of \wrhedge, with hyperparameters
\begin{align}
    \label{equation:learningrate-bandit}
    \textstyle \eta_{x_{i,h}} = \sqrt{{H \choose K \wedge H } X_i A_i^{K\wedge H + 1}  \log (8\sum_{i\in[m]}X_iA_i/p) / (H^3T) },~~~\wb{L} = H.
\end{align}
We remark that the \wrhedgebandit~is a non-trivial extension of \wrhedge~to the stochastic setting, as in each round $t$ it admits \emph{multiple} sample-based loss estimators, one for each $\bbb$ (coming from potentially different sampling distributions), with the same mean (cf. Line~\ref{line:observe-bandit-loss}). This generality is needed, as Algorithm~\ref{algorithm:kcfr-estimator-ii} uses different sampling policies to construct the loss estimator $\wt{L}^t_{x_{i,h}, b_{1:h'}}(\cdot)$ for each $b_{1:h'}\in\omegaii(x_{i,h})$ (cf.~\eqref{equation:pil-estimator-ii}).



\paragraph{More details on the sampling technique}
We remark that (i) The sampling policies $\{\pi_i^{t, (h, h', W)}\}$ in Algorithm~\ref{algorithm:kcfr-estimator-ii} are \emph{interlaced concatenations} of $\pi_i^t$ and $\pi^{\star, h}_i$ along time steps $h$, where the policy to take at each $h$ is determined by $W$. 
These policies are generalizations of the sampling policies in the Balanced CFR algorithm of~\citet{bai2022near} (which can be thought of as a simple non-interlacing concatenation). They allow \emph{time-selection aware sampling}: Each loss estimator $\wt{L}^t_{(x_{i,h}, b_{1:h'})}(\cdot)$ achieves low variance relative to the corresponding time selection function $S^t_{b_{1:h'}}$. Further, there is an \emph{efficient sharing} of sampling policies, as here roughly ${H\choose K\wedge H}X_iA_i^{K\wedge H}$ loss estimators (one for each $(x_{i,h}, b_{1:h'})$) are constructed using only (a much lower number of) $H{H\choose K\wedge H}$ policies.
(ii) Effectively, our unbiased loss estimators~\eqref{equation:pil-estimator-ii} involve both importance weighting (for steps $1:h$) and Monte-Carlo estimation (for steps $h+1$ onward) simultaneously, which is important for obtaining a sharp sample complexity result.

\paragraph{Theoretical guarantee} 
We now present our main result for the bandit feedback setting.
\begin{theorem}[Learning $\kefce$ from bandit feedback]
  \label{theorem:kcfr-bandit}
  For any $0\le K\le \infty$, $\epsilon \in (0,H]$ and $p \in [0,1)$, letting all players run Algorithm \ref{algorithm:sample-based-kscfr-sketch} (full description in Algorithm~\ref{algorithm:kefr-sampled-bandit}) together in a self-play fashion for $T$ iterations, with \Regalg~instantiated as \wrhedgebandit~(Algorithm \ref{algorithm:time-selection-sample}) with hyperparameters in~\eqref{equation:learningrate-bandit}. Let $\pi^t = \{\pi_i^t\}_{i \in [m]}$ denote the joint policy of all players at the $t$'th iteration. Then, with probability at least $1-p$, the correlated policy $\bar\pi={\rm Unif}(\{\pi^t\}_{t=1}^T)$ satisfies $\kefcegap(\wb{\pi})\le \eps$, as long as $T \ge \cO(H^3 {H  \choose K\wedge H } (\max_{i\in[m]} X_i A_i^{K\wedge H+1}) \iota /\epsilon^2)$. The total number of episodes played is
 \begin{align*}
   \textstyle
   3mH{H\choose K\wedge H}\cdot T = \cO\paren{ m{H \choose K\wedge H}^2 \paren{\max_{i\in[m]} X_i A_i^{K\wedge H+1}} \iota /{\epsilon^2} },
  \end{align*}
  where $\iota=\log(8\sum_{i\in[m]}X_iA_i/p)$ is a log factor and $\cO(\cdot)$ hides ${\rm poly}(H)$ factors.
\end{theorem}
To our best knowledge, Theorem~\ref{theorem:kcfr-bandit} provides the first result for learning $\kefce$ under bandit feedback. The sample complexity $\tO({H\choose K\wedge H}^2 \max_{i\in[m]} (X_i A_i^{K\wedge H+1})/\eps^2)$ (ignoring $m$, $H$ factors) has only an ${H\choose K\wedge H} A_i$ additional factor over the iteration complexity in the full feedback setting (Theorem~\ref{theorem:kcfr-full-feedback}), which is natural---The ${H\choose K\wedge H}$ comes from the number of episodes sampled within each iteration (Lemma~\ref{lemma:needed-episode}), and the $A_i$ arises from estimating loss vectors from bandit feedback. In particular, the special case of $K=1$ provides the first result for learning EFCEs from bandit feedback, with sample complexity $\tO(\max_{i\in[m]}X_iA_i^2/\eps^2)$. We remark that the linear in $X_i$ dependence at all $K\ge 0$ is optimal, as the sample complexity lower bound for the $K=0$ case (learning NFCCEs) is already $\Omega(\max_{i\in[m]} X_iA_i/\eps^2)$~\citep{bai2022near}\footnote{The sample complexity lower bound in~\citep{bai2022near} is stated for learning Nash Equilibria in two-player zero-sum IIEFGs, but can be directly extended to learning NFCCEs in multi-player general-sum IIEFGs.}. 

We remark that the policies $\{\pi_i^t\}_{t=1}^T$ maintained in Algorithm~\ref{algorithm:sample-based-kscfr-sketch} also achieves sublinear $\kefce$ regret. However, strictly speaking, this is not a regret bound of our algorithm, as the sampling policies $\pi^{t, (h, h', W)}_i$ actually used are not $\pi_i^t$.


\paragraph{Proof overview}
The proof of Theorem~\ref{theorem:kcfr-bandit} builds on the analysis in the full-feedback case, and further relies on several new techniques in order to achieve the sharp linear in $\max_{i\in[m]} X_i$ sample complexity: (1) A regret bound for the \wrhedgebandit~algorithm under the same-mean condition (Lemma~\ref{lemma:regret-with-time-selection-bandit}), which may be of independent interest; (2) Crucial use of the \emph{balancing property} of $\pi^{\star, h}_i$ (Lemma~\ref{lemma:balancing}) to control the variance of the loss estimators $\wt{L}^t_{(x_{i,h}, b_{1:h'})}(\cdot)$, which in turn produces sharp bounds on the regret terms and additional concentration terms (Lemma~\ref{lemma:bound-regret^I}-\ref{lemma:bound-bias2^II}).
The full proof can be found in Appendix~\ref{appendix:proof-kcfr-bandit}.

%% file: new_version/conclusion.tex
\section{Conclusion}
This paper proposes $\kefce$, a generalized definition of Extensive-Form Correlated Equilibria in Imperfect-Information Games, and designs sharp algorithms for computing $\kefce$ under full feedback and learning a $\kefce$ under bandit feedback. Our algorithms perform wide-range regret minimization over each infoset to minimize the overall $\kefce$ regret, and introduce new efficient sampling policies to handle bandit feedback. We believe our work opens up many future directions, such as accelerated techniques for computing $\kefce$ from full feedback, learning other notions of equilibria from bandit feedback, as well as empirical investigations of our algorithms.


%% file: new_version/tools.tex
\section{Technical tools}

\subsection{Technical lemmas}
The following Freedman's inequality can be found in~\citep[Lemma 9]{agarwal2014taming}.
\begin{lemma}[Freedman's inequality]
  \label{lemma:freedman}
  Suppose random variables $\set{X_t}_{t=1}^T$ is a martingale difference sequence, i.e. $X_t\in\cF_t$ where $\set{\cF_t}_{t\ge 1}$ is a filtration, and $\E[X_t|\cF_{t-1}]=0$. Suppose $X_t\le R$ almost surely for some (non-random) $R>0$. Then for any $\lambda\in(0, 1/R]$, we have with probability at least $1-\delta$ that
  \begin{align*}
    \sum_{t=1}^T X_t \le \lambda \cdot \sum_{t=1}^T \E\brac{X_t^2 | \cF_{t-1}} + \frac{\log(1/\delta)}{\lambda}.
  \end{align*}
\end{lemma}

%% file: new_version/time_selection.tex
\subsection{Minimizing wide-range regret}
\label{appendix:bandit-algorithm}


\def\cB{{\mathcal B}}
\def\cI{{\mathcal I}}
\def\oL{{\overline L}}

\paragraph{Time selection functions} For each element $b$ in a finite set $\cB$, $\{S^t_b\}_{t\ge 1}$ represent the time selection function indexed by $b$. We let $\cB^s$ be the set of swap time selection index and $\cB^e$ be the set of external time selection index.

\paragraph{Modifications}  We  denote the set of swap modification function set $\Psi^s = \{ \psi: [A] \to [A]\}$ and external modification function set $\Psi^e = \{ \psi: [A] \to [A]: \psi(b) = a, \forall b \in [A]\}$. Given a modification rule $\psi: [A] \to [A]$, define $M_\psi$ to be the matrix with a $1$ in column $\psi(b)$ of row $b$ for all $i\in[A]$ and zeros everywhere else.

Specializing the results in  \cite{khot2008minimizing}, here we present an algorithm for minimizing regret modified with time selection function and strategy modification pairs. 

\begin{algorithm}[h]
  \caption{Wide-Range Regret Minimization with Hedge (\WRHEDGE)}
  \label{algorithm:time-selection-swap}
  \begin{algorithmic}[1]
    \REQUIRE Learning rate $\eta>0$. Swap index set $\cB^s$ and external index set $\cB^e$; Swap strategy modification set $\Psi^s$ and external index set $\Psi^e$.
    \STATE Initialize $\cI \defeq (\cB^s\times \Psi^s) \cup (\cB^e\times \Psi^e)$, $S^0_b \setto 0$ for all $b \in \cB^s \cup \cB^e$, 
    \begin{align}
      \label{equation:wrhedge-initial-weight}
      q^0(b, \psi) \setto \frac{\vert \Psi^s \vert  \indic{b \in \cB^e} + \vert \Psi^e\vert \indic{b \in \cB^s}}{\sum_{(b', \psi') \in \cI}\brac{\vert \Psi^s \vert  \indic{b' \in \cB^e} + \vert \Psi^e\vert \indic{b' \in \cB^s}}}
    \end{align}
    for all $(b, \psi) \in \cI$.
    \FOR{iteration $t=1,\dots,T$}
    \STATE (\observeTSF) Receive time selection functions $\{ S^t_b \}_{b \in \cB^s \cup \cB^e}$. 
    \STATE Update distribution over $(b, \psi) \in \cI$: 
    \begin{align*}
      q^t(b,\psi) \propto q^{t-1}(b, \psi) \exp\Big\{\eta\exp(-\eta \|\ell^{t-1}\|_\infty) S^{t-1}_b\langle p^{t-1}, \ell^{t-1} \rangle - \eta S^{t-1}_b\langle \psi \circ p^{t-1}, \ell^{t-1} \rangle \Big\}. 
    \end{align*}
    \STATE \label{line:solve-equation} Set $p^t\in \Delta([A])$ as a solution to the equation ${p^t}^\top = {p^t}^\top \paren{ \frac{\sum_{(b, \psi) \in \cI} S^t_b q^t(b, \psi) M_\psi}{\sum_{(b, \psi) \in \cI} S^t_b q^t(b, \psi)} }$ . 
    \IF{\recommend~is called}
    \STATE Output the vector  $p^t$.
    \ENDIF
    \STATE (\observeLoss) Receive loss vector $ {\ell}^t  \in \mathbb{R}^A$. 
    \ENDFOR
  \end{algorithmic}
\end{algorithm}

With a slight modification on the initial weight and a refined proof, we have the following lemma on bounding the wide range regret.


\begin{lemma}[Wide-range regret bound of \wrhedge]
\label{lemma:regret-with-time-selection}
Let $\{ \ell^t(a)\}_{a \in [A], t \in [T]}$ and $\{ S^t_b\}_{b \in \cB^s \cup \cB^e, t \in [T]}$ be  arbitrary arrays of loss functions and time selection functions. 
Assume that $S^t_b \in [0, 1]$ and $\ell^t \in [0, \infty]^A$ for any $b \in \cB^s \cup \cB^e$, $a \in [A]$ and $t \in [T]$. Let $p^t$ be given as in Algorithm \ref{algorithm:time-selection-swap} with learning rate $\eta \in (0,\infty)$. Then we have
\begin{align*}
    \sup_{\psi \in \Psi^s} \sum_{t=1}^T S^t_b \Big( \<p^t, {\ell}^t\> - \<\psi \circ p^t, {\ell}^t\> \Big) \le &~ \sum_{t=1}^T \eta \|\ell^t\|_\infty S^t_b \< p^t, \ell^t \> +   \frac{\log [(\vert \cB^s \vert + \vert \cB^e\vert ) \vert \Psi^s \vert]}{\eta}, ~~~ \forall b \in \cB^s, \\
    \sup_{\psi \in \Psi^e} \sum_{t=1}^T S^t_b \Big( \<p^t, {\ell}^t\> - \<\psi \circ p^t, {\ell}^t\> \Big) \le &~ \sum_{t=1}^T \eta \|\ell^t\|_\infty S^t_b \< p^t, \ell^t \> +   \frac{\log [(\vert \cB^s \vert + \vert \cB^e\vert ) \vert \Psi^e \vert]}{\eta}, ~~~ \forall  b \in \cB^e.
\end{align*}
\end{lemma}

\begin{proof}
First, we define some quantities. For $(b, \psi) \in \cI$, 
the loss w.r.t. $\{S^t_b\}_{t\ge 1}$ till time $t$ is defined as 
\begin{align*}
    L^t_b \defeq \sum_{t' = 1}^t S^{t'}_b \< p^{t'}, \ell^{t'} \>, 
\end{align*}
and 
the loss w.r.t. $(\{S^t_b\}_{t\ge 1}, \psi)$ till time $t$ is defined as 
\begin{align*}
    L^t(b, \psi) \defeq \sum_{t' = 1}^t S^{t'}_b \< \psi \circ p^{t'}, \ell^{t'} \>.
\end{align*}
The weight of $(\{S^t_b\}_{t\ge 1}, \psi)$ at the end of time $t$ defined as 
$$
w^t(b, \psi) \defeq ( \vert \Psi^s \vert  \indic{b \in \cB^e} + \vert \Psi^e\vert \indic{b \in \cB^s})\exp\set{\eta \sum_{t'=1}^t\exp(-\eta \|\ell^{t'}\|_\infty) S^{t'}_b \< p^{t'}, \ell^{t'} \> - \eta L^t(b, \psi)},
$$
where $w^0(b, \psi)$ is set as $ \vert \Psi^s \vert  \indic{b \in \cB^e} + \vert \Psi^e\vert \indic{b \in \cB^s}$.
Let $W^t \defeq \sum_{(b, \psi) \in \cI}w^t(b, \psi)$. Then $q^t(b, \psi)$ is equal to $w^{t-1}(b, \psi)/W^{t-1}$. 
Next, we prove a claim (similar to Claim 6 in \cite{khot2008minimizing}): 
$$
W^t \le W^{t-1} \text{, for all}~ t\ge 1.
$$
In fact, by $\exp(-\eta x) \le 1-(1-\exp(-\eta \|\ell^t\|_\infty)) x/\|\ell^t\|_\infty$ and $\exp(\eta x) \le 1+ (\exp(\eta\|\ell^t\|_\infty)-1)x/\|\ell^t\|_\infty$ for any $\eta \in (0, \infty)$ and $x\in [0,\|\ell^t\|_\infty]$, we have
\begin{align*}
    W^t =& \sum_{(b, \psi) \in \cI} w^t(b, \psi) = \sum_{(b, \psi) \in \cI} w^{t-1}(b, \psi) \exp\set{\eta S^t_b \<\exp(-\eta \|\ell^t\|_\infty) p^t - \psi \circ p^t, \ell^t\>} \\
    \le & \sum_{(b, \psi) \in \cI} w^{t-1}(b, \psi) \paren{1- \frac{(1-\exp(-\eta \|\ell^t\|_\infty)) S^t_b}{\|\ell^t\|_\infty} \< \psi \circ p^t, \ell^t \>} \cdot \paren{ 1+ \frac{ (1-\exp(-\eta\|\ell^t\|_\infty)) S^t_b }{\|\ell^t\|_\infty} \< p^t, \ell^t \> } \\
    \overset{(i)}{\le} &~ W^{t-1} - \frac{1-\exp(-\eta \|\ell^t\|_\infty)}{\|\ell^t\|_\infty}W^{t-1} \sum_{(b, \psi) \in \cI} q^t(b, \psi) S^t_b\< \psi \circ p^t, \ell^t \> \\
    & ~~~ + \frac{1-\exp(-\eta \|\ell^t\|_\infty)}{\|\ell^t\|_\infty} W^{t-1}  \sum_{(b, \psi) \in \cI} q^t(b, \psi) S^t_b\< p^t, \ell^t \>\\
    \overset{(ii)}{=} & W^{t-1}.
\end{align*}
Here, (i) follows from $q^t(b, \psi) = w^{t-1}(b, \psi) / W^{t-1}$ and $\ell^t \in [0,1]^A$; (ii) uses the fact that $p^t$ solves $${p^t}^\top = {p^t}^\top \paren{ \frac{\sum_{(b, \psi) \in \cI} S^t_b q^t(b, \psi) M_\psi}{\sum_{(b, \psi) \in \cI} S^t_b q^t(b, \psi)} }$$ in line \ref{line:solve-equation} in the algorithm, which gives $$\sum_{(b, \psi) \in \cI} q^t(b, \psi) S^t_b p^t = \sum_{(b, \psi) \in \cI} q^t(b, \psi) S^t_b(\psi \circ p^t).$$ This finish the proof of the claim. 

The claim means $W^t$ is non-increasing, so for all $(b,\psi) \in \cI$, 
\begin{align*}
   &( \vert \Psi^s \vert  \indic{b \in \cB^e} + \vert \Psi^e\vert \indic{b \in \cB^s}) \exp\set{\eta \sum_{t'=1}^t\exp(-\eta \|\ell^{t'}\|_\infty) S^{t'}_b \< p^{t'}, \ell^{t'} \> - \eta L^t(b, \psi)} \\
   =&~ w^t(b, \psi) \le \sum_{(b', \psi') \in \cI} w^0(b', \psi') =  \vert \Psi^s \vert \vert \Psi^e \vert (\vert \cB^e \vert + \vert \cB^s \vert ),
\end{align*}
which gives 
\begin{align*}
    \sum_{t'=1}^t\exp(-\eta \|\ell^{t'}\|_\infty) S^{t'}_b \< p^{t'}, \ell^{t'} \> - L^t(b, \psi) \le  \frac{\log [\vert \Psi^s \vert  (\vert \cB^e \vert + \vert \cB^s \vert )]}{\eta},~~~~~ \forall b \in \cB^s, \\
    \sum_{t'=1}^t\exp(-\eta \|\ell^{t'}\|_\infty) S^{t'}_b \< p^{t'}, \ell^{t'} \> - L^t(b, \psi) \le  \frac{\log [\vert \Psi^e \vert (\vert \cB^e \vert + \vert \cB^s \vert )]}{\eta},~~~~~ \forall b \in \cB^e, \\
\end{align*}
Note that we have $1 \le \exp{(-\eta \|\ell^t\|_\infty)} + \eta \|\ell^t\|_\infty$. So we can get that, for any $b \in \cB^s$, 
\begin{align*}
     L^t(b) - L^t(b, \psi)  =& ~\sum_{t=1}^T S^t_b \< p^t, \ell^t \> - L^t(b, \psi)  \\
     \le &~ \sum_{t=1}^T\exp(-\eta \|\ell^t\|_\infty) S^t_b \< p^{t'}, \ell^{t'} \> - L^t(b, \psi) +  \sum_{t=1}^T \eta \|\ell^t\|_\infty S^t_b \< p^t, \ell^t \> \\
     \le & ~\sum_{t=1}^T \eta \|\ell^t\|_\infty S^t_b \< p^t, \ell^t \> +   \frac{\log [\vert \Psi^s \vert   (\vert \cB^e \vert + \vert \cB^s \vert )]}{\eta}.
\end{align*}
Note that the left side is exactly $\sum_{t=1}^T S^t_b \Big( \<p^t, {\ell}^t\> - \<\psi \circ p^t, {\ell}^t\> \Big)$. Consequently,
\begin{align*}
   \sum_{t=1}^T S^t_b \Big( \<p^t, {\ell}^t\> - \<\psi \circ p^t, {\ell}^t\> \Big) \le &~ \sum_{t=1}^T \eta \|\ell^t\|_\infty S^t_b \< p^t, \ell^t \> +   \frac{\log [\vert \Psi^s \vert   (\vert \cB^e \vert + \vert \cB^s \vert )] }{\eta}, ~~~~ \forall b \in \cB^s. 
\end{align*}
We have similar results for $b \in \cB^e$. Taking supremum over all $\psi \in \Psi^e$ or $\psi \in \Psi^s$ proves the lemma. 
\end{proof}

\subsection{Minimizing stochastic wide-range regret}

In this subsection, we consider stochastic loss functions, i.e. we can only observe the stochastic loss estimators $\wt\ell^t$. Moreover, we suppose for each $b \in \cB^e \cup \cB^s$ and $t$, there is a corresponding loss estimator $\wt\ell^t_b$. 

\begin{algorithm}[h]
  \caption{Stochastic Wide Range Regret Minimization with Hedge (\wrhedgebandit)}
  \label{algorithm:time-selection-sample}
  \begin{algorithmic}[1]
    \REQUIRE Learning rate $\eta>0$ and $\wb L$. Swap index set $\cB^s$ and external index set $\cB^e$; Swap strategy modification set $\Psi^s$ and external index set $\Psi^e$.
    \STATE Initialize $\cI \defeq (\cB^s\times \Psi^s) \cup (\cB^e\times \Psi^e)$, $S^0_b \setto 0$ for all $b \in \cB^s \cup \cB^e$,
    $$q^0(b, \psi) \setto \frac{\vert \Psi^s \vert  \indic{b \in \cB^e} + \vert \Psi^e\vert \indic{b \in \cB^s}}{\sum_{(b', \psi') \in \cI}\brac{\vert \Psi^s \vert  \indic{b' \in \cB^e} + \vert \Psi^e\vert \indic{b' \in \cB^s}}}$$
    for all $(b, \psi) \in \cI$.
    \FOR{iteration $t=1,\dots,T$}
    \STATE (\observeTSF) Receive time selection functions $\{ S^t_b \}_{b \in \cB^s \cup \cB^e}$. 
    \STATE Update distribution over $(b, \psi) \in \cI$: 
    \begin{align*}
      q^t((b,\psi)) \propto q^{t-1}((b, \psi)) \exp\Big\{\eta\exp(-\eta \wb L) S^{t-1}_b\langle p^{t-1}, \wt\ell^{t-1}_b \rangle - \eta S^{t-1}_b\langle \psi \circ p^{t-1}, \wt\ell^{t-1}_b \rangle \Big\}. 
    \end{align*}
    \STATE \label{line:solve-equation2} Set $p^t\in \Delta([A])$ as a solution to the equation ${p^t}^\top = {p^t}^\top \paren{ \frac{\sum_{(b, \psi) \in \cI} S^t_b q^t(b, \psi) M_\psi}{\sum_{(b, \psi) \in \cI} S^t_b q^t(b, \psi)} }$.
    \IF{\recommend~is called}
    \STATE Output the vector $p^t$.
    \ENDIF
    \STATE (\observeLoss) Receive loss vectors $\{ \wt {\ell}^t_b \}_{b\in\cB^s \cup \cB^e}$ (where $\E\paren{\wt \ell^t_b|\cF_{t-1}}= \ell^t$ doesn't depend on $b$).  
    \label{line:observe-bandit-loss}
    \ENDFOR
  \end{algorithmic}
\end{algorithm}



We present our algorithm Stochastic Wide-Range Regret Minimization with Hedge (\swrhedge) in Algorithm~\ref{algorithm:time-selection-sample}.
At each round $t$, there is a $\sigma$-field $\cF_{t-1}$ which is generated by all the random variables before the $t$-th round. So along with the execution of the algorithm, we have a filtration $\set{\cF_t}_{t\ge 0}$.


\begin{lemma}[Wide-range regret bound for \swrhedge]
\label{lemma:regret-with-time-selection-bandit}
Let $\{ \ell^t \}_{t \in [T]}$, $\{ M^t_b\}_{b \in \cB^s \cup \cB^e, t \in [T]}$, and $\{ w_b \}_{b \in \cB^s \cup \cB^e}$ be  arbitrary arrays of loss functions, time selection functions, and weighting functions. Let $0 < \oL < \infty$ be a parameter (that will serve as an upper bound of all $ w_b M^t_b \| \wt\ell^t_b \|_\infty$). Assume that (i) $M^t_b \ge 0$, $w_b > 0$,  and $\wt\ell^t_b \in [0, \oL/(w_b
M^t_b)]^A$ for any $b \in \cB^s \cup \cB^e$, $a \in [A]$ and $t \in [T]$; (ii) $\E\brac{\wt\ell^t_b \vert \cF_{t-1}} = \ell^t$ for all $b\in \cB^s \cup \cB^e$. Let $p^t$ be given as in Algorithm \ref{algorithm:time-selection-sample} with learning rate $\eta \in (0,\infty)$ and time selection function $\{S^t_b \}_{b \in \cB^s \cup \cB^e} = \{ w_bM^t_b \}_{b \in \cB^s \cup \cB^e}$.Then with probability at least $1-p$, we have
\begin{align*}
    \sup_{\psi \in \Psi^s} \sum_{t=1}^T M^t_b \Big( \<p^t, {\wt\ell}^t_b\> - \<\psi \circ p^t, \wt{\ell}^t_b\> \Big)  \le &~ \sum_{t=1}^T \eta \wb L M^t_b \< p^t, \wt\ell^t_b \> +   \frac{\log [(\vert \cB^s \vert + \vert \cB^e\vert ) \vert \Psi^s \vert/p]}{\eta w_b}, ~~~ \forall b \in \cB^s, \\
    \sup_{\psi \in \Psi^e} \sum_{t=1}^T M^t_b \Big( \<p^t, {\wt\ell}^t_b\> - \<\psi \circ p^t, \wt{\ell}^t_b\> \Big)  \le &~ \sum_{t=1}^T \eta \wb L M^t_b \< p^t, \wt\ell^t_b \> +   \frac{\log [(\vert \cB^s \vert + \vert \cB^e\vert ) \vert \Psi^e \vert/p]}{\eta w_b}, ~~~ \forall  b \in \cB^e.
\end{align*}
\end{lemma}

\begin{proof}
First, we define some quantities. For $(b, \psi) \in \cI$, 
the loss w.r.t. $\{S^t_b\}_{t\ge 1}$ till time $t$ is defined as 
\begin{align*}
    L^t(b) \defeq \sum_{t' = 1}^t S^{t'}_b \< p^{t'}, \wt\ell^{t'}_b \>, 
\end{align*}
and 
the loss w.r.t. $(\{S^t_b\}_{t\ge 1}, \psi)$ till time $t$ is defined as 
\begin{align*}
    L^t(b, \psi) \defeq \sum_{t' = 1}^t S^{t'}_b \< \psi \circ p^{t'}, \wt\ell^{t'}_b \>.
\end{align*}
The weight of $(\{S^t_b\}_{t\ge 1}, \psi)$ at the end of time $t$ defined as 
$$
w^t(b, \psi) \defeq ( \vert \Psi^s \vert  \indic{b \in \cB^e} + \vert \Psi^e\vert \indic{b \in \cB^s})\exp\set{\eta \exp(-\eta \wb{L}) L^t(b) - \eta L^t(b, \psi)},
$$
where $w^0(b, \psi)$ is set as $ \vert \Psi^s \vert  \indic{b \in \cB^e} + \vert \Psi^e\vert \indic{b \in \cB^s}$. 
Let $W^t \defeq \sum_{(b, \psi) \in \cI}w^t(b, \psi)$. Then $q^t(b, \psi)$ is equal to $w^{t-1}(b, \psi)/W^{t-1}$. 

Next, we prove a claim:
$$
\E\brac{W^t \vert \cF_{t-1}} \le W^{t-1} \text{, for all}~ t\ge 1.
$$
In fact, by $\exp(-\eta x) \le 1-(1-\exp(-\eta \wb{L})) x/\wb{L}$ and $\exp(\eta x) \le 1+ (\exp(\eta\wb{L})-1)x/\wb{L}$ for any $\eta \in (0, \infty]$ and $x\in [0,\wb{L}]$, we have
\begin{align*}
    W^t =& \sum_{(b, \psi) \in \cI} w^t(b, \psi) = \sum_{(b, \psi) \in \cI} w^{t-1}(b, \psi) \exp\set{\eta S^t_b \<\exp(-\eta \wb{L}) p^t - \psi \circ p^t, \wt\ell^t_b\>} \\
    \le & \sum_{(b, \psi) \in \cI} w^{t-1}(b, \psi) \paren{1- \frac{(1-\exp(-\eta \wb{L})) S^t_b}{\wb L} \< \psi \circ p^t, \wt\ell^t_b \>} \cdot \paren{ 1+ \frac{ (1-\exp(-\eta\wb{L})) S^t_b }{\wb L} \< p^t, \wt\ell^t_b \> } \\
    \overset{(i)}{\le} &~ W^{t-1} - \frac{1-\exp(-\eta \wb{L})}{\wb{L}} W^{t-1} \sum_{(b, \psi) \in \cI} q^t(b, \psi) S^t_b\< \psi \circ p^t, \wt\ell^t_b \> \\
    & ~~~ + \frac{1-\exp(-\eta \wb{L})}{\wb{L}} W^{t-1}  \sum_{(b, \psi) \in \cI} q^t(b, \psi) S^t_b\< p^t, \wt\ell^t_b \>.
\end{align*}
Here, (i) follows from $q^t(b, \psi) = w^{t-1}(b, \psi) / W^{t-1}$ and $S^t_b\|\wt\ell^t_b\|_\infty \le \wb{L}$ for any $b$. Using the above inequality, $\E\paren{\wt \ell^t_b|\cF_{t-1}}= \ell^t$ for any $b$, and the fact that $p^t$ solves $${p^t}^\top = {p^t}^\top \paren{ \frac{\sum_{(b, \psi) \in \cI} S^t_b q^t(b, \psi) M_\psi}{\sum_{(b, \psi) \in \cI} S^t_b q^t(b, \psi)} }$$ in line \ref{line:solve-equation2} in the algorithm, which gives $$\sum_{(b, \psi) \in \cI} q^t(b, \psi) S^t_b p^t = \sum_{(b, \psi) \in \cI} q^t(b, \psi) S^t_b(\psi \circ p^t),$$ we have 
\begin{align*}
    \E \brac{W^t \vert \cF_{t-1}} \le W^{t-1}.
\end{align*}
Taking expectation and using the tower property of conditional expectation yields that
\begin{align*}
    \E \brac{W^t} \le W^0.
\end{align*}
Therefore, by Markov inequality, we have with probability at least $1-p$ that
\begin{align*}
    W^t \le W^0/p.
\end{align*}
On this event, we have for all $(b,\psi) \in \cI$ that
\begin{align*}
   &( \vert \Psi^s \vert  \indic{b \in \cB^e} + \vert \Psi^e\vert \indic{b \in \cB^s}) \exp\set{\eta  \exp(-\eta \wb{L}) L^t(b) - \eta L^t(b, \psi)} \\
   =&~  w^t(b, \psi) \le W^t \le W^0/p \le \sum_{(b', \psi') \in \cI} w^0(b', \psi')/p =  \vert \Psi^s \vert \vert \Psi^e \vert (\vert \cB^e \vert + \vert \cB^s \vert )/p.
 \end{align*}
As a result, 
\begin{align*}
    \exp\set{\eta \exp(-\eta \wb{L}) L^t(b) - \eta L^t(b, \psi)} \le  \vert \Psi^s \vert  (\vert \cB^e \vert + \vert \cB^s \vert )/p,~~~~~ \forall b \in \cB^s, \\
    \exp\set{\eta \exp(-\eta \wb{L}) L^t(b) - \eta L^t(b, \psi)} \le  \vert \Psi^e \vert (\vert \cB^e \vert + \vert \cB^s \vert )/p,~~~~~ \forall b \in \cB^e.
\end{align*}
Note that we have $1 \le \exp{(-\eta \wb{L})} + \eta \wb{L}$. So we can get that, for any $b \in \cB^s$, 
\begin{align*}
     L^t(b) - L^t(b, \psi)  \le &~ \exp(-\eta \wb{L}) L^t(b) - L^t(b, \psi) +    \eta \wb{L} L^t(b) \\
     \le & ~ \frac{\log [\vert \Psi^s \vert   (\vert \cB^e \vert + \vert \cB^s \vert )/p]}{\eta} + \eta \wb{L} L^t(b).
\end{align*}
Note that the left side is exactly $\sum_{t=1}^T S^t_b \Big( \<p^t, {\wt\ell}^t_b\> - \<\psi \circ p^t, \wt{\ell}^t_b\> \Big)$. Consequently,
\begin{align*}
    \sum_{t=1}^T S^t_b \Big( \<p^t, {\wt\ell}^t_b\> - \<\psi \circ p^t, \wt{\ell}^t_b\> \Big) \le &~ \sum_{t=1}^T \eta \wb{L} S^t_b \< p^t, \wt\ell^t_b \> +   \frac{\log [\vert \Psi^s \vert   (\vert \cB^e \vert + \vert \cB^s \vert )/p] }{\eta}, ~~~~ \forall b \in \cB^s. 
\end{align*}
Because $S^t_b = M^t_b w_b$, dividing by $w_b$ gives that 
\begin{align*}
    \sum_{t=1}^T M^t_b \Big( \<p^t, {\wt\ell}^t_b\> - \<\psi \circ p^t, \wt{\ell}^t_b\> \Big) \le &~ \sum_{t=1}^T \eta \wb{L} M^t_b \< p^t, \wt\ell^t_b \> +   \frac{\log [\vert \Psi^s \vert   (\vert \cB^e \vert + \vert \cB^s \vert )/p] }{\eta w_b}, ~~~~ \forall b \in \cB^s.
\end{align*}
We have similar results for $b \in \cB^e$. Taking supreme over all $\psi \in \Psi^e$ or $\psi \in \Psi^s$ proves the lemma. 
\end{proof}


%% file: new_version/properties.tex
\section{Properties of the game}

\subsection{Basic properties}

Given the sequence-form transitions $p_{1:h}$ as in Eq. (\ref{eqn:sequence-form-transition}) and the sequence-form policies of the opponents $\{ \pi_{j, 1:h} \}_{j \neq i}$ as in Eq. (\ref{eqn:sequence-form-policy}), we define the marginal reaching probability $p_{1:h}^{\pi_{-i}} (s_h)$ and $p_{1:h}^{\pi_{-i}} (x_{i,h})$ as follows:
\begin{align}\label{eqn:pi-i(x)}
p_{1:h}^{\pi_{-i}}(s_h) &= p_{1:h}(s_h) \prod_{j\in [m], j\not= i} \pi_{j,1:h}(s_h, a_{j,h}), \\
p_{1:h}^{\pi_{-i}}(x_{i,h}) &= \sum_{s_h\in x_{i,h}}p_{1:h}^{\pi_{-i}}(s_h).
\end{align}

\begin{lemma}[Properties of $p^{\pi_{-i}}_{1:h}(x_h)$]
  \label{lemma:pnu}
  The following holds for any $\pi_{-i}=\set{\pi_j}_{j \neq i} \in \otimes_{j\neq i}\Pi_j$:
  \begin{enumerate}[label=(\alph*)]
  \item For any policy $\pi_i \in \Pi_i$, we have
    \begin{align*}
      \sum_{(x_h, a_h)\in\cX_{i, h}\times \cA_i} \pi_{i, 1:h}(x_h, a_h) p^{\pi_{-i}}_{1:h}(x_h) = 1.
    \end{align*}
  \item $0\le p^{\pi_{-i}}_{1:h}(x_h)\le 1$ for all $h \in [H], x_h \in \cX_{i, h}$.
  \end{enumerate}
\end{lemma}
\begin{proof}
  For (a), notice that
  \begin{align*}
    & \quad \pi_{i, 1:h}(x_h, a_h)p^{\pi_{-i}}_{1:h}(x_h) = \sum_{s_h\in x_h} p_{1:h}(s_h) \cdot \pi_{i, 1:h}(x_h, a_h) \cdot \prod_{j \neq i} \pi_{j, 1:h-1}(x_{j, h}(s_{h-1}), a_{j, h-1}) \\
    & = \sum_{s_h\in x_h} \P^{\pi_i, \pi_{-i}}\paren{{\rm visit}~(s_h, a_h)} = \P^{\pi_i, \pi_{-i}}\paren{{\rm visit}~(x_h, a_h)}.
  \end{align*}
  Summing over all $(x_h, a_h)\in\cX_{i, h}\times \cA_i$, the right hand side sums to one, thereby showing (a).

  For (b), fix any $x_h\in\cX_{i, h}$. Clearly $p^{\pi_{-i}}_{1:h}(x_h)\ge 0$. Choose any $a_h\in\cA_i$, and choose policy $\pi_i^{x_h, a_h}\in\Pi_i$ such that $\pi^{x_h, a_h}_{i, 1:h}(x_h, a_h)=1$ (such $\pi_i^{x_h, a_h}$ exists, for example, by deterministically taking all actions prescribed in infoset $x_h$ at all ancestors of $x_h$). For this $\pi_i^{x_h, a_h}$, using (a), we have
  \begin{align*}
    p^{\pi_{-i}}_{1:h}(x_h) = \pi^{x_h, a_h}_{i, 1:h}(x_h, a_h) \cdot p^{\pi_{-i}}_{1:h}(x_h) \le \sum_{(x_h', a_h')\in\cX_{i,h}\times \cA_i} \pi^{x_h, a_h}_{i, 1:h}(x_h', a_h') \cdot p^{\pi_{-i}}_{1:h}(x_h') = 1.
  \end{align*}
  This shows part (b).
\end{proof}

\begin{corollary}
\label{cor:averge_loss_bound}
For any policy $\pi_i \in\Pi_i$ and $h \in [H]$, we have
$$
\sum_{(x_h, a_h)\in\cX_{i,h}\times \cA_i} \pi_{i, 1:h}(x_h, a_h) \ell_{i,h}^t(x_h, a_h) \le 1.
$$
\end{corollary}
\begin{proof}
  Notice by definition
  $$
   \ell^t_h(x_h, a_h) = \sum_{s_h\in x_h, (a_{j, h})_{j \neq i} \in \otimes_{j \neq i} \cA_{j}} p_{1:h}(s_h) \prod_{j \neq i} \pi^t_{j, 1:h}(x_{j, h}(s_h), a_{j, h}) (1 - r_h(s_h, \ba_h)) \le p^{\pi_{-i}}_{1:h}(x_h),
  $$
  and the result is implied by Lemma~\ref{lemma:pnu} (b).
\end{proof}


\begin{lemma}
  \label{lemma:counterfactual-loss-bound}
  For any $h\in[H]$, the counterfactual loss function $\L_{i, h}^t$ defined in~\eqref{equation:L-definition} satisfies the bound
  \begin{enumerate}[label=(\alph*)]
  \item For any policy $\pi_i\in\Pi_i$, we have
    \begin{align*}
      \sum_{(x_h, a_h)\in\cX_{i,h}\times\cA_i} \pi_{i, 1:h}(x_h, a_h) \L_{i,h}^t(x_h, a_h) \le H-h+1.
    \end{align*}
  \item For any $(h, x_h, a_h)$, we have
  \begin{align*}
    0\le \L_{i,h}^t(x_h, a_h) \le p^{\pi_{-i}^t}_{1:h}(x_h)\cdot (H-h+1).
  \end{align*}
\end{enumerate}
\end{lemma}
\begin{proof}
  Part (a) follows from the fact that
  \begin{align*}
    \sum_{(x_h, a_h)\in\cX_{i,h}\times\cA_i} \pi_{i, 1:h}(x_h, a_h) \L_{i,h}^t(x_h, a_h) = \E_{\pi_i, \pi_{-i}^t}\brac{ \sum_{h'=h}^H r_{h'} } \le H-h+1,
  \end{align*}
  where the first equality follows from the definition of the loss functions $\l_h$ and $\L_h$ in~\eqref{equation:l-definition},~\eqref{equation:L-definition}.

  For part (b), the nonnegativity follows clearly by definition. For the upper bound, take any policy $\pi_i^{x_h, a_h}\in\Pi_i$ such that $\pi^{x_h, a_h}_{i, 1:h}(x_h, a_h)=1$. We then have
  \begin{align*}
    & \quad \L_{i,h}^t(x_h, a_h) = \pi^{x_h, a_h}_{i, 1:h}(x_h, a_h) \L_{i,h}^t(x_h, a_h) = \E_{\pi_i^{x_h, a_h}, \pi_{-i}^t}\brac{ \indic{{\rm visit}~x_h, a_h} \cdot \sum_{h'=h}^H r_{h'} } \\
    & = \P_{\pi_i^{x_h, a_h}, \pi_{-i}^t}\paren{{\rm visit}~x_h, a_h} \cdot \E_{\pi_i^{x_h, a_h}, \pi_{-i}^t}\brac{ \sum_{h'=h}^H r_{h'} \bigg| {\rm visit}~x_h, a_h} \\
    & \le \pi^{x_h, a_h}_{i, 1:h}(x_h, a_h) p^{\pi_{-i}^t}_{1:h}(x_h) \cdot (H-h+1) = p^{\pi_{-i}^t}_{1:h}(x_h) \cdot (H-h+1).
  \end{align*}
  This proves the lemma. 
\end{proof}


\subsection{Balanced exploration policy} \label{appendix:balanced-exploration-policy}

Here we collect properties of the balanced exploration policy $\pi^{\star,h}_i$ (cf. Definition~\ref{definition:balanced-policy}). Most results below have appeared in~\citep[Appendix C.2]{bai2022near} in the two-player zero-sum setting; For consistency of notation here we present them in our setting of multi-player general-sum IIEFGs. 

We begin by providing an interpretation of the balanced exploration policy $\pi^{\star, h}_{i, 1:h}$: its inverse $1/\pi^{\star, h}_{i, 1:h}$ can be viewed as the (product) of a ``transition probability'' over the game tree for the $i$'th player. 

For any $1 \le h \le H$ and $1 \le k \le h - 1$, define $p^{\star, h}_{i, k}(x_{k+1} \vert x_k, a_k) = \vert \cC_h (x_{k+1}) \vert  / 
\vert \cC_h(x_{k}, a_k) \vert$ (we use the convention that $\vert \cC_h(x_h) \vert = 1$). By this definition, $p^{\star, h}_{i, k}(\cdot \vert x_k, a_k)$ is a probability distribution over $\cC_h(x_k, a_k)$ and can be interpreted as a balanced transition probability from $(x_k, a_k)$ to $x_{k+1}$. The sequence-form of this balanced transition probability takes the form
\begin{align}\label{eqn:balanced_transition}
p^{\star, h}_{i, 1:h}(x_h) = \frac{\vert \cC_h(x_1) \vert}{X_{i, h}} \prod_{k=1}^{h-1} p^{\star, h}_{i, k}(x_{k+1} \vert x_k, a_k) = \frac{\vert \cC_h(x_1) \vert}{X_{i,h}} \prod_{k = 1}^{h-1} \frac{\vert \cC_h (x_{k+1}) \vert}{\vert \cC_h(x_{k}, a_k) \vert}. 
\end{align}

\begin{lemma}\label{lem:balancing_transition_relation}
For any $(x_h, a_h) \in \cX_{i,h} \times \cA_i$, the sequence form of the transition $p^{\star, h}_{i, 1:h}(x_h)$ and the sequence form of balanced exploration policy $\pi^{\star, h}_{i, 1:h}(x_h, a_h)$ are related by 
\begin{equation}\label{eqn:balanced_strategy_transition_relation}
p^{\star, h}_{i, 1:h}(x_h) = \frac{1}{X_{i,h} A_i \cdot \pi^{\star, h}_{i, 1:h}(x_h, a_h)}. 
\end{equation}
Furthermore, for any $i$-th player's policy $\pi_i \in \Pi_i$ and any $h \in [H]$, we have 
\begin{equation}\label{eqn:balanced_transition_unity}
\sum_{(x_h, a_h) \in \cX_{i,h} \times \cA_i} \pi_{i, 1:h}(x_h, a_h) p^{\star, h}_{i, 1:h} (x_h) = 1. 
\end{equation}
\end{lemma}

\begin{proof}
By the definition of the balanced transition probability as in Eq. (\ref{eqn:balanced_transition}) and the balanced exploration policy as in Eq. (\ref{eqn:balanced-strategy}), we have 
\[
\frac{1}{X_{i,h} A_i \cdot \pi^{\star, h}_{i, 1:h}(x_h, a_h)} = \frac{1}{X_{i,h} A_i} \prod_{k = 1}^{h-1} \frac{\vert \cC_h(x_k) \vert}{\vert \cC_h(x_k, a_k) \vert} \times A_i = \frac{\vert \cC_h(x_1) \vert}{X_{i,h}} \prod_{k = 1}^{h-1} \frac{\vert \cC_h(x_{k+1}) \vert}{\vert \cC_h(x_k, a_k) \vert} = p^{\star, h}_{i, 1:h}(x_h). 
\]
where the second equality used the property that $\vert \cC_h(x_h)\vert = 1$. This proves Eq. (\ref{eqn:balanced_strategy_transition_relation}). The proof of Eq. (\ref{eqn:balanced_transition_unity}) is similar to the proof of Lemma \ref{lemma:pnu} (a). 
\end{proof}

\begin{lemma}[Balancing property of $\pi_i^{\star, h}$]
  \label{lemma:balancing}
  For any \ith player's policy $\pi_i \in\Pi_i$ and any $h\in[H]$, we have
  \begin{align*}
    \sum_{(x_h, a_h)\in \mc{X}_{i, h}\times \mc{A}_i} \frac{\pi_{i, 1:h}(x_h, a_h)}{\pi^{\star, h}_{i, 1:h}(x_h, a_h)} = X_{i,h}A_i.
  \end{align*}
\end{lemma}
\begin{proof}
Lemma~\ref{lemma:balancing} follows as a direct consequence of Eq. (\ref{eqn:balanced_strategy_transition_relation}) and (\ref{eqn:balanced_transition_unity}) in Lemma \ref{lem:balancing_transition_relation}.
\end{proof}


Lemma~\ref{lemma:balancing} states that $\pi_i^{\star, h}$ is a good exploration policy in the sense that the distribution mismatch between it and \emph{any} $\pi_i\in\Pi_i$ has bounded $L_1$ norm. Further, the bound $X_{i,h}A_i$ is non-trivial---For example, if we replace $\pi^{\star, h}_{i, 1:h}$ with the uniform policy $\pi^{\rm unif}_{i, 1:h}(x_h, a_h)=1/A_i^h$, the left-hand side can be as large as $X_{i,h}A_i^h$ in the worst case.

%% file: new_version/relation-proof.tex
\section{Proofs for Section \ref{section:def}}\label{section:appendix-relation}

\subsection{Proof of Proposition \ref{proposition:trigger}}
\label{appendix:proof-trigger}
\begin{proof}
  It suffices to consider all trigger policy $\triggerpolicy$ where $\hat\pi_i$ is a pure policy (at each infoset $x_{h'}$, $\hat\pi_i$ chooses action $\hat\pi_i(x_{h'})$ deterministically.). We prove the two claims separately. 

\noindent
{\bf Step 1.} We first show that $\triggergap(\wb{\pi})\le \oneefcegap(\wb{\pi})$ for any correlated policy $\wb{\pi}$.  We first claim that, for any trigger policy ${\sf trig}(\pi_i, \hat\pi_i, (x_{i,h}^\star, a^\star))$ where $\pi_i, \hat \pi_i \in \Pi_i$, $x_{i,h}^\star \in \cX_{i, h}$, and $a^\star \in \cA_i$, there exists an $\oneefce$ strategy modification $\phi^\star \in \Phi_i^1$ such that, for any opponent's policy $\pi_{-i} \in \Pi_{-i}$, we have
\[
V_i^{{\sf trig}(\pi_i, \hat\pi_i, (x_{i,h}^\star, a^\star))\times\pi_{-i}} = V_i^{(\phi^\star \circ \pi_i) \times\pi_{-i}}.
\]
Given this claim, for any correlated policy $\wb\pi$, we have as desired
\begin{align*}
    \triggergap(\wb{\pi}) = & \max_{i\in[m]} \max_{(x_i,a) \in \cX_i \times \cA_i}
    \max_{\hat\pi_i \in \Pi_i}\paren{ 
    \E_{\pi\sim\bar\pi} V_i^{\triggerpolicy \times\pi_{-i}} - \E_{\pi\sim\bar\pi} V_i^{\pi}
    }\\
    \le & \max_{i\in[m]} \max_{\phi \in \Phi_{i}^1}
    \paren{ 
    \E_{\pi\sim\bar\pi} V_i^{(\phi\circ\pi_{i})\times\pi_{-i}} - \E_{\pi\sim\bar\pi} V_i^{\pi}
    } = \oneefcegap(\wb{\pi}).
\end{align*}

To prove such a claim, we can choose the $\oneefce$ strategy modification $\phi^\star$ to be the following: (1) At any Type-I {\infoseq} with infoset $x = x_{i,h}^\star$, $\phi^\star$ swaps $a_\star$ to $\hat \pi_i(x_{i, h}^\star)$ and swaps $a$ to $a$ for any $a \neq a_\star$; (2) At any Type-I {\infoseq} with infoset $x$ such that $x \neq x_{i, h}^\star$ and $x \not \succ (x_{i, h}^\star, a^\star)$, $\phi^\star$ does not swap the recommended action (swap the recommended action to itself); (3) At any Type-I {\infoseq} and Type-II {\infoseq} with infoset $x \succ (x_{i, h}^\star, a^\star)$, $\phi^\star$ chooses action $\hat \pi_i(x)$ (no matter seeing recommendation or not); (4) For any {\infoseq} that does not fall into the above categories, $\phi^\star$ can be arbitrarily defined since those {\infoseqs} will not be encountered by the design of $\phi^\star$ as above. It is easy to see that such an $\oneefce$ strategy modification $\phi^\star$ applied on any $\pi_i$ implements the trigger policy ${\sf trig}(\pi_i, \hat\pi_i, (x_{i,h}^\star, a^\star))$ so that their value functions are equal. This proves the claim. 

\noindent
{\bf Step 2.} We next show that $\oneefcegap(\wb{\pi}) \le \max_{i \in [m]} X_i A_i \cdot \triggergap(\wb{\pi})$ for any correlated policy $\wb{\pi} \in \Delta(\Pi)$.
For any $1 \mhyphen \efce$ strategy modification $\phi \in \Phi_i^1$ and any $\pi_i \in \Pi_i$, by decomposing the first time step $h$ in which $\phi\circ\pi_i(x_i) \not= \pi_i(x_i)$, we have the decomposition of identity
{\proofsize
\begin{align*}
    1 = &~ \sum_{h = 1}^H\sum_{x_{h} \in \cX_{i,h}} \sum_{a_h \in \cA_i} \indic{x_{h} ~ \text{is visited},~a_{1:h}~\text{are recommended before,}~ \text{and}~ \phi(x_{h}, a_{1:h}) \not= a_h}  \\
    &~ + \sum_{x_{H} \in \cX_{i,H}} \sum_{a_H \in \cA_i} \indic{x_{H} ~ \text{is visited},~a_{1:H}~\text{are recommended before,}~  \text{and}~ \phi(x_{H}, a_{1:H}) = a_H}.
\end{align*}
}As a consequence, for any $\phi \in \Phi_i^1$, $\pi_i \in \Pi_i$ and $\pi_{-i} \in \Pi_{-i}$, we have
{\proofsize
\begin{align*}
    &~ V_i^{(\phi \circ \pi_i) \times \pi_{-i}} - V_i^\pi =  (\E_{(\phi \circ \pi_i) \times \pi_{-i}} - \E_{\pi}) \brac{\sum_{k = 1}^H r_{i, k}}\\
    = &~ (\E_{(\phi \circ \pi_i) \times \pi_{-i}} - \E_{\pi}) \brac{\sum_{h = 1}^H\sum_{x_{h} \in \cX_{i,h}} \sum_{a_h \in \cA_i}  \indic{x_{h} ~ \text{is visited},~a_{1:h}~\text{are recommended before,}~  \text{and}~ \phi(x_{h}, a_{1:h}) \not= a_h} \sum_{k = 1}^H r_{i, k}} \\
    &~ + (\E_{(\phi \circ \pi_i) \times \pi_{-i}} - \E_{\pi}) \brac{\sum_{x_{H} \in \cX_{i,H}} \sum_{a_H \in \cA_i} \indic{x_{H} ~ \text{is visited},~a_{1:H}~\text{are recommended before,}~  \text{and}~ \phi(x_{H}, a_{1:H}) = a_H} \sum_{k = 1}^H r_{i, k}}\\
    = &~ \sum_{h = 1}^H\sum_{x_{h} \in \cX_{i,h}} \sum_{a_h \in \cA_i} (\E_{(\phi \circ \pi_i) \times \pi_{-i}} - \E_{\pi}) \brac{ \indic{x_{h} ~ \text{is visited},~a_{1:h}~\text{are recommended before,}~  \text{and}~ \phi(x_{h}, a_{1:h}) \not= a_h} \sum_{k = 1}^H r_{i, k}},
\end{align*}
}where the last equality used two facts: $(i)$ fixing $x_1, a_1, \cdots, x_H, a_H$, supposing that $\phi(x_{h}, a_{1:h}) = a_h$ for all $h \le H$, then the probability of $x_H$ is visited and $a_{1:H}$ are recommended are the same under $(\phi \circ \pi_i) \times \pi_{-i}$ and $\pi$; $(ii)$ the randomness of $\sum_{k = 1}^H r_{i, k}$ is independent of policy when fixing $x_1, a_1, \cdots, x_H, a_H$. So the second quantity of left hand size of that single equality is zero.

For any $(x_h, a_h) \in \cX_{i, h} \times \cA_i$, $\phi \in \Phi_i^1$ and $\pi_i \in \Pi_i$, we define the trigger policy ${\sf trig}(\pi_i, (\phi \circ \pi_i), (x_{h}, a_h))$ to be a policy that plays $\pi_i$ before triggered by $(x_h, a_h)$ and plays $\phi \circ \pi_i$ after triggered by $(x_h, a_h)$. 
Supposing that $\phi(x_{h'}, a_{1:h'})= a_{h'}$, $\forall h' < h$ and $\phi(x_{h}, a_{1:h}) \not= a_h$, the probability of $x_h$ is visited and $a_{1:h }$ are recommended are the same the same under $(\phi \circ \pi_i) \times \pi_{-i}$ and ${\sf trig}(\pi_i, (\phi \circ \pi_i), (x_h, a_h))  \times \pi_{-i}$, which gives

{\proofsize
\begin{equation}
  \label{equation:element-efce-1ce}
    \begin{aligned}
     & \quad \E_{(\phi \circ \pi_i) \times \pi_{-i}} \brac{\indic{x_{h} ~ \text{is visited},~a_{1:h}~\text{are recommended before,}~ \text{and}~ \phi(x_{h}, a_{1:h}) \not= a_h} \sum_{h = 1}^H r_{i, h}}\\
    & =  \E_{{\sf trig}(\pi_i, (\phi \circ \pi_i), (x_h, a_h))  \times \pi_{-i}} \brac{\indic{x_{h} ~ \text{is visited} ,~a_{1:h}~\text{are recommended before,}~\text{and}~ \phi(x_{h}, a_{1:h}) \not= a_h} \sum_{h = 1}^H r_{i, h}}.
    \end{aligned}
\end{equation}
}


Consequently, we have

{\proofsize
\begin{equation}\label{eqn:proof_containments_1efce_trigger_eq1}
\begin{aligned}
     &~ \E_{\pi\sim \wb\pi}\paren{ V_i^{(\phi \circ \pi_i) \times \pi_{-i}} - V_i^\pi }\\
     = &~ \sum_{h = 1}^H\sum_{x_{h} \in \cX_{i,h}} \sum_{a_h \in \cA_i} \E_{\pi\sim \wb\pi} \Big(\E_{(\phi \circ \pi_i) \times \pi_{-i}} - \E_{\pi} \Big) \brac{ \indic{x_{h} ~ \text{is visited},~a_{1:h}~\text{are recommended before,}~\text{and}~ \phi(x_{h}, a_{1:h}) \not= a_h} \sum_{h = 1}^H r_{i, h}} \\
     \stackrel{(i)}{=} & ~ \sum_{h = 1}^H\sum_{x_{h} \in \cX_{i,h}}\sum_{a_h \in \cA_i} \indic{\phi(x_{h}, a_{1:h}) \not= a_h}\\
     &~ \times \E_{\pi\sim \wb\pi}  \Big(\E_{{\sf trig}(\pi_i, (\phi \circ \pi_i), (x_h, a_h)) \times \pi_{-i}} - \E_{\pi}\Big)\brac{\indic{x_{h} ~ \text{is visited}~and~a_{1:h}~\text{are recommended before}}\sum_{h = 1}^H r_{i, h} }\\
     \stackrel{(ii)}{=}&~\sum_{h = 1}^H\sum_{x_{h} \in \cX_{i,h}}\sum_{a_h \in \cA_i} \indic{\phi(x_{h}, a_{1:h}) \not= a_h} \E_{\pi\sim \wb\pi}  \Big(\E_{{\sf trig}(\pi_i, (\phi \circ \pi_i), (x_h, a_h)) \times \pi_{-i}} - \E_{\pi}\Big)\brac{\sum_{h = 1}^H r_{i, h} }\\
    \stackrel{(iii)}{\le}&~ \sum_{h = 1}^H\sum_{x_{h} \in \cX_{i,h}}  \sum_{a_h \in \cA_i} \triggergap(\wb\pi) = X_i A_i \cdot \triggergap(\wb\pi)
\end{aligned}
\end{equation}
}Here in (i) we used equation (\ref{equation:element-efce-1ce}); in (iii) we bound the indicator by $1$ and use the fact that $\triggergap$ is non-negative (by the observation that in the definition (\ref{equation:trigger-def}), we can choose $x_i$ to be some leaf infoset $x_{i,H}$ and choose $\hat \pi_i(x_{i,H}) = a$ so that $\triggerpolicy = \pi_i$); in (ii) we use the fact that ${\sf trig}(\pi_i, (\phi \circ \pi_i), (x_h, a_h))$ and $\pi$ are identical on any infoset $x$ such that $x \neq x_{h}$ and $x \not \succ (x_h, a_h)$, so that

{\proofsize
\[
\begin{aligned}
&~\Big(\E_{{\sf trig}(\pi_i, (\phi \circ \pi_i), (x_h, a_h)) \times\pi_{-i}} - \E_{\pi} \Big)\brac{\indic{a_{1:h}~\text{are not recommended before} ~ \text{or}~x_{h} ~ \text{is not visited}}\sum_{h = 1}^H r_{i, h} } = 0.
\end{aligned}
\]
}Finally, take supermum over $\phi \in \Phi_i^1$ and then take supermum over $i \in [m]$, we get
\begin{align*}
    \oneefcegap(\wb\pi) = \max_{i \in [m]} \max_{\phi \in \Phi_i^1}\E_{\pi\sim \wb\pi}\paren{ V_i^{(\phi \circ \pi_i) \times \pi_{-i}} - V_i^\pi }  \le \max_{i \in [m]} X_i A_i \cdot \triggergap(\wb\pi).
\end{align*}
This proves the lemma. 
\end{proof}

\subsection{Proof of Proposition \ref{proposition:containment}}
\label{appendix:proof-containment}



\begin{proof}
We prove the two claims separately as follows.

    
\paragraph{Proof of $\kefcegap(\wb{\pi}) \le \kpefcegap(\wb{\pi})$} We claim that, for any $K\ge 0$ and strategy modification $\phi \in \Phi_i^{K}$, there exists $\phi' \in \Phi_i^{K + 1}$ such that for any policy $\pi_i \in \Pi_i$, we have that $\phi \circ \pi_i$ and $\phi' \circ \pi_i$ gives the same policy. Given this claim, we have
\begin{align*}
    \max_{\phi \in \Phi_{i}^K}
    \E_{\pi\sim\bar\pi} V_i^{(\phi\circ\pi_{i})\times\pi_{-i}}
    \le \max_{\phi \in \Phi_{i}^{K+1}}
    \E_{\pi\sim\bar\pi} V_i^{(\phi\circ\pi_{i})\times\pi_{-i}}.
\end{align*}
This implies $ \kefcegap(\wb{\pi}) \le (K+1)\mhyphen{\rm EFCEGap}(\wb{\pi}) $ for all $K \ge 0$. 

To prove such a claim, we can choose the strategy modification $\phi' \in \Phi_{i}^{K+1}$ to be the following: (1) For any {\infoseq} $(x_{i, h}, b_{1: h - 1}) \in \omegai \cap \Omega_i^{({\rm I}), K+1}$ and any action $a_h \in \cA_i$, we set $\phi'(x_{i, h}, b_{1: h-1}, a_h) = \phi(x_{i, h}, b_{1: h - 1}, a_h)$; (2) For any {\infoseq} $(x_{i, h}, b_{1: h - 1}) \in \Omega_i^{({\rm I}), K+1} \setminus \omegai$ and any action $a_h \in \cA_i$, we set $\phi'(x_{i, h}, b_{1: h - 1}, a_h) = \phi(x_{i, h}, b_{1: h_{\star\star}})$ where $h_{\star \star} = \inf \{ k \le h: \sum_{h''=1}^{k} \indic{a_{h''} \not= b_{h''}} = K \}$; (3) For any {\infoseq} $(x_{i, h}, b_{1: h'}) \in \Omega_i^{({\rm II}), K+1}$, we set $\phi'(x_{i, h}, b_{1: h'})= \phi(x_{i, h}, b_{1: h_{\star \star}})$ where $h_{\star\star} = \inf \{ k \le h': \sum_{h''=1}^{k} \indic{a_{h''} \not= b_{h''}} = K \}$. It is easy to see that for any $\pi_i \in \Pi_i$, we have that $\phi' \circ \pi_i$ is the same as $\phi \circ \pi_i$. This proves the claim.

\paragraph{Example of a game and a $\wb{\pi}$ with $\kefcegap(\wb{\pi}) = 0$ but $\kpefcegap(\wb{\pi}) \ge 1/3$}~

For any $K \ge 0$, we consider a two-player game with $H = K+1$ steps and perfect information. The action spaces are $\cA_1 = \set{1, 2}$ for the first player and $\cA_2 = \set{1, 2}$ for the second player in each time step. The state space $\mc{S} = \cup_{h = 1}^{H} \cS_h$ can be identified as $\cS_h =  \cA_1^{h-1}\times \cA_2^{h-1}$ for $h = 1, 2, \dots, K+1$ and both players' infosets are the same as the state space $\cX_1 = \cX_2 = \cS$. If action $(a_h, b_h)$ is taken at $s_h = (a_{1:h-1}, b_{1:h-1}) \in \cS_h$, then the environment will transit to the next state given by $s_{h + 1} = (a_{1: h}, b_{1: h})$. The reward for the second player is always $0$ at every time step. We design the reward for the first player (denoting as $r_h$ in short) as following:
    \begin{itemize}
        \item The reward $r_h(\cdot, \cdot) = 0$ when $h\le K$ for every state actions.
        \item The reward at $s_{K+1} = (a_{1:K+1}, b_{1:K+1})$ is defined as \[
        \begin{aligned}
            r_{K+1}(s_{K+1}, a_{K+1}, b_{K+1}) =&~ \indic{a_1 \not= b_1, \dots, a_{K+1} \not= b_{K+1}} \\
            &~+ \frac{1}{2}\cdot \indic{a_1 = b_1, \dots, a_{K+1} = b_{K+1}}.
        \end{aligned}
        \]
    \end{itemize}
    
Let $\Pi_\star = \{ (\pi_\star, \pi_\star) : \pi_\star \in \Pi_1\}$ where $\Pi_1$ is the set of pure policies of the first player. That means, $\Pi_\star$ is the set of pure policies such that two players take the same action (either $1$ or $2$) at each state. We define $\bar{\pi}$ as the uniform distribution over such a policy space $\Pi_\star$. We claim that $\kefcegap(\bar{\pi}) = 0$ but $\kpefcegap(\bar{\pi}) \ge 1/3$. Since the reward of the second player is always $0$, we only need to consider the value function gap of the first player. Note that the value function of the first player for the correlated policy $\wb{\pi}$ is $1/2$. 

We first consider the $\kpefcegap$. If the first player deviates from the recommended action in every time step (this is an allowed strategy modification in $\Phi_i^{K+1}$), she can receive reward $1$ so that her received value is $1$. As a consequence, we have 
$$
\kpefcegap(\wb{\pi}) \ge 1 - 1/2 > 1/3.
$$

We then consider the $\kefcegap$. If the first player chooses to deviate at any step, she need to play a different action from the second player at all time steps to receive an reward $1$, otherwise she will receive reward $0$. However, she is only allowed to see the recommendation $K$ times. There is at least one time step such that she cannot see the recommendation and she need to guess what is the second player's action. The probability that her guess coincides with the other player's action is $1/2$ no matter how she guess. So by deviating from the recommended action, the first player can receive a value at most $1/2$. That means, $\kefcegap(\wb{\pi}) \le 1/2 - 1/2 = 0$. This finishes the proof of the Proposition.
\end{proof}

\subsection{Relationship with other equilibria}
\newcommand{\purepi}{{\Pi^v}}
\label{appendix:relationship}

Throughout this subsection, in order to introduce the definition of NFCE, we reload the definition of a correlated policy to be a probability measure on all \emph{pure product policies} instead of general product policies. We let $\Pi_i^{\pure}$ denote the set of all possible pure policies for player $i$. Note that this does not affect our definition of $\kefce$ introduced in Section~\ref{section:def}. 

We first present the definitions of Normal-Form Correlated Equilibria ($\nfce$)  and Normal-Form Coarse Correlated Equilibria ($\nfcce$) (from e.g.~\citep{farina2020coarse}). For consistency with our $\kefce$ definition, we define both equilibria through defining their set of strategy modifications.  

\begin{definition}[$\nfce$~ strategy modification]
\label{definition:nfce-strategy-modification}
A \nfce ~strategy modification $\phi$ (for the \ith player) is a mapping $\phi(\cdot, \cdot) : \mc{X}_i \times \Pi_i^{\pure} \to \mc{A}_i$ . Let $\Phi_i^{\nfce}$ denote the set of all possible \nfce ~strategy modification for the \ith player.
For any $\phi\in \Phi_i^{\nfce}$, and any pure policy $\pi_i \in \Pi_i^{\pure}$, we define the modified policy $\phi \circ \pi_i$ as following: at infoset $x_{i,h}$, the modified policy $\phi \circ \pi_i$ takes action $\phi(x_{i, h}, \pi_i)$.
\end{definition}


\begin{definition}[$\nfcce$~ strategy modification]
A $\nfcce$ ~strategy modification $\phi$ (for the \ith player) is a mapping $\phi(\cdot ) : \mc{X}_i  \to \mc{A}_i$. Let $\Phi_i^{\nfcce}$ denote the set of all possible $\nfcce$ strategy modification for the \ith player.
For any $\phi\in \Phi_i^{\nfcce}$, and any pure policy $\pi_i \in \Pi_i^\pure$, we define the modified policy $\phi \circ \pi_i$ as following: at infoset $x_{i,h}$, the modified policy $\phi \circ \pi_i$ take action $\phi(x_{i, h} )$.
\end{definition}

At a high level, $\nfce$ has the ``strongest" form of strategy modifications, which can observe the entire pure policy $\pi_i$ (i.e. full set of recommendations on every infoset). $\nfcce$ has the ``weakest" form of strategy modifications, which cannot observe any recommendation at all (so that each $\phi\in\Phi_i^{\nfcce}$ is equivalent to a pure policy in $\Pi_i^\pure$).

\begin{definition}[$\nfce$ and $\nfcce$]
An $\epsilon$-approximate $\{\nfce$, $ \nfcce\}$ of a POMG is a correlated policy $\bar\pi$ such that 
\begin{align*}
    \set{ \nfce , \nfcce}{\rm Gap}(\wb{\pi}) \defeq \max_{i\in[m]} \max_{\phi \in \Phi_{i}^\set{\nfce,  \nfcce}}
    \paren{ 
    \E_{\pi\sim\bar\pi} V_i^{(\phi\circ\pi_{i})\times\pi_{-i}} - \E_{\pi\sim\bar\pi} V_i^{\pi}
    }
    \le \epsilon.
\end{align*}
We say $\bar\pi$ is an (exact) $\set{\nfce, \nfcce}$ if the above holds with $\epsilon=0$.
\end{definition}



\begin{proposition}[Relationship between $\kefce$ and $\nfce$, $\nfcce$]
\label{prop:nfce-nfcce-relation}
For any correlated policy $\wb\pi $, we have
\begin{enumerate}[label=(\alph*)]
    \item $\inftyefcegap(\wb{\pi})\le \nfcegap(\wb{\pi})$, i.e. $\nfce$ is stricter than $\inftyefce$ ($\nfcegap(\wb{\pi})\le \eps$ implies $\inftyefcegap(\wb{\pi})\le \eps$). 
    
    Further, the converse bound does not hold even if multiplicative factors are allowed: there exists a game with $H=2$ and a correlated policy $\wb{\pi}$ for which
    \begin{align*}
        \inftyefcegap(\wb{\pi})=0~~~{\rm but}~~~\nfcegap(\wb{\pi})> 1/20.
    \end{align*}
    \item $\zeroefcegap(\wb{\pi})= \nfccegap(\wb{\pi})$, i.e. $\zeroefce$ is equivalent to $\nfcce$.
\end{enumerate}
\end{proposition}
\begin{proof}
(a) We first show that $\inftyefcegap(\wb{\pi})\le \nfcegap(\wb{\pi})$. Indeed, for any $\inftyefce$ strategy modification $\phi \in \Phi_i^{\inftyce}$, we let $\phi' \in \Phi_i^\nfce$ such that for any $x_{i,h}\in\cX_i$ and $\pi_i\in\Pi_i^\pure$, $\phi'(x_{i, h}, \pi_i)\defeq \phi(x_{i, h}, b_{1:h-1})$ where $b_{1:h-1}\defeq (\pi_i(x_{i,1}), \dots, \pi_i(x_{i,h-1}))$ and $x_{i,1}\preceq \dots\prec x_{i,h-1}\prec x_{i,h}$ are the unique history of infosets leading to $x_{i,h}$.
By comparing the execution of $\phi\circ\pi$ (cf. Algorithm~\ref{algorithm:phi-circ-pi}) and $\phi'\circ \pi$ (cf. Definition~\ref{definition:nfce-strategy-modification}), the policy $\phi \circ \pi_i$ exactly implements (i.e. is the same as) $\phi' \circ \pi_i$. This gives 
\begin{align*}
    \inftyefcegap(\wb{\pi}) = \max_{\phi \in \Phi_{i}^\inftyce}
    \E_{\pi\sim\bar\pi} V_i^{(\phi\circ\pi_{i})\times\pi_{-i}}
    \le \max_{\phi \in \Phi_{i}^\nfce}
    \E_{\pi\sim\bar\pi} V_i^{(\phi\circ\pi_{i})\times\pi_{-i}} = \nfcegap(\wb{\pi}).
\end{align*}

We next prove the second claim (converse bound does not hold), by constructing the following example.

\begin{example}[$\inftyce$ is not $\epsilon$-approximate $\nfce$ with $\eps = 1/20 $]
We consider a two-player game with $2$ steps and perfect information. The infosets for both players is $\cX_{1, 1} = \cX_{2, 1} = \set{s_0}$ and $\cX_{2, 1} = \cX_{2, 2} = \set{s_{1,1}, s_{1,2}, s_{2,1}, s_{2,2}}$. The action spaces are $\cA_1 = \set{a_1, a_2}$ and $\cA_2 = \set{b_1, b_2}$. If action pair $(a_i, b_j)$ ($i, j = 1, 2$) is chosen at $s_0$, $s_{i,j}$ would be reached with probability $1$. The reward for the second player is always $0$. And we design the reward for the first player as following:
\begin{itemize}
    \item The reward at $s_0$ depends only on the action of the first player: the reward is $1/2$ if $a_1$ is chosen and $0$ if $a_2$ is chosen.
    \item The rewards at $s_{1,1}$ and $s_{1,2}$ are always $0$. The rewards at $s_{2,1}$ and $s_{2,2}$ depends only on the action of the second player: the rewards are all $1$ if $b_1$ is chosen and $0$ if $b_2$ is chosen.
\end{itemize}
Suppose $\bar{\pi}$ is the uniform distribution of all the deterministic policies that takes $(a_1, b_1)$ or $(a_2, b_2)$ at each infosets (there are $2^5 = 32$ such policies). We would verify that $\bar{\pi}$ is a $\inftyce$ but not a $1/20 \mhyphen \nfce$. Since the reward of the second player is always $0$, we only need to consider the first player.

On one hand, the first player can only modify his action at $s_0$ to $a_2$ if he observes that the recommendation at $s_0$, $s_{2,1}$ and $s_{2,2}$ are all $a_1$ (which happens iff the recommendation for his opponent are all $b_1$). Otherwise, he does not change his action. Only $\frac{1}{8}$ of the deterministic policies are modified and for each deterministic policy $\pi_i$ that are modified, the value function of $\pi_i$ are increased by $\frac{1}{2}$. So using this $\nfce$ strategy modification, the first player's value function are increased by $\frac{1}{16}$, which gives that
\begin{align*}
    \nfcegap(\wb{\pi}) \ge 1/16 > 1/20.
\end{align*}

On the other hand, for any $\inftyce$ strategy modification, taking $a_2$ at $s_0$ always has utility $1/2$ since the actions taken by the second player at $\cX_{2}$ are all uniformly distributed  conditional on the recommendation at $s_0$. The utility of taking $a_1$ at $s_0$ is also $1/2$. This means that any $\inftyce$ strategy modification of $\bar\pi$ has value function $1/2$, so does $\bar\pi$. Consequently, $\wb{\pi}$ is an exact $\inftyefce$, i.e. $\inftyefcegap(\wb{\pi})=0$.
\end{example}

(b) Consider $\kefce$ with $K=0$. From the definition of strategy modifications and the executing of modified policy (Algorithm \ref{algorithm:phi-circ-pi}), for $\phi \in \Phi_i^0$ and policy $\pi_i \in \Pi_i^\pure$,  $\phi\circ\pi_i$ takes action $\phi(x_i, \emptyset)$ at $x_i$. So $\phi$ is equivalent to a modification $\phi' \in \Phi_i^{\nfcce}$ which satisfies $\phi'(x_i) = \phi(x_i, \emptyset)$ for all $x_i \in \cX_i$. Here, the equivalence means that $\phi\circ \pi_i = \phi' \circ \pi_i$ for any $\pi_i \in \Pi_i^\pure$. This gives 
\begin{align*}
    \zeroefcegap(\wb{\pi}) = \max_{\phi \in \Phi_{i}^0}
    \E_{\pi\sim\bar\pi} V_i^{(\phi\circ\pi_{i})\times\pi_{-i}}
    = \max_{\phi \in \Phi_{i}^\nfcce}
    \E_{\pi\sim\bar\pi} V_i^{(\phi\circ\pi_{i})\times\pi_{-i}} = \nfccegap(\wb{\pi}),
\end{align*}
which is the desired result.
\end{proof}


\paragraph{Equivalence between $\inftyce$ and BCE}
Next, we give an (informal) argument of the equivalence between ``behavioral deviations'' considered in~\citep{morrill2021efficient} and our $\inftyce$ strategy modifications. As a consequence, (exact) $\infty$-EFCE is equivalent to the ``Behavioral Correlated Equilibria'' (BCE) of~\citep{morrill2021efficient}.

A ``behavioral deviation'' $\phi$ for one player states that at each infoset, the player can choose from three options: (i) follow the recommendation action, (ii) choose a action without ever seeing the recommendation action, or (iii)
choose an action after seeing the recommendation action. Further, the choice of these three options as well as the action to deviate to may depend on the infoset as well as the recommendation history. This is exactly equivalent to the $\inftyce$ strategy modification defined in Definition~\ref{definition:strategy-modification}.





\section{Properties of $\kefce$ strategy modification}

For any $\phi\in\Phi_i^K$, we define its ``probabilistic'' expression $\mu^\phi$ as follows: For any $a_h\in\cA_i$,
\begin{align*}
    & \mu^\phi_h(a_h|x_{i,h}, b_{1:h-1}, b_h) \defeq \indic{ a_h = \phi(x_{i,h}, b_{1:h-1}, b_h) } && \textrm{for all}~~(x_{i,h}, b_{1:h-1}), b_h \in \omegaih\times \cA_i, \\
    & \mu^\phi_h(a_h|x_{i,h}, b_{1:h'}) \defeq \indic{ a_h = \phi(x_{i,h}, b_{1:h'}) } && \textrm{for all}~~(x_{i,h}, b_{1:h'}) \in \omegaiih.
\end{align*}
In words, $\mu^\phi_h(\cdot|x_{i,h}, b_{1:h-1}, b_h)\in\Delta(\cA_i)$ is the pure policy that takes action $\phi(x_{i,h}, b_{1:h-1}, b_h)$ deterministically, for any Type-I \infoseq~$(x_{i,h}, b_{1:h-1})$ and recommendation $b_h\in\cA_i$; $\mu^\phi_h(a_h|x_{i,h}, b_{1:h'})$ is the pure policy that takes action $\phi(x_{i,h}, b_{1:h'})$ for any Type-II \infoseq~$(x_{i,h}, b_{1:h'})$. For convenience, we abuse notation slightly to let 
\begin{align*}
    \muphi_h(\cdot|x_{i,h}, b_{1:h}) \defeq \mu^\phi_h(\cdot|x_{i,h}, b_{1:h-1}, b_h), ~~~
    \muphi_h(\cdot|x_{i,h}, b_{1:h'}) \defeq \mu^\phi_h(\cdot|x_{i,h}, b_{1:h'}).
\end{align*}
Moreover, we use $\dev(a_{1:k}, b_{1:k})$ to denote the Hamming distance of two action sequences $a_{1:k},b_{1:k}\in\cA_i^k$:
$$
\dev(a_{1:k}, b_{1:k}) \defeq \sum_{h=1}^k \indic{a_h \not= b_h},
$$
and define the following notation as shorthand for the indicator that $a_{1:h}$ and $b_{1:h}$ differs in $\set{\le K-1, K}$ elements:
\begin{align*}
    &\indicDleK \defeq \indic{\dev(a_{1:h-1}, b_{1:h-1}) \le K-1};\\
    &\indicDeqK \defeq \indic{\dev(a_{1:h-1}, b_{1:h-1}) = K}.
\end{align*}


\begin{lemma}\label{lemma:general-balancing-mu}
For any $\phi\in\Phi_i^K$ and any (potentially mixed) policy $\pi_i$ for the \ith player, $\phi\circ \pi_i$ is also a (potentially mixed) policy for the \ith player, with sequence-form expression 
\begin{align*}
    (\phi\circ\pi_i)_{1:h}(x_h, a_h) = & \sum_{b_{1:h}} \indicDleK \prod_{k=1}^{h} \muphi_k(a_{k} \vert x_k, b_{1:k})  \prod_{k = 1}^{h} \pi_i(b_k| x_k) \\
    & +\sum_{b_{1:h}}\indicDeqK \prod_{k=1}^{h} \muphi_k(a_{k} \vert x_k, b_{ 1: k\wedge \tau_K}) \prod_{k = 1}^{\tau_K} \pi_i(b_k | x_k),
\end{align*}
where 
\begin{align}
\label{equation:tauk}
    \tau_K \defeq \inf\Big\{h'\le h-1: \sum_{h''=1}^{h'} \indic{a_{h''}\neq b_{h''}} \ge K \Big\}
\end{align} 
is the time step of the $K$-th deviation, and the event $\{\tau_K\le k\}$ can be determined by $(a_{1:k}, b_{1:k})$ for any $k \ge 1$. 
Further, 
\begin{align*}
    \sum_{(x_h, a_h) \in \cX_{i, h} \times \cA_i} \frac{(\phi \circ \pi_i)_{1: h}(x_h, a_h)}{\pi_{i,1:h}^{\star, h}(x_h, a_h)} = X_{i, h} A_i.
\end{align*}
\end{lemma}

\begin{proof}
Suppose the ancestors of $x_h$ are $x_1 \prec x_2 \prec \cdots \prec x_{h-1} \prec x_h$ and the actions leading to $x_h$ are $a_1, \dots , a_{h-1}$. The sequence-form expression $(\phi \circ \pi_i)(x_h, a_h)$ is the probability of $(\phi \circ \pi_i)$ choose $a_k$ at $x_k$ for all $k \in [h]$. Let the recommended action $b_{k} = \pi_i(x_k)$ for all $k \in [h]$. 

If $\abs{\set{h'\in[h-1]: a_{h'}\neq b_{h'}}} \le K-1$, the conditional probability of $(\phi \circ \pi)$ choose $a_k$ at $x_k$ for all $k \in [h]$ is $\prod_{k=1}^{h} \muphi_k(a_{k} \vert x_k, b_{1:k})$, as the player would always swap the action; 
If $\abs{\set{h'\in[h]: a_{h'}\neq b_{h'}}} \ge K$, the conditional probability of $(\phi \circ \pi)$ choose $a_k$ at $x_k$ for all $k \in [h]$ is $\prod_{k=1}^{h} \muphi_k(a_{k} \vert x_k, b_{1:k\wedge \tau_K})$. So by the law of total probability, we have
\begin{align*}
    (\phi\circ\pi_i)_{1:h}(x_h, a_h) = & \sum_{b_{1:h}} \indicDleK \prod_{k=1}^{h} \muphi_k(a_{k} \vert x_k, b_{1:k})  \prod_{k = 1}^{h} \pi_i(b_k| x_k) \\
     +\sum_{b_{1:h}}&\indic{\abs{\set{h'\in[h]: a_{h'}\neq b_{h'}}} \ge K} \prod_{k=1}^{h} \muphi_k(a_{k} \vert x_k, b_{ 1: k\wedge \tau_K}) \prod_{k = 1}^{h} \pi_i(b_k | x_k).
\end{align*}
Notice that $\prod_{k=1}^{h} \muphi_k(a_{k} \vert x_k, b_{ 1: k\wedge \tau_K})$ only depend on $b_{1: \tau_K}$ and $\sum_{b_{\tau_K: h}} \prod_{k = 1}^{h} \pi_i(b_k | x_k) = \prod_{k = 1}^{\tau_K} \pi_i(b_k | x_k)$, so the second summation admits a simpler form:
\begin{align*}
    &~\sum_{b_{1:h}}\indic{\abs{\set{h'\in[h]: a_{h'}\neq b_{h'}}} \ge K} \prod_{k=1}^{h} \muphi_k(a_{k} \vert x_k, b_{ 1: k\wedge \tau_K}) \prod_{k = 1}^{h} \pi_i(b_k | x_k)\\
    = &~ \sum_{b_{1:h'}: \sum_1^{h'} \indic{a_k\not=b_k} = K ~\text{and}~ b_{h'}\not = a_{h'} }  \prod_{k=1}^{h} \muphi_k(a_{k} \vert x_k, b_{ 1: k\wedge h'}) \prod_{k = 1}^{h'} \pi_i(b_k | x_k)\\
    = &~ \sum_{b_{1:h}}\indicDeqK \prod_{k=1}^{h} \muphi_k(a_{k} \vert x_k, b_{ 1: k\wedge \tau_K}) \prod_{k = 1}^{\tau_K} \pi_i(b_k | x_k).
\end{align*}
The last equality is because we can add $a_{h'+1:h}$ behind $b_{1:h'}$ to get a new $b_{1:h}$ which doesn't change the value of the summation.
Then the first part of this lemma is proved. The second part actually is a direct corollary of Lemma \ref{lemma:balancing}.
\end{proof}


This Lemma has the following corollary.

\begin{corollary}\label{lemma:mu-pi-visit-xih}
For the \ith player and any pure policy $\pi \in \Pi$, fix any $\phi \in \Phi_i^K$ and $x_{h} \in \mc{X}_{i, h}$ with $(a_1, \dots, a_{h-1})$ being the unique history of actions leading to $x_h$. Then the probability that $x_{h}$ is reached by the \ith player under policy $(\phi \circ \pi_i)\times \pi_{-i}$ is
\begin{align*}
    &\P_{\phi\circ\pi_i\times\pi_{-i}}\paren{x_{h} ~ \text{is reached by the i}^{th}~\text{player}}\\
    = & \sum_{b_{1:h-1}} \indicDleK \prod_{k=1}^{h-1} \muphi_k(a_{k} \vert x_k, b_{1:k}) \cdot \prod_{k = 1}^{h-1} \pi_i(b_k | x_k) p^{\pi_{-i}}_{1:h}(x_{i,h})\\
    & +\sum_{b_{1:h-1}}\indicDeqK \prod_{k=1}^{h-1} \muphi_k(a_{k} \vert x_k, b_{ 1: k\wedge \tau_K}) \cdot \prod_{k = 1}^{\tau_K} \pi_i(b_k | x_k) p^{\pi_{-i}}_{1:h}(x_{i,h}).
\end{align*}
\end{corollary}
\begin{proof}
We have\begin{align*}
    &\P_{\phi\circ\pi_i\times\pi_{-i}}\paren{x_{h} ~ \text{is reached by the i}^{th}~\text{player}}\\
    = & \sum_{a\in \cA_i} (\phi \circ \pi_i)_{1:h}(x_h, a) p_{1:h}^{\pi_{-i}}(x_{i,h}),
\end{align*}
where $p_{1:h}^{\pi_{-i}}(x_{i,h})$ is defined in equation (\ref{eqn:pi-i(x)}). So applying Lemma \ref{lemma:general-balancing-mu} yields the desired result.
\end{proof}

As each step $h$,  one (and only one) $x_h \in \cX_{i, h}$ is visited, the summation of the above reaching probability over $x_h$ is $1$. This directly yields the following corollary.
\begin{corollary}\label{lemma:sum=1}
For the \ith player and any policy $\pi \in \Pi$, fix any $\phi \in \Phi_i^K$, we have
\begin{align*}
    & \sum_{x_h \in \cX_{i, h}} \sum_{b_{1:h-1}} \indicDleK \prod_{k=1}^{h-1} \muphi_k(a_{k} \vert x_k, b_{1:k}) \cdot \prod_{k = 1}^{h-1} \pi_i(b_k | x_k) p^{\pi_{-i}}_{1:h}(x_{i,h})\\
    & +\sum_{x_h \in \cX_{i, h}} \sum_{b_{1:h-1}}\indicDeqK \prod_{k=1}^{h-1} \muphi_k(a_{k} \vert x_k, b_{ 1: k\wedge \tau_K}) \cdot \prod_{k = 1}^{\tau_K} \pi_i(b_k | x_k) p^{\pi_{-i}}_{1:h}(x_{i,h})
    =  1.
\end{align*}
\end{corollary}

%% file: new_version/regret-decomposition.tex
\section{Regret decomposition for $\kefce$}
\label{appendix:regret-decomposition}

This section presents the definition of the $\kefce$ regret along with its properties, which will be useful for the proofs of our main results.

Let $\{ \pi^t \}_{t \in [T]}$ be a sequence of policy. Define the $\kce$ regret for the \ith player as
\begin{equation}
  \label{eqn:KCE-regret}
R_{i,K}^T = \max_{\phi \in \Phi_i^K} \sum_{t = 1}^T \Big( V_{i}^{\phi\circ \pi_i^t \times \pi_{-i}^t} - V_{i}^{\pi^t} \Big).
\end{equation}

\begin{lemma}[Online-to-batch for $\kefce$]\label{lem:online-to-batch}
Let $\{ \pi^t= (\pi_i^t)_{i \in [n]} \}_{t \in [T]}$ be a sequence of product policies for all players over $T$ rounds. Then, for the average (correlated) policy $\bar\pi={\rm Unif}(\{\pi^t\}_{t=1}^T)$, we have
\begin{align*}
    \kefcegap(\wb{\pi}) = \max_{i\in [m]} R_{i,K}^T / T.
\end{align*}

\end{lemma}
\begin{proof}
This follows directly by the definition of $\kefcegap$:

\begin{align*}
    \kefcegap(\wb\pi) =& \max_{i\in[m]} \max_{\phi \in \Phi_i^K} \paren{V_{i}^{\phi \circ \wb\pi} - V_{i}^{\wb\pi}}\\
    =& \max_{i\in[m]} \max_{\phi \in \Phi_i^K} \E_{\pi\sim\wb\pi} \brac{V_{i}^{\phi\circ \pi_i \times \pi_{-i}} - V_{i}^{\pi}}\\
    =& \max_{i\in[m]} \max_{\phi \in \Phi_i^K} \frac{1}{T}\sum_{t=1}^T \brac{V_{i}^{\phi\circ \pi_i^t \times \pi_{-i}^t} - V_{i}^{\pi^t}} \\
    =& \max_{i\in[m]} R_{i,K}^T / T.
\end{align*}

\end{proof}

For $(x_{i,h}, b_{1:h-1}, \vphi) \in \omegai \times \Psi^s$, we define the immediate local swap regret as
\begin{equation*}
\hat R_{(x_h, b_{1:h-1}), \vphi}^T \defeq \sum_{t = 1}^T \prod_{k = 1}^{h-1} \pi_i^t (b_{k}|x_{k})  \Big( \< \pi_{i, h}^t(\cdot \vert x_{i, h}) - \vphi \circ \pi_{i, h}^t(\cdot \vert x_{i, h}), \L_{i,h}^{t}(x_{i, h},\cdot) \> \Big).
\end{equation*}
and the (overall) local swap regret as
\begin{align}
\label{equation:local-swap-regret}
    \hat R_{(x_h, b_{1:h-1})}^{T, \rm swap} \defeq \max_{\vphi} \hat R_{(x_h, b_{1:h-1}), \vphi}^T.
\end{align}

For $(x_{i, h}, b_{1:h'}, a) \in \omegaii \times \cA_i$, we define the immediate external regret as
\begin{equation*}
\hat R_{(x_{h},b_{1:h'}), a}^T \defeq \sum_{t = 1}^T \prod_{k = 1}^{h'} \pi_i^t (b_{k}|x_{k})  \Big(  \< \pi_{i, h}^t(\cdot \vert x_{i, h}), L_{i,h}^{t}(x_{i, h},\cdot) \>- \L_{i,h}^{t}(x_{i, h}, a) \Big). 
\end{equation*}
and the (overall) local external regret as
\begin{align}
\label{equation:local-external-regret}
    \hat R_{(x_{h}, b_{1 : h'})}^{T, \rm ext} \defeq \max_{a\in\cA_i} \hat R_{(x_{h}, b_{1 : h'}), a}^T.
\end{align}

\paragraph{$\kefce$ regret decomposition}
Our main result in this section is the following regret decomposition that decomposes the $\kefce$ regret $R_{i,K}^T$ into combinations of local regrets at each \infoseq.

\begin{lemma}[Regret decomposition for $\kefce$]
\label{lem:k-ce-regret_decomposition} 
We have $R_{i, K}^T \leq \sum_{h=1}^H R_h^T$, with 
\[
R_{h}^{T} \defeq \max_{\phi \in \Phi_i^K} \sum_{x_{h}\in\cX_{i,h}}  G_{h}^{T, \swap}(x_{h}; \muphi) + \max_{\phi \in \Phi_i^K} \sum_{x_{h}\in\cX_{i,h}} G_{h}^{T, \ext}(x_{h}; \muphi)  .
\]
Here 
\begin{equation}\label{eqn:G1_in_K_CE}
\begin{aligned}
 G_{h}^{T, \swap}(x_h; \muphi) \equiv&~ \sum_{b_{1:h-1}} \indicDleK \prod_{k=1}^{h-1} \muphi_k(a_{k} \vert x_k, b_{1:k}) \what{R}^{T, \swap}_{(x_h, b_{1:h-1} )}, \\
 \end{aligned}
 \end{equation}
and
\begin{equation}\label{eqn:G2_in_K_CE}
 \begin{aligned}
  G_{h}^{T, \ext}(x_h; \muphi) \equiv&~ \sum_{b_{1:h-1}} \indicDeqK \prod_{k=1}^{h-1} \muphi_k(a_k \vert x_k, b_{1:k \wedge \tau_K}) \what{R}^{T, \ext}_{(x_h, b_{1:\tau_K})}\\
  =&~\sum_{h'=K}^{h-1} \sum_{b_{1:h'}} \indic{\tau_K=h'} \prod_{k=1}^{h-1} \muphi_k(a_k \vert x_k, b_{1:k\wedge h'}) \what{R}^{T, \ext}_{(x_h, b_{1:h'})}.
\end{aligned}
\end{equation}
Above, $a_{1:h-1}$ is the unique sequence of actions leading to $x_h$, and $\tau_K$ (cf. definition in~\eqref{equation:tauk}) depends on $a_{1:h-1}$, $b_{1:h-1}$.
\end{lemma}

Throughout the rest of the proofs, the definitions of the sequence $a_{1:h-1}$ and $\tau_K$ is the same as in the above regret minimizer, i.e. $a_{1:h-1}$ is determined by $x_h$, and $\tau_K$ is determined by $(x_h, b_{1:h-1})$.


\begin{proof}[Proof of Lemma~\ref{lem:k-ce-regret_decomposition}]
We begin by performing the following performance decomposition
{\proofsize \begin{align*}
    & \quad R_{i, K}^T = \max_{\phi \in \Phi_i^K}\sum_{t=1}^T (V_{i}^{\muphi\circ\pi_i^t\times \pi_{-i}^t }-V_i^{\pi^t})\\
    &=\max_{\phi \in \Phi_i^K} \sum_{t=1}^T \left( \E_{\muphi \circ\pi_i^t\times \pi_{-i}^t}\brac{ \sum_{h=1}^H r_{i,h} }-  \E_{\pi^t}\brac{ \sum_{h=1}^H r_{i,h}}
    \right)\\
    &= \max_{\phi \in \Phi_i^K} \sum_{t=1}^T \sum_{h=1}^H \paren{ \E_{((\muphi\circ\pi_i^t)_{1:h}\pi_{i,h+1:H}^t) \times \pi_{-i}^t}\brac{ \sum_{k=1}^H r_{i,k}} - \E_{((\muphi\circ\pi_i^t)_{1:h-1}\pi_{i,h:H}^t) \times \pi_{-i}^t}\brac{ \sum_{k=1}^H r_{i,k}}}  \\
    &= \max_{\phi \in \Phi_i^K} \sum_{t=1}^T \sum_{h=1}^H \paren{ \E_{((\muphi\circ\pi_i^t)_{1:h}\pi_{i,h+1:H}^t) \times \pi_{-i}^t}\brac{ \sum_{k=h}^H r_{i,k}} - \E_{((\muphi\circ\pi_i^t)_{1:h-1}\pi_{i,h:H}^t) \times \pi_{-i}^t}\brac{ \sum_{k=h}^H r_{i,k}}}\\
    &= \max_{\phi \in \Phi_i^K}  \sum_{h=1}^H \sum_{t=1}^T \paren{  \E_{((\muphi\circ\pi_i^t)_{1:h-1}\pi_{i,h:H}^t) \times \pi_{-i}^t}\brac{ \sum_{k=h}^H (1-r_{i,k})}-\E_{((\muphi\circ\pi_i^t)_{1:h}\pi_{i,h+1:H}^t) \times \pi_{-i}^t}\brac{ \sum_{k=h}^H (1-r_{i,k})} }.
\end{align*}}

Here, $((\muphi\circ\pi_i^t)_{1:h}\pi_{i,h+1:H}^t) \times \pi_{-i}^t$ refers to the policy that the \ith player uses $\muphi\circ\pi_i^t$ for the first $h$ step, and then uses $\pi_{i}^t$ where as other players always use $\pi_{-i}^t$. The last step use that fact that $((\muphi\circ\pi_i^t)_{1:h}\pi_{i,h+1:H}^t) \times \pi_{-i}^t$ and $((\muphi\circ\pi_i^t)_{1:h-1}\pi_{i,h:H}^t) \times \pi_{-i}^t$ are the same for the first $h-1$ steps, so the expected reward in first $h-1$ steps are the same, too. 
Therefore, if we let
{\proofsize
$$
\wt{R}_h^T = \max_{\phi \in \Phi_i^K}  \sum_{t=1}^T \paren{  \E_{((\muphi\circ\pi_i^t)_{1:h-1}\pi_{i,h:H}^t) \times \pi_{-i}^t}\brac{ \sum_{k=h}^H (1-r_{i,k})}-\E_{((\muphi\circ\pi_i^t)_{1:h}\pi_{i,h+1:H}^t) \times \pi_{-i}^t}\brac{ \sum_{k=h}^H (1-r_{i,k})} }
$$
}we have $R_{i, K}^T \le  \sum_{h=1}^H \wt{R}_h^T$.

We next show that
{\proofsize
\begin{align*}
    \wt{R}_h^T \le \max_{\phi \in \Phi_i^K} \sum_{x_{h}\in\cX_{i,h}} G_{h}^{T, \swap}(x_{h}; \muphi) + \max_{\phi \in \Phi_i^K} \sum_{x_{h}\in\cX_{i,h}} G_{h}^{T, \ext}(x_{h}; \muphi) = R_h^T,
\end{align*}
}{\normalsize which yields the desired result. Fix any $h \in [H]$ and $\muphi \in \Phi_i^K$, from the execution of modified policy $\muphi\circ \pi_i$ in Algorithm \ref{algorithm:phi-circ-pi},}
{\proofsize
\begin{align*}
    &\E_{((\muphi\circ\pi_i^t)_{1:h}\pi_{i,h+1:H}^t) \times \pi_{-i}^t} \brac{ \sum_{k=h}^H (1-r_{i,k}) } \\
    =~& \sum_{x_{h}\in\cX_{i,h}} \sum_{b_{1: h-1}}  \prod_{k=1}^{h-1} \pi_i^t(b_{k}|x_{k}) \underbrace{\prod_{k=1}^{h-1} \muphi_k (a_{k}|x_k,b_{1:k \wedge \tau_K})}_{\text{probability of taking `right' actions leading to}~ x_{h}} \underbrace{\Big\langle \muphi_h(\cdot|x_h,b_{1:h\wedge \tau_K}), L_{i,h}^{t}(x_h,\cdot) \Big\rangle}_{\text{counterfactual loss}},
\end{align*}
}\ziang{\text{$\phi\circ\pi_h$ comes from $b_h= \pi_i^t(x_h)$}}where we assume $b_h = \pi_i^t(x_{h})$ in $\muphi_k(\cdot|x_h,b_{1:h\wedge \tau_K})$ and $\tau_K \defeq \inf\{h'\le h-1: \sum_{h''=1}^{h'} \indic{a_{h''}\neq b_{h''}} \ge K\}$. Similarly,
{\proofsize \begin{align*}
    &\E_{((\muphi\circ\pi_i^t)_{1:h-1}\pi_{i,h:H}^t) \times \pi_{-i}^t} \brac{ \sum_{k=h}^H (1-r_{i,k}) } \\
    =~& \sum_{x_{h}\in\cX_{i,h}} \sum_{b_{1: h-1}}  \prod_{k=1}^{h-1} \pi_i^t(b_{k}|x_{k}) \underbrace{\prod_{k=1}^{h-1} \muphi_k (a_{k}|x_k,b_{1:k \wedge \tau_K})}_{\text{probability of taking `right' actions leading to}~ x_{h}} \underbrace{\Big\langle \pi_i^t(\cdot|x_h), L_{i,h}^{t}(x_h,\cdot) \Big\rangle }_{\text{counterfactual utility}}.
\end{align*}
}Substituting these into $\wt{R}_h^T$, we have
{\proofsize \begin{align*}
    & \quad \wt{R}_h^T = \max_{\phi \in \Phi_i^K}  \sum_{t=1}^T \paren{  \E_{((\muphi\circ\pi_i^t)_{1:h-1}\pi_{i,h:H}^t) \times \pi_{-i}^t}\brac{ \sum_{k=h}^H (1-r_{i,k})}-\E_{((\muphi\circ\pi_i^t)_{1:h}\pi_{i,h+1:H}^t) \times \pi_{-i}^t}\brac{ \sum_{k=h}^H (1-r_{i,k})} }\\
    & = \max_{\phi \in \Phi_i^K} \sum_{t=1}^T  \sum_{x_{h}\in\cX_{i,h}} \sum_{b_{1:h-1}}\left( \prod_{k=1}^{h-1} \pi_i^t(b_{k}|x_{k}) \prod_{k=1}^{h-1}\muphi_k (a_{k}|x_k,b_{1:k \wedge \tau_K}) \Big\langle \pi_i^t(\cdot|x_h)-\muphi_h(\cdot|x_h,b_{1:h\wedge \tau_K}), L_{i,h}^{t}(x_h,\cdot) \Big\rangle
    \right)\\
    & = \max_{\phi \in \Phi_i^K} \sum_{x_{h}\in\cX_{i,h}} \sum_{b_{1:h-1}}\left(\prod_{k=1}^{h-1}\muphi_k (a_{k}|x_k,b_{1:k \wedge \tau_K}) \sum_{t=1}^T \prod_{k=1}^{h-1} \pi_i^t(b_{k}|x_{k}) \Big\langle \pi_i^t(\cdot|x_h)-\muphi_h(\cdot|x_h,b_{1:h\wedge \tau_K}), L_{i,h}^{t}(x_h,\cdot) \Big\rangle 
    \right).
\end{align*}}

For fixed $\muphi$, $x_h$, based on whether $\dev(a_{1:h-1},b_{1:h-1})\le K-1$ or not, we have 
{\proofsize \begin{align*}
    \wt{R}_h^T &= \max_{\phi \in \Phi_i^K} \sum_{x_{h}\in\cX_{i,h}} \sum_{b_{1:h-1}}\left(\prod_{k=1}^{h-1}\muphi_k (a_{k}|x_k,b_{1:k \wedge \tau_K}) \sum_{t=1}^T \prod_{k=1}^{h-1} \pi_i^t(b_{k}|x_{k}) \Big\langle \pi_i^t(\cdot|x_h)-\muphi_h(\cdot|x_h,b_{1:h\wedge \tau_K}), L_{i,h}^{t}(x_h,\cdot) \Big\rangle
    \right)\\
    &= \max_{\phi \in \Phi_i^K}({\rm I}_h+ {\rm II}_h),
\end{align*}
}where
{\proofsize \begin{align*}
   {\rm I}_h & \defeq   \sum_{x_{h}\in\cX_{i,h}} \sum_{b_{1:h-1}}\indicDleK\prod_{k=1}^{h-1}\muphi_k (a_{k}|x_k,b_{1:k  }) \\
   &\sum_{t=1}^T \prod_{k=1}^{h-1} \pi_i^t(b_{k}|x_{k}) \Big\langle \pi_i^t(\cdot|x_h)-\muphi_h(\cdot|x_h,b_{1:h\wedge \tau_K}), L_{i,h}^{t}(x_h,\cdot) \Big\rangle,
\end{align*}
}and
{\proofsize \begin{align*}
    {\rm II}_h & \defeq  \sum_{x_{h}\in\cX_{i,h}} \sum_{b_{1:h-1}}\indic{\dev(a_{1:h-1},b_{1:h-1})\ge K}\prod_{k=1}^{h-1}\muphi_k (a_{k}|x_k,b_{1:k \wedge \tau_K}) \\ &\sum_{t=1}^T \prod_{k=1}^{h-1} \pi_i^t(b_{k}|x_{k}) \Big\langle \pi_i^t(\cdot|x_h)-\muphi_h(\cdot|x_h,b_{1:h\wedge \tau_K}), L_{i,h}^{t}(x_h,\cdot) \Big\rangle.\\
\end{align*}
}For ${\rm I}_h$,
since $\dev(a_{1:h-1},b_{1:h-1})\le K-1$, at step $h$, the number of deviations is less than $K$ and $\tau_K\ge h$. 

Moreover, max over $\phi \in \Phi_i^K$ can be separated into max over all $\muphi_h(\cdot|x_h,b_{1:h\wedge \tau_K})$, so we have
{\proofsize \begin{align*}
   \max_{\phi \in \Phi_i^K}  {\rm I}_h  = & \max_{\phi \in \Phi_i^K} \sum_{x_{h}\in\cX_{i,h}} \sum_{b_{1:h-1}}\indicDleK\prod_{k=1}^{h-1}\muphi_k (a_{k}|x_k,b_{1:k \wedge \tau_K})\\
   & \times \sum_{t=1}^T \prod_{k=1}^{h-1} \pi_i^t(b_{k}|x_{k}) \Big\langle \pi_i^t(\cdot|x_h)-\muphi_h(\cdot|x_h,b_{1:h\wedge \tau_K}), L_{i,h}^{t}(x_h,\cdot) \Big\rangle\\
    \le & \max_{\phi \in \Phi_i^K} \sum_{x_{h}\in\cX_{i,h}} \sum_{b_{1:h-1}} \indicDleK \prod_{k=1}^{h-1}\muphi_k (a_{k}|x_k,b_{1:k}) \\
    & \times \max_{\varphi} \sum_{t=1}^T \prod_{k=1}^{h-1} \pi_i^t(b_{k}|x_{k}) \Big\langle \pi_i^t(\cdot|x_h)-(\varphi \circ \pi_i^t)(\cdot|x_h), L_{i,h}^{t}(x_h,\cdot) \Big\rangle\\
    \stackrel{(i)}{=} & \max_{\phi \in \Phi_i^K} \sum_{x_{h}\in\cX_{i,h}} \sum_{b_{1:h-1}} \indicDleK \prod_{k=1}^{h-1}\muphi_k (a_{k}|x_k,b_{1:k}) \what{R}^{T, \swap}_{(x_h, b_{1:h-1}) }\\ =& \max_{\phi \in \Phi_i^K} \sum_{x_{h}\in\cX_{i,h}}  G_h^{T, \swap}(x_h).
\end{align*}
}Above, (i) follows by the definition of the local swap regret in~\eqref{equation:local-swap-regret}.
For ${\rm II}_h$, since $\dev(a_{1:h-1},b_{1:h-1})\ge K$, at some step $h' \in [K, h-1]$, the number of deviations was already $K$.  In fact, such $h'$ is $\tau_K$. In this case, $\pi^t_i(x_h)$ cannot be observed. 
for  ${\rm II}_h$, we have
{\proofsize \begin{align*}
    \max_{\phi \in \Phi_i^K}   {\rm II}_h = & \max_{\phi \in \Phi_i^K} \sum_{x_{h}\in\cX_{i,h}} \sum_{b_{1:h-1}}\indic{\dev(a_{1:h-1},b_{1:h-1}) \ge K }\prod_{k=1}^{h-1}\muphi_k (a_{k}|x_k,b_{1:k \wedge \tau_K})\\ 
    &\times \sum_{t=1}^T \prod_{k=1}^{h-1} \pi_i^t(b_{k}|x_{k}) \Big\langle \pi_i^t(\cdot|x_h)-\muphi_h(\cdot|x_h,b_{1:h\wedge \tau_K}), L_{i,h}^{t}(x_h,\cdot) \Big\rangle\\
    =~&\max_{\phi \in \Phi_i^K} \sum_{x_{h}\in\cX_{i,h}} \sum_{b_{1:h-1}}\sum_{h' = K}^{h-1} \indic{\tau_K = h' }\prod_{k=1}^{h-1}\muphi_k (a_{k}|x_k,b_{1:k \wedge h'})\\ 
    & \times \sum_{t=1}^T \prod_{k=1}^{h-1} \pi_i^t(b_{k}|x_{k})  \Big\langle \pi_i^t(\cdot|x_h)-\muphi_h(\cdot|x_h,b_{1:h\wedge h'}), L_{i,h}^{t}(x_h,\cdot) \Big\rangle \\
    =~&\max_{\phi \in \Phi_i^K} \sum_{x_{h}\in\cX_{i,h}} \sum_{h' = K}^{h-1} \sum_{b_{1:h-1}} \indic{\tau_K = h' } \prod_{k=1}^{h-1} \muphi_k (a_{k}|x_k,b_{1:k \wedge h'})\\
    & \times \sum_{t=1}^T \prod_{k=1}^{h-1} \pi_i^t(b_{k}|x_{k}) \Big\langle \pi_i^t(\cdot|x_h)-\muphi_h(\cdot|x_h,b_{1:h\wedge h'}), L_{i,h}^{t}(x_h,\cdot) \Big\rangle\\
    \overset{(i)}{=}~&\max_{\phi \in \Phi_i^K} \sum_{x_{h}\in\cX_{i,h}} \sum_{h' = K}^{h-1} \sum_{b_{1:h'}} \indic{\tau_K = h' }\prod_{k=1}^{h-1}\muphi_k (a_{k}|x_k,b_{1:k \wedge h'})\\
    & \times \sum_{t=1}^T \prod_{k=1}^{h'} \pi_i^t(b_{k}|x_{k}) \Big\langle \pi_i^t(\cdot|x_h)-\muphi_h(\cdot|x_h,b_{1:h\wedge h'}), L_{i,h}^{t}(x_h,\cdot) \Big\rangle\\
    \le~&\max_{\phi \in \Phi_i^K} \sum_{x_{h}\in\cX_{i,h}} \sum_{h' = K}^{h-1} \sum_{b_{1:h'}} \indic{\tau_K = h' }\prod_{k=1}^{h-1}\muphi_k (a_{k}|x_k,b_{1:k \wedge h'}) \\
    &\times \max_a \sum_{t=1}^T \prod_{k=1}^{h'} \pi_i^t(b_{k}|x_{k}) \paren{ L^t_{i,h}(x_h,\pi_{i}^t(x_h)) - L^t_{i,h}(x_h,a)}\\
    \stackrel{(ii)}{=}~&\max_{\phi \in \Phi_i^K} \sum_{h' = K}^{h-1} \sum_{b_{1:h'}} \indic{\tau_K = h' }\prod_{k=1}^{h-1}\muphi_k (a_{k}|x_k,b_{1:k \wedge h'}) \what{R}^{T, \ext}_{(x_{h}, b_{1:h'}), x_h }.
\end{align*}
}Here, (i) uses the fact that for fixed $x_h$, $h'$ and $b_{1:h'}$ with $\tau_K = h'$, there is only one and only one $b_{h'+1:h-1}$ satisfying $\tau_K = h$ and $\sum_{b_{h'+1:h-1}} \prod_{k=h'+1}^{h-1} \pi_i^t(b_{k}|x_{k}) = 1$; (ii) follows by definition of the local external regret in~\eqref{equation:local-external-regret}.
Furthermore, for fixed $x_h$, we can expand $b_{1:h'}$ to $b_{1:h} \defeq (b_{1:h'}, a_{h^{\star }+ 1}, \cdots, a_{h-1})$ such that $\dev(a_{1:h-1},b_{1:h-1}) =K$, then we can rewrite 

{\proofsize \begin{align*}
    \sum_{h' = K}^{h-1} \sum_{b_{1:h'}} \indic{\tau_K = h' }\prod_{k=1}^{h-1}\muphi_k (a_{k}|x_k,b_{1:k \wedge h'}) \what{R}^{T, \ext}_{(x_h, b_{1:h'}), x_h }
\end{align*}
} as
{\proofsize \begin{align*}
    &\sum_{h' = K}^{h-1} \sum_{b_{1:h'}} \indic{\tau_K = h' }\prod_{k=1}^{h-1}\muphi_k (a_{k}|x_k,b_{1:k \wedge h'}) \what{R}^{T, \ext}_{(x_h, b_{1:h'}), x_h }\\
    =& \sum_{h' = K}^{h-1} \sum_{b_{1:h-1}} \indic{\tau_K = h' , \dev(a_{1:h-1},b_{1:h-1}) =K }\prod_{k=1}^{h-1}\muphi_k (a_{k}|x_k,b_{1:k \wedge h'}) \what{R}^{T, \ext}_{(x_h, b_{1:h'}), x_h }\\
    =& \sum_{b_{1:h-1}} \indicDeqK \prod_{k=1}^{h-1} \muphi_k(a_k \vert x_k, b_{1:k \wedge \tau_K}) \what{R}^{T, \ext}_{(x_h, b_{1:\tau_K}), x_h} = G_h^{T, \ext}(x_h).
\end{align*}
}Consequently,
$$
\max_{\phi \in \Phi_i^K} {\rm II}_h \le  \max_{\phi \in \Phi_i^K} \sum_{x_{h}\in\cX_{i,h}} G_h^{T, \ext}(x_h).
$$
Finally, combining the above bounds for $\max_{\phi\in\Phi_i^K} {\rm I}_h$ and $\max_{\phi\in\Phi_i^K} {\rm II}_h$ gives the desired result:
$$
\wt{R}_h^T = \max_{\phi \in \Phi_i^K} ({\rm I}_h + {\rm II}_h ) \le  \max_{\phi \in \Phi_i^K} \sum_{x_h\in \mc{X}_{i,h }} G_h^{T, \swap}(x_h)+\max_{\phi \in \Phi_i^K} \sum_{x_h\in \mc{X}_{i,h }} G_h^{T, \ext}(x_h)  = R_h^T. 
$$
\end{proof}

%% file: new_version/proof-fullfeedback.tex
\section{Proofs for Section~\ref{section:full-feedback}}
\label{appendix:full-feedback}


The rest of this section is devoted to proving Theorem~\ref{theorem:kcfr-full-feedback}.

\subsection{Proof of Theorem \ref{theorem:kcfr-full-feedback}}
\label{appendix:proof-full-feedback}


\begin{proof-of-theorem} [\ref{theorem:kcfr-full-feedback}]
  The proof follows by bounding the $\kefce$ regret 
  \begin{align*}
      R_{i,K}^T = \max_{\phi \in \Phi_i^K} \sum_{t = 1}^T \Big( V_{i}^{\phi\circ \pi_i^t \times \pi_{-i}^t} - V_{i}^{\pi^t} \Big)
  \end{align*}
  for all players $i \in [m]$, and then converting to a bound on $\kefcegap(\wb{\pi})$ by the online-to-batch conversion (Lemma~\ref{lem:online-to-batch}). 
  
  By the regret decomposition for $R_{i,K}^T$ (Lemma~\ref{lem:k-ce-regret_decomposition}), we have $R_{i, K}^T \leq \sum_{h=1}^H R_h^T$, where 
  \[
  R_{h}^{T} \defeq \max_{\phi \in \Phi_i^K} \sum_{x_{h}\in\cX_{i,h}}   G_{h}^{T, \swap}(x_{h}; \muphi) + \max_{\phi \in \Phi_i^K} \sum_{x_{h}\in\cX_{i,h}} G_{h}^{T, \ext}(x_{h}; \muphi) .
  \]
  

The following two lemmas bound two terms 
$$\sum_{h=1}^H \max_{\phi \in \Phi_i^K}\sum_{x_{h}\in \cX_{i, h}}G_{h}^{T, \swap}(x_{h};\muphi)~~\text{and}~~ \sum_{h=1}^H \max_{\phi \in \Phi_i^K}\sum_{x_{h}\in \cX_{i, h}}G_{h}^{T, \ext}(x_{h};\muphi).
$$ Their proofs are presented in Section~\ref{appendix:proof-g1-fullfeedback} \&~\ref{appendix:proof-g2-fullfeedback} respectively.

\begin{lemma}[Bound on summation of $G_h^{T, \swap}(x_{i,h})$ with full feedback]\label{lemma:bound-G_h1-fullfeedback}
If we choose learning rates as 
$$
\eta_{x_h } = \sqrt{{H \choose K\wedge H } X_i  A_i^{K \wedge
 H} \log  A_i / (H^2T) }
$$ for all $x_h \in \cX_{i,h}$ (same with \eqref{equation:learning-rate-full-feedback}).
Then we have
 {\proofsize \begin{gather*}
     \sum_{h=1}^H \max_{\phi \in \Phi_i^K}\sum_{x_{h}\in \cX_{i, h}}G_h^{T, \swap}(x_{h};\muphi)
     \le  
    \sqrt{H^4 {H \choose K\wedge H }X_i A_i^{K \wedge
 H   }  T\log A_i} .
 \end{gather*}
} 
\end{lemma}

\begin{lemma}[Bound on summation of $G_h^{T, \ext}(x_{i,h})$  with full feedback]\label{lemma:bound-G_h2-fullfeedback}
If we choose learning rates as 
$$
\eta_{x_h} = \sqrt{{H \choose K\wedge H } X_i A_i^{K\wedge h} \log A_i / (H^2T)}
$$
for all $x_h \in \cX_{i,h}$ (same with \eqref{equation:learning-rate-full-feedback}). Then we have
 {\proofsize \begin{gather*}
     \sum_{h=1}^H \max_{\phi \in \Phi_i^K}\sum_{x_{h}\in \cX_{i, h}}G_h^{T, \ext}(x_{h};\muphi)
     \le   \mc{O}  \paren{ \sqrt{H^4 {H \choose K\wedge H } X_i A_i^{K \wedge H } T \log A_i } }.
 \end{gather*}
}
\end{lemma}

So by Lemma~\ref{lemma:bound-G_h1-fullfeedback},~\ref{lemma:bound-G_h2-fullfeedback}, we get
{\proofsize
\begin{align*}
    & R_{i, K}^T \le \sum_{h=1}^H R_h^T\\
    \le &~ \sum_{h=1}^H \paren{\max_{\phi \in \Phi_i^K}\sum_{x_{h}\in \cX_{i, h}}G_h^{T, \swap}(x_{h};\muphi) + \max_{\phi \in \Phi_i^K}\sum_{x_{h}\in \cX_{i, h}}G_h^{T, \ext}(x_{h};\muphi)}\\
    \le & ~\mc{O}  \paren{ \sqrt{H^4 {H \choose K\wedge H } X_i A_i^{K \wedge
 H } T \log A_i } } +\mc{O}  \paren{ \sqrt{H^4 {H \choose K\wedge H }  X_i A_i^{K\wedge H} T \log A_i } } \\
 = &~ \cO\paren{ \sqrt{H^4{H\choose K\wedge H} A_i^{K\wedge H}X_iT \log A_i}}.
\end{align*}
}Therefore, we have the regret bound
\begin{align}
    \label{equation:kefce-regret-bound}
    R_{i, K}^T \le \cO\paren{ \sqrt{H^4{H\choose K\wedge H} A_i^{K\wedge H}X_iT \log A_i}}.
\end{align}
As long as
{\proofsize
\begin{align*}
    T \ge \cO\paren{   H^4 {H \choose K\wedge H} \paren{\max_{i\in[m]}X_iA_i^{K \wedge H}}\log A_i / \eps^2 },
\end{align*}
}we have by the online-to-batch lemma (Lemma~\ref{lem:online-to-batch}) that the average policy $\bar\pi={\rm Unif}(\{\pi^t\}_{t=1}^T)$ satisfies

{\proofsize
\begin{align}\label{equation:bound-kceregret}
     \kefcegap(\wb{\pi}) = \frac{\max_{i\in[m]} R_{i,K}^T}{T} \le \max_{i\in[m]} \cO\sqrt{ H^4 {H \choose K\wedge H} \paren{\max_{i\in[m]}X_iA_i^{K \wedge H}}\log A_i / T } \le \eps.
\end{align}
}This proves Theorem~\ref{theorem:kcfr-full-feedback}.
\end{proof-of-theorem}

We remark that the above proof does not depend on the particular choice of $\pi_{-i}^t$, and thus the regret bound~\eqref{equation:kefce-regret-bound} also holds even if we control the \ith player only, and $\pi_{-i}^t$ are arbitrary (potentially adversarial depending on all information before iteration $t$ starts). This directly gives the following corollary.
\begin{corollary}[$\kefce$ regret bound for \kscfr]
\label{corollary:kscfr-regret}
For any $0\le K\le \infty$, $\epsilon \in (0,H]$, suppose the \ith player runs Algorithm \ref{algorithm:kefr-fullfeedback} togethe  against arbitrary (potentially adversarial) opponents $\set{\pi_{-i}^t}_{t=1}^T$, where \Regalg~is instantiated as Algorithm \ref{algorithm:time-selection-swap}~with learning rates specified in (\ref{equation:learning-rate-full-feedback}). Then the \ith player achieves $\kefce$ regret bound:
\begin{align*}
    R_{i,K}^T = \max_{\phi \in \Phi_i^K} \sum_{t = 1}^T \Big( V_{i}^{\phi\circ \pi_i^t \times \pi_{-i}^t} - V_{i}^{\pi^t} \Big) \le \cO\paren{ \sqrt{H^4{H\choose K\wedge H}X_i A_i^{K\wedge H}T\log A_i} }.
\end{align*}
\end{corollary}

\subsection{Proof of Lemma \ref{lemma:bound-G_h1-fullfeedback}}
\label{appendix:proof-g1-fullfeedback}

\begin{proof}
Recall that $G_{h}^{T, \swap}(x_h; \muphi)$ is defined as 
$$
G_{h}^{T, \swap}(x_h; \muphi) \equiv~ \sum_{b_{1:h-1}} \indicDleK \prod_{k=1}^{h-1} \muphi_k(a_{k} \vert x_k, b_{1:k}) \what{R}^{T, \swap}_{(x_h, b_{1:h-1} )}.
$$
For each $h\in [H]$ and $(x_{h}, b_{1: h-1}) \in \omegai$, we have
{\proofsize \begin{align*}
    \hat R_{(x_{h}, b_{1: h-1})}^{T, \rm swap} =& \max_\vphi \sum_{t = 1}^T \prod_{k = 1}^{h-1} \pi_i^t (b_{k}|x_{k}) \Big( \< \pi_{i, h}^t(\cdot \vert x_{i, h}) - \vphi \circ \pi_{i, h}^t(\cdot \vert x_{i, h}), \L_{i,h}^{t}(x_{i, h},\cdot) \> \Big) .
\end{align*}
}

For $x_h\in \cX_{i,h}$, we first apply regret minimization lemma (Lemma \ref{lemma:regret-with-time-selection}) on $\regmin_{x_h}$ to give an upper bound of $\hat R_{(x_{h}, b_{1: h-1})}^{T, \rm swap}$.  The regret minimizer $\regmin_{x_{h}}$ observes time selection functions $S^t_{b_{1:h-1}} = \prod_{k=1}^{h-1} \pi_i^t(b_{k}|x_{k})$ for $(x_h, b_{1:h-1}) \in \omegai$ and losses $L_{i,h}^t(x_{h},\cdot)$. Suppose $\regmin_{x_h}$ uses the learning rate $\eta_{x_h}$,  by the bound on regret with respect to a time selection function index and strategy modification pair (Lemma \ref{lemma:regret-with-time-selection}), we have
{\proofsize \begin{align*}
    &\max_\vphi \sum_{t = 1}^T \prod_{k = 1}^{h-1} \pi_i^t (b_{k}|x_{k}) \Big( \< \pi_{i, h}^t(\cdot \vert x_{i, h}) - \vphi \circ \pi_{i, h}^t(\cdot \vert x_{i, h}), \L_{i,h}^{t}(x_{i, h},\cdot) \> \Big) \\
    \le~& \eta  \sum_{t = 1}^T  \prod_{k = 1}^{h-1} \pi_i^t (b_{k}|x_{k}) \|L_{i,h}^t(x_h,\cdot)\|_\infty \<  \pi_i^t(\cdot \vert x_{h}) , L_{i,h}^{t}(x_{h},\cdot) \> +   \frac{(A_i + H)\log A_i }{\eta} \\
    \le~& \eta H^2 \sum_{t = 1}^T  \prod_{k = 1}^{h-1} \pi_i^t (b_{k}|x_{k}) p_{1:h}^{\pi_{-i}^t}(x_h) +   \frac{(A_i+H) \log A_i }{\eta} .
\end{align*}
}Here, we use (i) $\|L_{i,h}^t(x_h,\cdot)\|_\infty \le H p_{1:h}^{\pi_{-i}^t}(x_h)$, and (ii) $(|\cB^s|+|\cB^e|) |\Psi^s| \le A_i^{H+A_i}$ since the number of \infoseqs~is no more than $A_i^H$ .
Then we can get
{\proofsize \begin{align*}
    & \max_{\phi \in \Phi_i^K}\sum_{x_{h}\in \cX_{i, h}}G_{h}^{T, \swap}(x_{h};\muphi)\\=~& \max_{\muphi \in \Phi_i^K} \sum_{x_{h}\in \cX_{i, h}} \sum_{b_{1:h-1}} \indicDleK \prod_{k=1}^{h-1} \muphi_k(a_{k} \vert x_k, b_{1:k}) \\
    &\times \max_\vphi \sum_{t = 1}^T \prod_{k = 1}^{h-1} \pi_i^t (b_{k}|x_{k}) \Big( L_{i,h}^{t}(x_{h},\pi_i^t(x_{h})) - L_{i,h}^{t}(x_{h}, \vphi \circ \pi_i^t(x_{h}))  \Big) \\
    \le~& \eta H^2 \max_{\muphi \in \Phi_i^K} \sum_{x_{h}\in \cX_{i, h}} \sum_{b_{1:h-1}} \indicDleK \prod_{k=1}^{h-1} \muphi_k(a_{k} \vert x_k, b_{1:k})   \sum_{t=1}^T \prod_{k = 1}^{h-1} \pi_i^t (b_{k}|x_{k}) p_{1:h}^{\pi_{-i}^t}(x_h)  \\
    &+ \frac{(A_i+H) \log A_i }{\eta} \max_{\muphi \in \Phi_i^K} \sum_{x_{h}\in \cX_{i, h}} \sum_{b_{1:h-1}} \indicDleK   \prod_{k=1}^{h-1} \muphi_k(a_{k} \vert x_k, b_{1:k}) .
\end{align*}
}Letting 
{\proofsize\begin{gather*}
    {\rm{I}}_h =  \max_{\muphi \in \Phi_i^K} \sum_{x_{h}\in \cX_{i, h}} \sum_{b_{1:h-1}} \indicDleK  \prod_{k=1}^{h-1} \muphi_k(a_{k} \vert x_k, b_{1:k});\\
    {\rm{II}}_h =  \max_{\muphi \in \Phi_i^K} \sum_{x_{h}\in \cX_{i, h}} \sum_{b_{1:h-1}} \indicDleK \prod_{k=1}^{h-1} \muphi_k(a_{k} \vert x_k, b_{1:k}) 
     \sum_{t=1}^T \prod_{k = 1}^{h-1} \pi_i^t (b_{k}|x_{k}) p_{1:h}^{\pi_{-i}^t}(x_h).
\end{gather*}
}For fixed $\muphi \in \Phi_i^K$ and $x_h\in\mc{X}_{i,h}$, by counting the number of $b_{1:h-1}$ such that $\dev(a_{1:h-1},b_{1:h-1})\le K-1$, we have 
{\proofsize \begin{align*}
    &\sum_{b_{1:h-1}} \indicDleK  {\prod_{k=1}^{h-1} \muphi_k(a_{k} \vert x_k, b_{1:k})}\\
    \le ~ &  \sum_{b_{1:h-1}}\indicDleK  \le {h-1 \choose (K-1) \wedge (h-1)} A^{(K-1)\wedge(h-1)}.
\end{align*}}
Consequently,

{\proofsize \begin{equation}\label{equation:full-feedback-first-term-G1}
    {\rm I}_h \le  X_{i,h} {h-1 \choose (K-1) \wedge (h-1)}A_i^{K \wedge h - 1}.
\end{equation}
}Note that by Corollary \ref{lemma:sum=1}, for fixed $t$, {\proofsize
$$\sum_{x_{h}\in \cX_{i, h}} \sum_{b_{1:h-1}}\indicDleK \prod_{k=1}^{h-1} \muphi_k(a_{k} \vert x_k, b_{1:k}) \prod_{k = 1}^{h-1} \pi_i^t(b_k | x_k) p^{\pi^t_{-i}}_{1:h}(x_{i,h}) \le 1.
$$
}Consequently, we have
{\proofsize
\begin{equation*}
    \begin{gathered}
        \max_{\muphi \in \Phi_i^K} \sum_{x_{h}\in \cX_{i, h}} \sum_{b_{1:h-1}}\indicDleK \prod_{k=1}^{h-1} \muphi_k(a_{k} \vert x_k, b_{1:k}) \sum_{t = 1}^T \prod_{k = 1}^{h-1} \pi_i^t(b_k | x_k) p^{\pi^t_{-i}}_{1:h}(x_{i,h}) \le T.
    \end{gathered}
\end{equation*}
}This yields that
\begin{align}\label{equation:full-feedback-second-term-G1}
    &{\rm II}_h \le T.
\end{align}
Taking summation over $h\in[H]$, we have
{\proofsize \begin{align*}
    &\sum_{h=1}^H    \max_{\phi \in \Phi_i^K}\sum_{x_{h}\in \cX_{i, h}}G_{h}^{T, \swap}(x_{h};\muphi) \le \sum_{h=1}^H(\frac{(A_i+H) \log A_i }{\eta}{\rm I}_{h}+ H^2 \eta{\rm II}_{h})\\
    \le&~H^3 \eta T + \sum_{h=1}^H \frac{ (A_i+H) \log A_i }{\eta}X_{i,h} {h-1 \choose (K-1) \wedge (h-1)}A_i^{K \wedge h-1}\\
    \le &~ H^3 \eta T +  \frac{(A_i+H) \log A_i}{\eta}  X_i {H-1 \choose K\wedge H-1 }A_i^{K\wedge H  - 1}\\
    \le &~ H^3 \eta T +  \frac{H \log A_i}{\eta}  X_i {H \choose K\wedge H }A_i^{K\wedge H }.
\end{align*}}
As we chose $\eta = \sqrt{{H \choose K\wedge H } X_i A_i^{K \wedge
 H  } \log A_i / (H^2T) }$ per~\eqref{equation:learning-rate-full-feedback}, we have 
{\proofsize \begin{align*}
    \sum_{h=1}^H    \max_{\phi \in \Phi_i^K}\sum_{x_{h}\in \cX_{i, h}}G_{h}^{T, \swap}(x_{h};\muphi)
     \le ~\sqrt{H^4 {H \choose K\wedge H }X_i A_i^{K \wedge
 H }  T\log A_i}.
\end{align*}}
\end{proof}

\subsection{Proof of Lemma \ref{lemma:bound-G_h2-fullfeedback}}
\label{appendix:proof-g2-fullfeedback}
\begin{proof}
Recall that $G_{h}^{T, \ext}(x_h; \muphi)$ is defined as
$$
G_{h}^{T, \ext}(x_h; \muphi) \equiv~ \sum_{b_{1:h-1}} \indicDeqK \prod_{k=1}^{h-1} \muphi_k(a_k \vert x_k, b_{1:k \wedge \tau_K}) \what{R}^{T, \ext}_{(x_h, b_{1:\tau_K})}.
$$
Since $(x_{h}, b_{1,\tau_K})\in \omegaii$, we have 
{\proofsize \begin{align*}
    &\hat R_{(x_{h}, b_{1,\tau_K})}^{T, \rm ext} = \max_a\sum_{t = 1}^T \prod_{k = 1}^{\tau_K} \pi_i^t (b_{k}|x_{k})  \Big(  \< \pi_{i, h}^t(\cdot \vert x_{i, h}), L_{i,h}^{t}(x_{i, h},\cdot) \>- \L_{i,h}^{t}(x_{i, h}, a) \Big).
\end{align*}
}For $x_h\in \cX_{i,h}$, we can apply regret minimization lemma (Lemma \ref{lemma:regret-with-time-selection}) on $\regmin_{x_h}$ to give an upper bound of $\hat R_{(x_{h}, b_{1: h_{\tau_K}})}^{T, \rm ext}$.  The regret minimizer $\regmin_{x_{h}}$ observes time selection functions $S^t_{b_{1:h_{\tau_K}}} = \prod_{k=1}^{h_{\tau_K}} \pi_i^t(b_{k}|x_{k})$ for $(x_h, b_{1:h_{\tau_K}}) \in \omegaii$ and losses $L_{i,h}^t(x_{h},\cdot)$. Suppose all $\regmin_{x_h}$ use the same learning rate $\eta$,  by the bound on regret with respect to a time selection function index and strategy modification pair (Lemma \ref{lemma:regret-with-time-selection}), we have
{\proofsize \begin{align*}
    &\max_a\sum_{t = 1}^T \prod_{k = 1}^{\tau_K} \pi_i^t (b_{k}|x_{k})  \Big(  \< \pi_{i, h}^t(\cdot \vert x_{i, h}), L_{i,h}^{t}(x_{i, h},\cdot) \>- \L_{i,h}^{t}(x_{i, h}, a) \Big) \\
    \le~& \eta  \sum_{t = 1}^T  \prod_{k = 1}^{\tau_K} \pi_i^t (b_{k}|x_{k}) \|L_{i,h}^t(x_h,\cdot)\|_\infty \<  \pi_i^t(\cdot \vert x_{h}) , L_{i,h}^{t}(x_{h},\cdot) \> +   \frac{2H\log A_i }{\eta} \\
    \le~& \eta H^2 \sum_{t = 1}^T  \prod_{k = 1}^{\tau_K} \pi_i^t (b_{k}|x_{k}) p_{1:h}^{\pi_{-i}^t}(x_h) +   \frac{ 2H\log A_i }{\eta} .
\end{align*}
}Here, we use (i) $\|L_{i,h}^t(x_h,\cdot)\|_\infty \le H p_{1:h}^{\pi_{-i}^t}(x_h)$, and (ii) $(|\cB^s|+|\cB^e|) |\Psi^e| \le A_i^{H+1}$ since the number of \infoseqs~is no more than $A_i^H$ . Then we can get
{\proofsize \begin{align*}
    & \max_{\muphi \in \Phi_i^K} G_{h}^{T, \ext}(x_h; \muphi)\\
    =~& \max_{\muphi \in \Phi_i^K} \sum_{x_{h}\in \cX_{i, h}} \sum_{b_{1:h-1}} \indicDeqK \prod_{k=1}^{h-1} \muphi_k(a_{k} \vert x_k, b_{1:k \wedge \tau_K})  \\
    &\times \max_a\sum_{t = 1}^T \prod_{k = 1}^{\tau_K} \pi_i^t (b_{k}|x_{k})  \Big(  \< \pi_{i, h}^t(\cdot \vert x_{i, h}), L_{i,h}^{t}(x_{i, h},\cdot) \>- \L_{i,h}^{t}(x_{i, h}, a) \Big) \\
    \le~&  \eta H^2  \max_{\muphi \in \Phi_i^K} \sum_{x_{h}\in \cX_{i, h}} \sum_{b_{1:h-1}} \indicDeqK \prod_{k=1}^{h-1} \muphi_k(a_{k} \vert x_k, b_{1:k\wedge \tau_K}) \cdot \sum_{t=1}^T \prod_{k = 1}^{\tau_K} \pi_i^t (b_{k}|x_{k}) p_{1:h}^{\pi_{-i}^t}(x_h) \\
    &+ \frac{ 2H\log A_i}{\eta} \max_{\muphi \in \Phi_i^K} \sum_{x_{h}\in \cX_{i, h}} \sum_{b_{1:h-1}} \indicDeqK   \prod_{k=1}^{h-1} \muphi_k(a_{k} \vert x_k, b_{1:k \wedge \tau_K}) .
\end{align*}
}Similarly to the proof of Lemma \ref{lemma:bound-G_h1-fullfeedback}, letting 
{\proofsize \begin{gather*}
    {\rm{I}}_h \defeq  \max_{\muphi \in \Phi_i^K} \sum_{x_{h}\in \cX_{i, h}} \sum_{b_{1:h-1}} \indicDeqK  \prod_{k=1}^{h-1} \muphi_k(a_{k} \vert x_k, b_{1:k \wedge h'});\\
    {\rm{II}}_h \defeq  \max_{\muphi \in \Phi_i^K} \sum_{x_{h}\in \cX_{i, h}} \sum_{b_{1:h-1}} \indicDeqK \prod_{k=1}^{h-1} \muphi_k(a_{k} \vert x_k, b_{1:k\wedge h'}) \cdot \sum_{t=1}^T \prod_{k = 1}^{h'} \pi_i^t(b_k | x_k) p_{1:h}^{\pi_{-i}^t}(x_h).
\end{gather*}
}For fixed $\muphi \in \Phi_i^K$ and $x_h\in\mc{X}_{i,h}$, by counting the number of $b_{1:h-1}$ such that $\dev(a_{1:h-1},b_{1:h-1}) = K $, we have 
{\proofsize \begin{align*}
    &\sum_{b_{1:h-1}} \indicDeqK  {\prod_{k=1}^{h-1} \muphi_k(a_{k} \vert x_k, b_{1:k \wedge h'})}\le ~   \sum_{b_{1:h-1}}\indicDeqK  \le {h-1 \choose K\wedge H} A^{K\wedge H }.
\end{align*}
}Consequently,
{\proofsize \begin{equation}\label{equation:full-feedback-first-term-G2}
    {\rm I}_h \le  X_{i,h} {h-1 \choose K \wedge H  }A_i^{K\wedge H}.
\end{equation}
}Note that for fixed $t$, by Corollary \ref{lemma:sum=1},
{\proofsize
$$\sum_{x_{h}\in \cX_{i, h}}\sum_{b_{1:h-1}}\indicDeqK \prod_{k=1}^{h-1} \muphi_k(a_{k} \vert x_k, b_{1:k\wedge h'}) \prod_{k = 1}^{h'} \pi_i^t(b_k | x_k) p^{\pi^t_{-i}}_{1:h}(x_{i,h}) \le 1.$$
}Consequently, we have
{\proofsize
\begin{equation*}
    \begin{gathered}
        \max_{\muphi \in \Phi_i^K} \sum_{x_{h}\in \cX_{i, h}} \sum_{b_{1:h-1}}\indicDeqK \prod_{k=1}^{h-1} \muphi_k(a_{k} \vert x_k, b_{1:k \wedge h'}) \sum_{t = 1}^T\prod_{k = 1}^{h'} \pi_i^t(b_k | x_k)p^{\pi^t_{-i}}_{1:h}(x_{i,h}) \le T.
    \end{gathered}
\end{equation*}
}This yields that
\begin{align}\label{equation:full-feedback-second-term-G2}
    &{\rm II}_h \le T.
\end{align}
Taking summation over $h\in[H]$, we have
{\proofsize \begin{align*}
    &\sum_{h=1}^H \max_{\muphi \in \Phi_i^K} G_{h}^{T, \ext}(x_h; \muphi) \le \sum_{h=1}^H(\frac{2H\log A_i}{\eta}{\rm I}_{h} + H^2\eta{\rm II}_{h})\\
    \le&~ H^3 \eta T + \sum_{h=1}^H \frac{ 2H\log A_i}{\eta}X_{i,h} {h-1 \choose K\wedge H }A_i^{K\wedge H}\\
    \le & ~H^3 \eta T\ +  \frac{2H\log A_i }{\eta}  X_i {H-1 \choose K\wedge H }A_i^{K\wedge H}\\
    \le & ~H^3 \eta T\ +  \frac{2H\log A_i }{\eta}  X_i {H \choose K\wedge H }A_i^{K\wedge H}.
\end{align*}}
As we choose $\eta =\sqrt{{H \choose K\wedge H } A_i^{K\wedge H} X_i \log A_i / (H^2T) }$ per \eqref{equation:learning-rate-full-feedback}, we have 
{\proofsize \begin{align*}
     \sum_{h=1}^H \max_{\muphi \in \Phi_i^K} G_{h}^{T, \ext}(x_h; \muphi) &\le 3\sqrt{H^4 {H \choose K\wedge H }X_i A_i^{K \wedge H }  T \log A_i }.
\end{align*}}
\end{proof}

%% file: new_version/proof-banditfeedback.tex
\section{Proofs for Section \ref{section:bandit-feedback}}
\label{appendix:bandit-feedback}

This section is devoted to proving Theorem \ref{theorem:kcfr-bandit}. The additional notation presented at the beginning of Section~\ref{appendix:full-feedback} is also used in this section.



\subsection{Sample-based loss estimator for Type-I \infoseqs}
\label{appendix:sample-based-estimator-i}

We first present the sample-based loss estimator for Type-I \infoseqs~in Algorithm~\ref{algorithm:kcfr-estimator-i}, complementary to the Type-II case presented in the main text (Algorithm~\ref{algorithm:kcfr-estimator-ii}). Here, Line~\ref{line:fill} uses the ``fill operator'' defined as follows: For any index set $I\subset\ZZ_{\ge 1}$ and $n'\ge |I|$, $\fill(I, n')$ is defined as the unique superset of $I$ with size $n'$ and the smallest possible additional elements, i.e. 
\begin{align}
\label{equation:fill}
    \fill(I, n') \defeq I \cup \{I^c_{(1)}, \dots, I^c_{(n'-|I|)}\},
\end{align}
where $I^c\defeq \ZZ_{\ge 1}\setminus I$ with sorted elements $I^c_{(1)}<I^c_{(2)}<\cdots$. For example, $\fill(\{ 1, 3, 8 \}, 5) = \{1 ,3, 8\} \cup \{ 2, 4\} = \{1 ,2, 3, 4, 8\}$. 

\begin{algorithm}[t]
  \small
   \caption{Sample-based loss estimator for Type-I \infoseqs~(\ith player's version)}
   \label{algorithm:kcfr-estimator-i}
   \begin{algorithmic}[1]
   \REQUIRE Policy $\pi_i^t$, $\pi_{-i}^t$. Balanced exploration policies $\{\pi^{\star, h}_i\}_{h\in[H]}$.
   \FOR{$1\le h\le H$, $\wb{W}\subseteq [h-1]$ with $|\wb{W}|= (K-1)\wedge (h-1)$} \label{line:typeI-samples} \label{line:sampling-start}
   \STATE Set policy $\pi_i^{t, (h, \wb{W})}\setto (\pi^{\star, h}_{i, k})_{k\in \wb{W}\cup\set{h}}\cdot (\pi^t_{i, k})_{k\in [h-1]\setminus \wb{W}} \cdot \pi^t_{i, (h+1):H}$.
   \STATE Play $\pi_i^{t, (h, \wb{W})}\times \pi_{-i}^t$ for one episode, observe trajectory 
   \begin{align*}
       ( x_{i,1}^{t, (h, \wb{W})}, a_{i,1}^{t, (h, \wb{W})}, r_{i,1}^{t, (h, \wb{W})}, \dots, x_{i,H}^{t, (h, \wb{W})}, a_{i,H}^{t, (h, \wb{W})}, r_{i,H}^{t, (h, \wb{W})} ).
   \end{align*}
   \ENDFOR
  \FOR{all $(x_{i,h}, b_{1:h-1})\in \omegai$} \label{line:estimation-start}
  \STATE Find $(x_{i, 1}, a_1) \prec \cdots \prec (x_{i, h-1}, a_{h-1}) \prec x_{i, h}$.
  \STATE Set $\wb{W}\setto \fill(\set{k\in[h-1]: b_{k}\neq a_{k}}, (K-1)\wedge (h-1))$. \label{line:fill} 
  \STATE Construct loss estimator for all $a\in\cA_i$ (below $a_h\in\cA_i$ is arbitrary):
  {\proofsize \begin{align}
                \label{equation:pil-estimator-i}
                \hspace{-1em}
      \wt L^t_{(x_{i,h}, b_{1:h-1})}(a) \setto  \frac{\indic{(x_{i,h}^{t, (h, \wb{W})}, a_{i,h}^{t, (h, \wb{W})}) = (x_{i,h}, a)}}{
      \pi_{i,1:h}^{t, (h, \wb{W})}(x_{i,h}, a)} \cdot  \sum_{h''=h}^H \paren{1 - r_{i,h''}^{t, (h, \wb{W})}}.
  \end{align}}
  \ENDFOR
  \ENSURE Loss estimators $\set{\wt L^t_{(x_{i,h}, b_{1:h-1})}(\cdot)}_{(x_{i,h}, b_{1:h-1})\in \omegai}$.
 \end{algorithmic}
\end{algorithm}

\subsection{Full description of sample-based \kscfr}
\label{appendix:bandit-feedback-algorithms}

Algorithm~\ref{algorithm:kefr-sampled-bandit} presents the detailed description of sample-based \kscfr~(sketched in Algorithm~\ref{algorithm:sample-based-kscfr-sketch}).

\paragraph{Self-play protocol}
Here we explain the protocol of how we let all players play Algorithm~\ref{algorithm:kefr-sampled-bandit} for $T$ rounds via self-play in Theorem~\ref{theorem:kcfr-bandit}. Within each round, each player first determines her own policy $\pi_i^t$ by Line~\ref{line:begin-sample-bandit}-\ref{line:end-sample-bandit}. Then, we let all players compute their sample-based loss estimators in a round-robin fashion: The first player obtains her loss estimators first (Line~\ref{line:obtain-loss-estimators}) by playing the sampling policies within Algorithm~\ref{algorithm:kcfr-estimator-i} \&~\ref{algorithm:kcfr-estimator-ii}, in which case all other players keep playing $\pi_{-1}^t$. Then the same procedure goes on for players $2,\dots,m$. Note that overall, each round plays $m$ times the number of episodes required for each player (specified by Algorithm~\ref{algorithm:kcfr-estimator-i} \&~\ref{algorithm:kcfr-estimator-ii}). The following Lemma gives a bound on this number of episodes.


\begin{lemma}[Number of episodes played by sampling algorithms]
\label{lemma:needed-episode}
One call of Algorithm \ref{algorithm:kcfr-estimator-i} and Algorithm \ref{algorithm:kcfr-estimator-ii} plays (combinedly) ${H+1 \choose K \wedge H+1} + K \wedge H -1\le 3H{H\choose K\wedge H}$
episodes.  
\end{lemma}
\begin{proof}
The proof follows by counting the number of $\wb{W}$'s and $W$'s in the sampling algorithms, i.e. the cardinalities of
$$
\mc{W}_{1, h} \defeq \set{ \wb{W} :\wb{W}\subseteq [h-1] ~\text{with}~ |\wb{W}|= (K-1)\wedge (h-1)}
$$
for $1\le h \le H$, which comes from Line~\ref{line:typeI-samples} in Algorithm \ref{algorithm:kcfr-estimator-i}, and 
$$
\mc{W}_{2, h', h} \defeq \set{W: K\le h'<h, W\subseteq[h'] ~\text{with}~ |W|=K ~\text{and ending in}~ h'} 
$$
for $K \le h \le H$, 
which comes from line \ref{line:typeII-samples} in Algorithm \ref{algorithm:kcfr-estimator-ii}.

For the first kind of sets , we have
\begin{align*}
    \sum_{h=1}^H|\mc{W}_{1, h}|  = &~ \sum_{h = 1}^{K\wedge H-1} {h-1 \choose (K-1)\wedge (h-1)} + \sum_{h = K\wedge H}^H {h-1 \choose (K-1)\wedge (h-1)}\\
    = &~ K\wedge H-1 + \sum_{h = K}^H {h-1 \choose K\wedge H-1}\\
    = &~ K\wedge H - 1 + {H \choose K\wedge H}.
\end{align*}
For the second kind of sets , we have 
\begin{align*}
    &~\sum_{K \le h' < h \le H}|\mc{W}_{2,h', h}|  =  \sum_{K \le h' < h \le H} {h'-1 \choose K-1} \\
    &~ = \sum_{h = K}^H \sum_{h' = K}^{h-1} {h' -1 \choose K-1} = \sum_{h = K}^H {h -1 \choose K} \\
    &~ = {H \choose K + 1} = {H \choose K\wedge H + 1}.
\end{align*}
Taking summation gives that the number of episodes equals $$\sum_{h=1}^H |\mc{W}_{1,h}|+\sum_{K \le h' < h \le H}|\mc{W}_{2,h', h}|  = {H+1 \choose K \wedge H+1} + K \wedge H -1.$$
 
Finally, we show the above quantity can be upper bounded by $3H{H\choose K\wedge H}$. For $K\ge H$, we have $K\wedge H=H$ and the above quantity is $1+H-1=H\le 3H=3H{H\choose K\wedge H}$. For $K<H$, we have $K\wedge H=K$, and thus
\begin{align*}
    & \quad {H+1 \choose K \wedge H+1} + K \wedge H -1 = {H+1\choose K+1} + K-1 = {H\choose K}\cdot \frac{H+1}{K+1} + K-1 \\
    & \le 2H\cdot {H\choose K} + H \le 3H\cdot {H\choose K},
\end{align*}
where the last inequality follows from the fact that ${H\choose K}\ge 1$. This is the desired bound.
\end{proof}


\begin{algorithm}[t]
  \small
   \caption{Sample-based \kcfr~(\ith player's version)}
   \label{algorithm:kefr-sampled-bandit}
   \begin{algorithmic}[1]
 \REQUIRE Weights $\{w_{b_{1:h-1}}(x_{i,h})\}_{x_{i,h},b_{1:h-1}\in\omegai(x_{i,h})}$ and  $\{w_{b_{1:h'}}(x_{i,h})\}_{x_{i,h},b_{1:h'}\in\omegaii(x_{i,h})}$ defined in~\eqref{equation:wb-i},~\eqref{equation:wb-ii}, learning rates $\{\eta_{x_{i,h}}\}_{x_{i,h}\in\cX_i}$, loss upper bound $\wb{L}>0$.
     \STATE Initialize regret minimizers $\{\regmin_{x_{i,h}}\}_{x_{i,h}\in\cX_i}$ with $\Regalg$, learning rate $\eta_{x_{i,h}}$, and loss upper bound $\wb{L}$.
     \FOR{iteration $t=1,\dots,T$}
     \FOR{$h=1,\dots,H$} \label{line:begin-sample-bandit}
     \FOR{$x_{i,h}\in\mc{X}_{i,h}$}
     
     \STATE  $S^t_{b_{1:h-1}}\defeq M^t_{b_{1:h-1}} w_{b_{1:h-1}}(x_{i,h})$ where $M^t_{b_{1:h-1}} \defeq \prod_{k = 1}^{h-1} \pi_{i, k}^t(b_k \vert x_k)$. 
    \STATE $S^t_{b_{1:h'}} \defeq M^t_{b_{1:h'}} w_{b_{1:h'}}(x_{i,h})$ where $M^t_{b_{1:h'}} \defeq  \prod_{k = 1}^{h'} \pi_{i, k}^t(b_k \vert x_k)$.. 
    \STATE $\regmin_{x_{i,h}}.\observeTSF( \{ S^t_{b_{1:h - 1}} \}_{b_{1:h - 1} \in \omegai(x_{i, h})} \cup \{ S^t_{b_{1:h'}} \}_{b_{1:h'} \in \omegaii(x_{i, h})})$.
     \STATE \label{line:end-sample-bandit} Set policy $\pi_i^t(\cdot|x_{i, h}) \setto \regmin_{x_{i,h}} .\recommend()$.  
     \ENDFOR
     \ENDFOR

   \STATE \label{line:obtain-loss-estimators} Obtaining sample-based loss estimators $\set{\wt L^t_{(x_{i,h}, b_{1:h-1})}}_{b_{1:h - 1} \in \omegai(x_{i, h})}$ and $\set{\wt L^t_{(x_{i,h}, b_{1:h'})}}_{b_{1:h'} \in \omegaii(x_{i, h})}$ from Algorithm~\ref{algorithm:kcfr-estimator-ii} \&~\ref{algorithm:kcfr-estimator-i} respectively.
     \FOR{all $x_{i,h}\in\cX_i$}
     \STATE $\regmin_{x_{i,h}}.\observeLoss( \{ \wt L^t_{(x_{i,h},b_{1:h - 1})} \}_{b_{1:h - 1} \in \omegai(x_{i, h})} \cup \{ \wt L^t_{(x_{i,h}, b_{1:h'})} \}_{b_{1:h'} \in \omegaii(x_{i, h})})$.
     \ENDFOR
     \ENDFOR
     \ENSURE Policies $\{ \pi_i^t \}_{t \in [T]}$. 
 \end{algorithmic}
\end{algorithm}

\subsection{Proof of Theorem \ref{theorem:kcfr-bandit}}
\label{appendix:proof-kcfr-bandit}
The proof follows a similar structure as the proof of Theorem \ref{theorem:kcfr-full-feedback} (cf. Section \ref{appendix:full-feedback}), with different bounds on the regret terms and bounds on additional concentration terms. 


\begin{proof}
 The proof follows by bounding the $\kefce$ regret 
  \begin{align*}
      R_{i,K}^T = \max_{\phi \in \Phi_i^K} \sum_{t = 1}^T \Big( V_{i}^{\phi\circ \pi_i^t \times \pi_{-i}^t} - V_{i}^{\pi^t} \Big)
  \end{align*}
  for all players $i \in [m]$, and then converting to a bound on $\kefcegap(\wb{\pi})$ by the online-to-batch conversion (Lemma~\ref{lem:online-to-batch}). 
  
  By the regret decomposition for $R_{i,K}^T$ (Lemma~\ref{lem:k-ce-regret_decomposition}), we have $R_{i, K}^T \leq \sum_{h=1}^H R_h^T$, where 
  \[
  R_{h}^{T} \defeq \max_{\phi \in \Phi_i^K} \sum_{x_{h}\in\cX_{i,h}}   G_{h}^{T, \swap}(x_{h}; \muphi) + \max_{\phi \in \Phi_i^K} \sum_{x_{h}\in\cX_{i,h}} G_{h}^{T, \ext}(x_{h}; \muphi) .
  \]
  

We bound the terms $\sum_{h=1}^H \max_{\phi \in \Phi_i^K}\sum_{x_{h}\in \cX_{i, h}}G_{h}^{T, \swap}(x_{h};\muphi)$ and $\sum_{h=1}^H \max_{\phi \in \Phi_i^K}\sum_{x_{h}\in \cX_{i, h}}G_{h}^{T, \ext}(x_{h};\muphi)$ when we play Sample-based \kscfr~(Algorithm~\ref{algorithm:kefr-sampled-bandit}) in the following two lemmas. Their proofs are presented in Section~\ref{appendix:bound-G1-bandit} \&~\ref{appendix:bound-G2-bandit} respectively.

\begin{lemma}[Bound on summation of $G_h^{T, \swap}(x_{i,h})$ with bandit feedback]\label{lemma:bound-G1-bandit}
If we choose learning rates as 
$$
\eta_{x_h} = \sqrt{{H \choose K\wedge H }X_i A_i^{K \wedge H +1} \log (8 X_iA_i/p) / (H^3T) } 
$$ for all $x_h \in \cX_i$ (same with \eqref{equation:learningrate-bandit}).
With probability at least $1-p/2$, we have
 {\proofsize \begin{gather*}
     \sum_{h=1}^H \max_{\phi \in \Phi_i^K}\sum_{x_{h}\in \cX_{i, h}}G_{h}^{T, \swap}(x_{h};\muphi)
     \le   \mc{O}  \paren{ \sqrt{H^3 {H \choose K\wedge H }A_i^{K \wedge
 H +1} X_i T \iota } } + \mc{O}  \paren{ H {H \choose K\wedge H }A_i^{K\wedge H + 1} X_i\iota}\\
+ \mc{O}  \paren{ {H \choose K\wedge H }  A_i^{K\wedge H  + 1 } X_i   \iota \sqrt{\frac{H {H \choose K\wedge H }  A_i^{K\wedge H   +1 } X_i \iota }{T}}},
 \end{gather*}}
where $\iota = \log ( 8 X_i A_i/p)$ is a log factor.
\end{lemma}

\begin{lemma}[Bound on summation of $G_h^{T, \ext}(x_{i,h})$ with bandit feedback]\label{lemma:bound-G2-bandit}
If we choose learning rates as 
$$
\eta_{x_h} = \sqrt{{H \choose K\wedge H } X_i A_i^{K\wedge H + 1}  \log (8X_iA_i /p)/ (H^3T) } 
$$ for all $x_h \in \cX_i$ (same with \eqref{equation:learningrate-bandit}).
With probability at least $1-p/2$, 
we have
{\proofsize \begin{gather*}
 \sum_{h=1}^H \max_{\phi \in \Phi_i^K}\sum_{x_{h}\in \cX_{i, h}}G_{h}^{T, \ext}(x_{h};\muphi)
 \le   \mc{O}  \paren{ \sqrt{H^3 {H \choose K\wedge H }A_i^{K\wedge H +1} X_i T \iota } } + \mc{O}  \paren{ H {H \choose K\wedge H }A_i^{K\wedge H + 1} X_i\iota}\\
+ \mc{O}  \paren{ {H \choose K\wedge H }  A_i^{K\wedge H  + 1 } X_i   \iota \sqrt{\frac{H {H \choose K\wedge H }  A_i^{K\wedge H   +1 } X_i \iota }{T}}},
\end{gather*}
}where $\iota = \log (8 X_i A_i/p)$ is a log factor.
\end{lemma}

By Lemma~\ref{lemma:bound-G1-bandit},~\ref{lemma:bound-G2-bandit}, and a union bound for all
$i\in[m]$, we get

{\proofsize
\begin{align*}
    & R_{i, K}^T \le \sum_{h=1}^H R_h^T\\
    \le &~ \sum_{h=1}^H \paren{\max_{\phi \in \Phi_i^K}\sum_{x_{h}\in \cX_{i, h}}G_h^{T, \swap}(x_{h};\muphi) + \max_{\phi \in \Phi_i^K}\sum_{x_{h}\in \cX_{i, h}}G_h^{T, \ext}(x_{h};\muphi)}\\
 \le &~ \cO\paren{ \sqrt{H^3{H\choose K\wedge H}A_i^{K\wedge H + 1}X_iT\iota} + H{H\choose K\wedge H}A_i^{K\wedge H + 1}X_i\iota +  {H \choose K\wedge H }  A_i^{K\wedge H  + 1 } X_i   \iota \sqrt{\frac{H {H \choose K\wedge H }  A_i^{K\wedge H   +1 } X_i \iota }{T}}}
\end{align*}
}
with probability at least $1-p$ for all $i\in [m]$ simultaneously, where $\iota=\log(8\sum_{j\in[m]} X_jA_j/p)$. 



Further using the ``trivial'' bound $R_{i,K}^T\le HT$ (by the fact that $V^\pi_i\in[0, H]$ for any joint policy $\pi$) gives

{\proofsize
\begin{equation*}
  \begin{aligned}
    \quad R_{i,K}^T 
    & \stackrel{(i)}{\le} HT \cdot \min\set{1, \cO\paren{ \sqrt{H{H\choose K\wedge H} X_iA_i^{K\wedge H+1}\iota/T}}} \\
    & \le \cO\paren{ \sqrt{H^3{H\choose K\wedge H}A_i^{K\wedge H+1}X_iT\iota} },
  \end{aligned}
\end{equation*}
}where (i) follows by noticing that:
\begin{itemize}
    \item if $T < H{H\choose K\wedge H} X_iA_i^{K\wedge H+1}\iota$, $R_{i,K}^T \le HT = HT \min\set{1, \cO\paren{ \sqrt{H{H\choose K\wedge H} X_iA_i^{K\wedge H+1}\iota/T}}} $;
    \item if $T \ge H{H\choose K\wedge H} X_iA_i^{K\wedge H + 1}\iota$, $R_{i,K}^T \le HT \cdot \cO\paren{ \sqrt{H{H\choose K\wedge H} X_iA_i^{K\wedge H+1}\iota/T}} $.
\end{itemize}

Therefore, as long as

{\proofsize
\begin{align*}
    T \ge \cO\paren{   H^3 {H \choose K\wedge H} \paren{\max_{i\in[m]}X_iA_i^{K \wedge H + 1}}\iota / \eps^2 },
\end{align*}
}we have by the online-to-batch lemma (Lemma~\ref{lem:online-to-batch}) that the average policy $\bar\pi={\rm Unif}(\{\pi^t\}_{t=1}^T)$ satisfies

{\proofsize
\begin{align*}
     \kefcegap(\wb{\pi}) = \frac{\max_{i\in[m]} R_{i,K}^T}{T} \le \max_{i\in[m]} \cO\sqrt{ H^3 {H \choose K\wedge H} \paren{\max_{i\in[m]}X_iA_i^{K \wedge H+1}}\iota / T } \le \eps.
\end{align*}
}This proves the first part of Theorem~\ref{theorem:kcfr-bandit}. 

Finally, we count how many episodes are played at each iteration. By our self-play protocol (cf. Section~\ref{appendix:bandit-feedback-algorithms}) and Lemma~\ref{lemma:needed-episode}, each iteration involves $m$ rounds of sampling (one for each player), where each round plays at most $3H{H\choose K\wedge H}$ episodes. Therefore, each iteration plays at most $3mH{H\choose K\wedge H}$ episodes, and so the total number of episodes played by Algorithm~\ref{algorithm:kefr-sampled-bandit} is
\begin{align*}
    3mH{H\choose K\wedge H} \cdot T = \cO\paren{ mH^4 {H \choose K\wedge H}^2 \paren{\max_{i\in[m]}X_iA_i^{K \wedge H + 1}}\iota / \eps^2 }.
\end{align*}
This is the desired result.

\end{proof}

\subsection{Proof of Lemma \ref{lemma:bound-G1-bandit}}\label{appendix:bound-G1-bandit}

\begin{proof}
Recall that $G_{h}^{T, \swap}(x_h; \muphi)$ is defined as (eq. (\ref{equation:local-swap-regret}))
$$
G_{h}^{T, \swap}(x_h; \muphi) \equiv~ \sum_{b_{1:h-1}} \indicDleK \prod_{k=1}^{h-1} \muphi_k(a_{k} \vert x_k, b_{1:k}) \what{R}^{T, \swap}_{(x_h, b_{1:h-1} )}.
$$
For each $h\in [H]$ and $(x_{h}, b_{1: h-1}) \in \omegai$, we have
{\proofsize  \begin{align*}
    \hat R_{(x_{h}, b_{1: h-1})}^{T, \rm swap} =&\max_\vphi \sum_{t = 1}^T \prod_{k = 1}^{h-1} \pi_i^t (b_{k}|x_{k}) \Big( \< \pi_{i, h}^t(\cdot \vert x_{i, h}) - \vphi \circ \pi_{i, h}^t(\cdot \vert x_{i, h}), \L_{i,h}^{t}(x_{i, h},\cdot) \> \Big)\\
    \le & \max_\vphi \sum_{t = 1}^T \prod_{k = 1}^{h-1} \pi_i^t (b_{k}|x_{k})  
    \Big\langle \pi_{i, h}^t(\cdot \vert x_{i, h}) - \vphi \circ \pi_{i, h}^t(\cdot \vert x_{i, h}) , \wt L_{(x_{h}, b_{1: h-1})}^{t} \Big\rangle\\
    &+ \sum_{t = 1}^T \prod_{k = 1}^{h-1} \pi_i^t (b_{k}|x_{k})
    \Big\langle\pi_{i, h}^t(\cdot \vert x_{i, h}),  \L_{i,h}^{t}(x_{i, h},\cdot) -\wt L^t_{(x_{h}, b_{1: h-1})} \Big\rangle \\
    &+ \max_\vphi \sum_{t = 1}^T \prod_{k = 1}^{h-1} \pi_i^t (b_{k}|x_{k})
    \Big\langle \vphi \circ \pi_{i, h}^t(\cdot \vert x_{i, h}), \wt L^t_{(x_{h}, b_{1: h-1})}- \L_{i,h}^{t}(x_{i, h},\cdot) \Big\rangle.
\end{align*} }
Substituting this into $\max_{\muphi \in \Phi_i^K}\sum_{x_{h}\in \cX_{i, h}}G_{h}^{T, \swap}(x_{h};\muphi)$ yields that

{\proofsize  \begin{align*}
    & \max_{\muphi \in \Phi_i^K}\sum_{x_{h}\in \cX_{i, h}}G_{h}^{T, \swap}(x_{h};\muphi)\\
    =~& \sum_{x_{ h}} \sum_{b_{1:h-1}} \indicDleK \prod_{k=1}^{h-1} \muphi_k(a_{k} \vert x_k, b_{1:k}) \what{R}^{T, \swap}_{(x_h, b_{1:h-1} )}\\
    \le~& \underbrace{\max_{\muphi \in \Phi_i^K} \sum_{x_{h}\in \cX_{i, h}} \sum_{b_{1:h-1}} \indicDleK \prod_{k=1}^{h-1} \muphi_k(a_{k} \vert x_k, b_{1:k})  \max_\vphi \sum_{t = 1}^T \prod_{k = 1}^{h-1} \pi_i^t (b_{k}|x_{k})  
    \Big\langle \pi_{i, h}^t(\cdot \vert x_{i, h}) - \vphi \circ \pi_{i, h}^t(\cdot \vert x_{i, h}) , \wt L_{(x_{h}, b_{1: h-1})}^{t} \Big\rangle}_{\defeq \wt{\rm REGRET}_h^{T, \swap}}\\
    &+  \underbrace{ \max_{\muphi \in \Phi_i^K} \sum_{x_{h}\in \cX_{i, h}} \sum_{b_{1:h-1}} \indicDleK \prod_{k=1}^{h-1} \muphi_k(a_{k} \vert x_k, b_{1:k}) \sum_{t = 1}^T \prod_{k = 1}^{h-1} \pi_i^t (b_{k}|x_{k})
    \Big\langle\pi_{i, h}^t(\cdot \vert x_{i, h}), \L_{i,h}^{t}(x_{i, h},\cdot) - \wt L^t_{(x_{h}, b_{1: h-1})} \Big\rangle}_{\defeq {\rm BIAS}_{1, h}^{T, \swap}}  \\
    &+ \underbrace{ \max_{\muphi \in \Phi_i^K} \sum_{x_{h}\in \cX_{i, h}} \sum_{b_{1:h-1}} \indicDleK \prod_{k=1}^{h-1} \muphi_k(a_{k} \vert x_k, b_{1:k}) \max_\vphi \sum_{t = 1}^T \prod_{k = 1}^{h-1} \pi_i^t (b_{k}|x_{k})
    \Big\langle \vphi \circ \pi_{i, h}^t(\cdot \vert x_{i, h}), \wt L^t_{(x_{h}, b_{1: h-1})}-\L_{i,h}^{t}(x_{i, h},\cdot) \Big\rangle}_{\defeq {\rm BIAS}_{2, h}^{T, \swap}}.
\end{align*} }


The bounds for $\sum_{h=1}^H \wt{\rm REGRET}_h^{T, \swap}$, $\sum_{h=1}^H {\rm BIAS}_{1, h}^{T, \swap}$, and $\sum_{h=1}^H {\rm BIAS}_{2, h}^{T, \swap}$  are given in the following three lemmas (proofs deferred to Appendix~\ref{appendix:proof-bound-regret^I}, \ref{appendix:proof-bound-bias1^I}, and \ref{appendix:proof-bound-bias2^I}) respectively.


\begin{lemma}[Bound on $\wt{\rm REGRET}_h^{T, \swap}$]\label{lemma:bound-regret^I}
If we choose learning rates as 
$$
\eta_{x_h} = \sqrt{ {H \choose K\wedge H } X_i A_i^{K \wedge H +1} \log(X_i A_i/p) / (H^3T) } \cdot 
$$ for all $x_h \in \cX_i$ (same with \eqref{equation:learningrate-bandit}).
 Then with probability at least $1-p/4$, we have
{\proofsize  \begin{align*}
     \sum_{h=1}^H   \wt{\rm REGRET}_{h}^{T, \swap} &\le \sqrt{H^3 {H \choose K\wedge H }A_i^{K \wedge
 H +1} X_i T \iota}\\
    + &\mc{O}  \paren{ {H \choose K\wedge H }  A_i^{K \wedge H} X_i  \iota \sqrt{\frac{H {H \choose K\wedge H }  A_i^{K \wedge H +1 } X_i \iota }{T}}}
\end{align*} 
}where $\iota = \log (8X_i A_i/p)$ is a log factor.
\end{lemma}

\begin{lemma}[Bound on ${\rm BIAS}_{1, h}^{T, \swap}$]\label{lemma:bound-bias1^I}
With probability at least $1-\frac{p}{8}$, we have
{\proofsize  \begin{align*}
    \sum_{h=1} ^H {\rm BIAS}_{1, h}^{T, \swap}\le \mc{O}\paren{\sqrt{H^3 {H \choose K\wedge H} A_i^{K\wedge H} X_iT \iota}+ H {H \choose K\wedge H -1}A_i^{K\wedge H} X_i\iota}, 
\end{align*} 
}where $\iota = \log(8X_iA_i/p)$ is a log factor.
\end{lemma}

\begin{lemma}[Bound on ${\rm BIAS}_{2, h}^{T, \swap}$]\label{lemma:bound-bias2^I}
With probability at least $1-\frac{p}{8}$, we have
{\proofsize  \begin{align*}
    \sum_{h=1} ^H {\rm BIAS}_{2, h}^{T, \swap}\le \mc{O}\paren{\sqrt{H^3 {H \choose K\wedge H} A_i^{K\wedge H+1} X_iT \iota}+ H {H \choose K\wedge H }A_i^{K\wedge H+1} X_i\iota}, 
\end{align*} 
}where $\iota = \log(8X_iA_i/p)$ is a log factor.
\end{lemma}


Combining Lemma~\ref{lemma:bound-regret^I},~\ref{lemma:bound-bias1^I}, and~\ref{lemma:bound-bias2^I}, we have with probability at least $1-p/2$ that
 {\proofsize  \begin{align*}
     \sum_{h=1}^H \max_{\phi \in \Phi_i^K}\sum_{x_{h}\in \cX_{i, h}}G_{h}^{T, \swap}(x_{h};\muphi) &\le  \sum_{h=1}^H \wt{\rm REGRET}_h^{T, \swap} +  \sum_{h=1}^H {\rm BIAS}_{1, h}^{T, \swap} +  \sum_{h=1}^H {\rm BIAS}_{2, h}^{T, \swap}\\
     &\le \mc{O}  \paren{\sqrt{H^3 {H \choose K\wedge H }A_i^{K \wedge
 H +1} X_i T \iota }}\\
    &+ \mc{O}  \paren{ {H \choose K\wedge H }  A_i^{K \wedge H} X_i   \iota \sqrt{\frac{H {H \choose K\wedge H }  A_i^{K \wedge H +1 } X_i \iota }{T}}}\\
    &+  \mc{O}  \paren{ H {H \choose K\wedge H }A_i^{K\wedge H + 1} X_i\iota}.
 \end{align*} 
 }
 \end{proof}

\subsubsection{Proof of Lemma \ref{lemma:bound-regret^I}: Bound on $\wt{\rm REGRET}_h^{T, \swap}$}
\label{appendix:proof-bound-regret^I}
\begin{proof}
Recall that $\wt{\rm REGRET}_h^{T, \swap}$ is defined as 
{\proofsize
\begin{align*}
  \max_{\muphi \in \Phi_i^K} \sum_{x_{h}\in \cX_{i, h}} \sum_{b_{1:h-1}} \indicDleK \prod_{k=1}^{h-1} \muphi_k(a_{k} \vert x_k, b_{1:k})  \max_\vphi \sum_{t = 1}^T \prod_{k = 1}^{h-1} \pi_i^t (b_{k}|x_{k})  \Big\langle \pi_i^t( \cdot | x_{h})- \vphi \circ \pi_i^t(\cdot | x_{h}) , \wt L_{(x_{h}, b_{1: h-1})}^{t} \Big\rangle.
\end{align*}
}We first apply regret minimization lemma (Lemma \ref{lemma:regret-with-time-selection-bandit}) on $\regmin_{x_h}$ to give an upper bound of 
$$
\max_\vphi \sum_{t = 1}^T \prod_{k = 1}^{h-1} \pi_i^t (b_{k}|x_{k})  \Big\langle \pi_i^t( \cdot | x_{h})- \vphi \circ \pi_i^t(\cdot | x_{h}) , \wt L_{(x_{h}, b_{1: h-1})}^{t} \Big\rangle.
$$
For $\regmin_{x_h}$, Algorithm \ref{algorithm:kefr-sampled-bandit} gives that $M^t_{b_{1:h-1}} = \prod_{k = 1}^{h-1} \pi_i^t (b_{k}|x_{k})$ and $w_{b_{1:h-1}} = \prod_{k \in  \fill(W, (h-1)\wedge(K-1)) \cup \{h\}} \pi^{\star, h}_{i,k}(a_k \vert x_k) $ for any $b_{1:h-1} \in \omegai(x_h)$,  where $W = \set{k\in [h-1]: b_k\not= a_k}$. Letting  $\wb{W} = \fill(W, (h-1)\wedge(K-1))$, $\eta$ be the learning rate of $\regmin_{x_h}$ and $\wb L = H$. Since regret minimizers $\mc{R}_{x_{h}}$ observe $\set{\wt L^t_{(x_{i,h}, b_{1:h'})}(\cdot)}_{(x_{i,h}, b_{1:h'})\in \omegaii}$ and $\set{\wt L^t_{(x_{i,h}, b_{1:h-1})}(\cdot)}_{(x_{i,h}, b_{1:h-1})\in \omegai}$ from Algorithm~\ref{algorithm:kcfr-estimator-ii} \&~\ref{algorithm:kcfr-estimator-i} as its loss vector at round $t$, we have 
\begin{align*}
     &~M_{b_{1:h-1}}^t w_{b_{1:h-1}} \wt L^t_{(x_{i,h}, b_{1:h-1})}(\cdot)  \\
     =& ~M_{b_{1:h-1}}^t w_{b_{1:h-1}} \frac{\indic{(x_{i,h}^{t, (h, \wb{W})}, a_{i,h}^{t, (h, \wb{W})}) = (x_{i,h}, \cdot)}}{\prod_{k\in \wb{W}\cup\set{h}} \pi^{\star, h}_{i, k}(a_{k}|x_{i,k}) \prod_{k \in [h-1] \setminus \wb{W}} \pi_i^t(b_k \vert x_{i,k})} \cdot  \sum_{h''=h}^H \paren{1 - r_{i,h''}^{t, (h, \wb{W})}} \\
     =&~ \prod_{k \in \wb{W} } \pi_i^t(b_k \vert x_{i,k}) \indic{(x_{i,h}^{t, (h, \wb{W})}, a_{i,h}^{t, (h, \wb{W})}) = (x_{i,h}, \cdot)} \sum_{h''=h}^H \paren{1 - r_{i,h''}^{t, (h, \wb{W})}}\\ \in& ~[0,H] = [0,\wb{L}]. 
\end{align*}
Similarly, we have $M_{b_{1:h'}}^t w_{b_{1:h-1}} \wt L_{(x_{i,h}, b_{1:h'})}(\cdot ) \in [0, \wb{L}]$. Moreover, let $\cF_{t-1}$ be the $\sigma$-algebra containing all the information after $\pi^t$ is sampled, by the sampling algorithm, we have 
\begin{align*}
    \E\brac{ \wt L^t_{(x_{i,h}, b_{1:h-1})}(\cdot) \vert \cF_{t-1}} =  \E\brac{ \wt L^t_{(x_{i,h}, b_{1:h'})}(\cdot)\vert \cF_{t-1}} = L^t_{h} (x_{i,h}, \cdot),
\end{align*}
for all $(x_{i,h}, b_{1:h-1})\in \omegai$ and all $(x_{i,h}, b_{1:h'})\in \omegaii$. So the assumptions in Lemma \ref{lemma:regret-with-time-selection-bandit} are satisfied. 
By Lemma \ref{lemma:regret-with-time-selection-bandit}, with probability at least $1-p/8$, for all $x_h \in \cX_i$, we have
{\proofsize  \begin{align*}
    &\max_\vphi \sum_{t = 1}^T \prod_{k = 1}^{h-1} \pi_i^t (b_{k}|x_{k})  \Big\langle \pi_i^t( \cdot | x_{h})- \vphi \circ \pi_i^t(\cdot | x_{h}) , \wt L_{(x_{h}, b_{1: h-1})}^{t} \Big\rangle. \\ 
    \le ~&\frac{2 A_i\log (8X_iA_i/p)}{\eta w_{b_{1:h-1}}} + \eta\sum_{t=1}^T M^t_{b_{1:h-1}} \wb{L} \Big\langle \pi_i^t (\cdot|x_{k}), \wt L_{(x_{h}, b_{1: h-1})}^{t}(\cdot) \Big\rangle\\
    \overset{}{=}~ & \frac{2 A_i\log (8X_iA_i/p)}{\eta\prod_{k\in \wb{W} \cup \set{h}} \pi^{\star,h}_{i, k}(a_{k} \vert x_k )} +H\eta\prod_{k\in [h-1]} \pi^{t}_{i, k}(b_{k} \vert x_k )\\
    &~~~ \times \sum_{t=1}^T \Big\langle \pi_i^t (\cdot|x_{k}), \frac{ \indic{(x_h, \cdot) = (x_h^{t,(h,\wb{W})}, a_{h}^{t, (h, \wb{W})}) }\paren{H-h+1 - \sum_{h' = h}^H r_{i, h'}^{t, (h, \wb{W})}}}{\prod_{k\in \wb{W} \cup \{h\}} \pi^{\star,h}_{i, k}(a_{k} \vert x_k ) \prod_{k \in [h-1] \setminus \wb{W}} \pi_{i,k}^t(b_{k} \vert x_k)} \Big\rangle\\
    \le~& \frac{2 A_i\log (8X_iA_i/p)}{\eta\prod_{k\in \wb{W} \cup \set{h}} \pi^{\star,h}_{i, k}(a_{k} \vert x_k )}+H^2 \eta  \sum_{t=1}^T\Big\langle \pi_i^t (\cdot|x_{k}), \frac{\prod_{k \in \wb{W}} \pi_{i,k}^t(b_{k} \vert x_k) \indic{(x_h, \cdot) = (x_h^{t,(h,\wb{W})}, a_{h}^{t, (h, \wb{W})}) }}{\prod_{k\in \wb{W} \cup \set{h} } \pi^{\star,h}_{i, k}(a_{k} \vert x_k )} \Big\rangle .
\end{align*} 
}Here, we use $(i)$ our choices of $M_{b_{1:h-1}}^t$ and $w_{b_{1:h-1}}$; $(ii) ~(|\cB^s|+|\cB^e|) |\Psi^s| \le A_i^{H+A_i} \le X_i A_i^{A_i}$ and $(iii)$  taking union bound over all infosets. 
Plugging this into $\wt{\rm REGRET}_{h}^{T, \swap}$, we have
{\proofsize  \begin{align*}
    &\wt{\rm REGRET}_{h}^{T, \swap}\\=~& \max_{\muphi \in \Phi_i^K} \sum_{x_{h}\in \cX_{i, h}} \sum_{b_{1:h-1}} \indicDleK \prod_{k=1}^{h-1} \muphi_k(a_{k} \vert x_k, b_{1:k}) \\
    &\times \max_\vphi \sum_{t = 1}^T \prod_{k = 1}^{h-1} \pi_i^t (b_{k}|x_{k})  \Big\langle \pi_i^t( \cdot | x_{h})- \vphi \circ \pi_i^t(\cdot | x_{h}) , \wt L_{(x_{h}, b_{1: h-1})}^{t} \Big\rangle.\\
    \le~& \frac{2 A_i\log(8X_i A_i/p)}{\eta} \max_{\muphi \in \Phi_i^K} \sum_{x_{h}\in \cX_{i, h}} \sum_{b_{1:h-1}} \indicDleK  \frac{\prod_{k=1}^{h-1} \muphi_k(a_{k} \vert x_k, b_{1:k})}{\prod_{k\in \wb{W} \cup \set{h}} \pi^{\star,h}_{i, k}(a_{k} \vert x_k )} \\
    &+  H^2   \eta \max_{\muphi \in \Phi_i^K} \sum_{x_{h}\in \cX_{i, h}} \sum_{b_{1:h-1}} \indicDleK\frac{\prod_{k=1}^{h-1} \muphi_k(a_{k} \vert x_k, b_{1:k})}{\prod_{k\in \wb{W} \cup \set{h} } \pi^{\star,h}_{i, k}(a_{k} \vert x_k )} \\
    & \times \sum_{t=1}^T \underbrace{\prod_{k \in \wb{W}} \pi_{i,k}^t(b_{k} \vert x_k) \Big\langle \pi_i^t (\cdot|x_{k}),  \indic{(x_h, \cdot) = (x_h^{t,(h,\wb{W})}, a_{h}^{t, (h, \wb{W})}) } \Big\rangle}_{\defeq\wb{\Delta}_{t}^{x_h,b_{1:h-1}}}.
\end{align*} 
}Letting
{\proofsize \begin{gather*}
    {\rm{I}}_h \defeq \frac{2 A_i\log (8X_iA_i/p)}{\eta} \max_{\muphi \in \Phi_i^K} \sum_{x_{h}\in \cX_{i, h}} \sum_{b_{1:h-1}} \indicDleK  \frac{\prod_{k=1}^{h-1} \muphi_k(a_{k} \vert x_k, b_{1:k})}{\prod_{k\in \wb{W} \cup \set{h}} \pi^{\star,h}_{i, k}(a_{k} \vert x_k )};\\
    {\rm{II}}_h \defeq H^2\eta \max_{\muphi \in \Phi_i^K} \sum_{x_{h}\in \cX_{i, h}} \sum_{b_{1:h-1}} \indicDleK\frac{\prod_{k=1}^{h-1} \muphi_k(a_{k} \vert x_k, b_{1:k})}{\prod_{k\in \wb{W} \cup \set{h}} \pi^{\star,h}_{i, k}(a_{k} \vert x_k )} \sum_{t=1}^T \wb{\Delta}_t^{x_h, b_{1:h-1}}.
\end{gather*} }

Using Lemma \ref{lemma:balancing}, we have
{\proofsize  \begin{align*}
    &\sum_{x_{h}\in \cX_{i, h}} \sum_{b_{1:h-1}} \indicDleK  \frac{\prod_{k=1}^{h-1} \muphi_k(a_{k} \vert x_k, b_{1:k})}{ \prod_{k\in \wb{W} \cup \set{h}} \pi^{\star,h}_{i, k}(a_{k} \vert x_k )}\\
    =& \sum_{x_{h}\in \cX_{i, h}} \sum_{b_{1:h-1}}\indicDleK  \frac{\prod_{k=1}^{h-1} \muphi_k(a_{k} \vert x_k, b_{1:k})\prod_{k\in [h-1]\setminus \wb{W}} \pi^{\star,h}_{i, k}(a_{k} \vert x_k )}{ \prod_{k\in [h]} \pi^{\star,h}_{i, k}(a_{k} \vert x_k )}\\
    \overset{(i)}{=}&\sum_{x_{h}\in \cX_{i, h}} \sum_{a_h}  \frac{\sum_{b_{1:h-1}} \indicDleK \prod_{k=1}^{h-1} \muphi_k(a_{k} \vert x_k, b_{1:k})\prod_{k\in [h-1]\setminus \wb{W}} \pi^{\star,h}_{i, k}(b_{k} \vert x_k ) \cdot \pi_{i,h}^{\star,h}(a_h \vert x_h)}{ \prod_{k\in [h]} \pi^{\star,h}_{i, k}(a_{k} \vert x_k )} \\
    \overset{(ii)}{=}&  \sum_{\wb{W} \subset [h-1], |\wb{W}|= (h-1)\wedge(K-1)}A_{i}^{|\wb{W}|}  \\
    & \times\sum_{x_h, a_h}  \frac{\sum_{b_{1:h-1}: \wb{W} = \fill( \set{k\in[h-1]: a_k \not = b_k},K\wedge H -1)}  \prod_{k=1}^{h-1} \muphi_k(a_{k} \vert x_k, b_{1:k})\prod_{k\in [h-1]\setminus \wb{W}} \pi^{\star,h}_{i, k}(b_{k} \vert x_k )\prod_{k\in \wb{W}} \pi^{\rm unif}_{i, k}(b_{k} \vert x_k ) \pi^{\star,h}_{i, h}(a_{h} \vert x_h )}{ \prod_{k\in [h]} \pi^{\star,h}_{i, k}( a_{k} \vert x_k )} \\
    \overset{(iii)}{\le }& \sum_{\wb{W} \subset [h-1], |\wb{W}|=(K-1) \wedge (h-1) } A_{i}^{|\wb{W}|} X_{i,h} A_i~~ \\
    \overset{}{=}~ & X_{i,h} {h-1 \choose (K-1) \wedge (h-1)}A_i^{K \wedge h}.
\end{align*} 
}Here, (i) uses that $\pi_{i,h}^{\star,h}(\cdot|x_h)$ is uniform distribution on $\cA_i$ and that for $k\in [h-1] \setminus \wb{W} $, we have $a_k=b_k$; (ii) follows from grouping $x_h$ and $b_{1:h}$ by $|\wb{W}|$, where $\pi_{i, k}^{\rm unif}$ is the uniform distribution on $\cA_i$; (iii) uses Lemma \ref{lemma:balancing} and that the 
numerator is no more than the sequence-form probability of some policy. This policy can be understand as: 
\begin{itemize}
    \item Sample recommended action $b_k$ from $\pi_{i,k}^{\star, h}(\cdot | x_k)$ if step $k \in \wb{W}$. Otherwise, sample recommended action $b_k$ from $\pi_{i,k}^{\rm unif}(\cdot | x_k)$.
    
    \item ``True'' actions are sampled from $\phi_k(a_k |x_k, b_{1:k})$ for $k \in [h-1]$. At step $h$, take action $a_h$.
\end{itemize}
Consequently, 
{\proofsize \begin{equation}\label{equation:swap-mu-over-pi-bound}
    \sum_{x_{h}\in \cX_{i, h}} \sum_{b_{1:h-1}} \indicDleK  \frac{\prod_{k=1}^{h-1} \muphi_k(a_{k} \vert x_k, b_{1:k})}{ \prod_{k\in \wb{W} \cup \set{h}} \pi^{\star,h}_{i, k}(a_{k} \vert x_k )} \le X_{i,h} {h-1 \choose K \wedge h-1}A_i^{K \wedge h}.
\end{equation}
}So we have
$${\rm I}_h \le \frac{2 A_i\log (8X_iA_i/p)}{\eta}X_{i,h} {h-1 \choose (K-1) \wedge (h-1)}A_i^{K \wedge h}.$$
For ${\rm II}_h$, obvserve that the random variables $\wb{\Delta}_{t}^{x_h,b_{1:h-1}}$ satisfy the following:  
\begin{itemize}
    \item $\wb{\Delta}_{t}^{x_h,b_{1:h-1}}= \prod_{k \in \wb{W}} \pi_{i,k}^t(b_{k} \vert x_k) \Big\langle \pi_i^t (\cdot|x_{k}),  \indic{(x_h, \cdot) = (x_h^{t,(h,\wb{W})}, a_{h}^{t, (h, \wb{W})}) } \Big\rangle \in [0,1]$; 
    \item Let $\cF_{t-1}$ be the $\sigma$-algebra containing all information after $\pi^t$ is sampled, then
    {\proofsize  \begin{align*}
        &~\E\brac{\wb{\Delta}_{t}^{x_h,b_{1:h-1}}\vert \cF_{t-1}} \\
        =& \prod_{k \in \wb{W}} \pi_{i,k}^t(b_{k} \vert x_k)\E\brac{ \sum_{a\in \cA_i} \pi_i^t(a|x_k) \indic{(x_h, a) = (x_h^{t,(h,\wb{W})}, a_{h}^{t, (h, \wb{W})})} \Big\vert \cF_{t-1}} \\
        =&\prod_{k \in \wb{W}} \pi_{i,k}^t(b_{k} \vert x_k) \sum_{a\in\cA_i}  \pi_i^t(a|x_k) \P^{((\pi^{\star, h}_{i,k})_{k\in \wb{W}\cup \set{h}}(\pi^{t}_{i,k})_{k\in[h-1]\setminus \wb{W}} ) \times \pi_{-i}^t}(x_h^{t,(h,\wb{W})} = x_h , a_h^{t, (h, \wb{W})} = a)\\
        = & \prod_{k \in \wb{W}} \pi_{i,k}^t(b_{k} \vert x_k) \sum_{a \in \cA_i} \pi_i^t(a|x_k) \paren{\prod_{k \in \wb{W} } \pi_{i,k}^{\star,h}(a_{k} \vert x_k) \cdot\prod_{k \in [h-1]\setminus \wb{W}} \pi_{i,k}^{t}(a_{k} \vert x_k)\cdot \pi_{i,h}^{\star,h}(a|x_h) p^{\pi^t_{-i}}_{1:h}(x_{h})} \\
        = & \prod_{k \in [h-1]} \pi_{i,k}^t(b_{k} \vert x_k) \cdot \prod_{k \in \wb{W} \cup \set{h} } \pi_{i,k}^{\star,h}(a_{k} \vert x_k) \cdot p^{\pi^t_{-i}}_{1:h}(x_{h}),
    \end{align*} } where the last equation is because $\prod_{k \in \wb{W}} \pi_{i,k}^t(b_{k} \vert x_k) \cdot \prod_{k \in [h-1]\setminus \wb{W}} \pi_{i,k}^{t}(a_{k} \vert x_k) = \prod_{k \in [h-1]} \pi_{i,k}^t(b_{k} \vert x_k)$;
    
    \item The conditional variance $\E[(\wb\Delta_t^{x_{h}, b_{1:h-1}})^2|\cF_{t-1}]$ can be bounded as
    {\proofsize  \begin{align*}
        & ~\quad \E\brac{ \paren{\wb\Delta_t^{x_{h}, b_{1:h-1}}}^2 \Big| \cF_{t-1}} \le \E\brac{\wb\Delta_t^{x_{h}, b_{1:h-1}} \Big| \cF_{t-1}}\\
    &= \prod_{k \in [h-1]} \pi_{i,k}^t(b_{k} \vert x_k) \prod_{k \in \wb{W} \cup \set{h} } \pi_{i,k}^{\star,h}(a_{k} \vert x_k)p^{\pi^t_{-i}}_{1:h}(x_{h}).
    \end{align*} }
    Here, the inequality comes from that $\wb\Delta_t^{x_{h}, b_{1:h-1}}\in [0,1]$.
\end{itemize}
Therefore, we can apply Freedman's inequality and union bound to get that  for any fixed $\lambda \in (0, 1]$, with probability at least $1-p/8$, the following holds simultaneously for all $(h,x_{i, h},  b_{1:h-1})$: 

{\proofsize  \begin{align*}
    \sum_{t = 1}^T \wb\Delta_t^{x_h, b_{1:h-1}} \le& \paren{ \lambda+1 } \prod_{k \in [h-1]} \pi_{i,k}^t(b_{k} \vert x_k) \prod_{k \in \wb{W} \cup \set{h} } \pi_{i,k}^{\star,h}(a_{k} \vert x_k)p^{\pi^t_{-i}}_{1:h}(x_{h})
    + \frac{C\log(X_i A_i/p)}{\lambda},
\end{align*} }
where $C> 0$ is some absolute constant. Plugging this bound into ${\rm II}_h$ yields that,
{\proofsize  \begin{align*}
    &{\rm II}_h \le H^2 \eta  \max_{\muphi \in \Phi_i^K} \sum_{x_{h}\in \cX_{i, h}} \sum_{b_{1:h-1}} \indicDleK\frac{\prod_{k=1}^{h-1} \muphi_k(a_{k} \vert x_k, b_{1:k})}{\prod_{k\in \wb{W} \cup \set{h}} \pi^{\star,h}_{i, k}(a_{k} \vert x_k )}\\
    & \brac{\paren{ \lambda + 1 } \prod_{k \in [h-1]} \pi_{i,k}^t(b_{k} \vert x_k) \prod_{k \in \wb{W} \cup \set{h} } \pi_{i,k}^{\star,h}(a_{k} \vert x_k) \cdot p^{\pi^t_{-i}}_{1:h}(x_{h})
    + \frac{C\log(X_i A_i/p)}{\lambda}}.
\end{align*} }
Note that by Corollary \ref{lemma:sum=1}, for fixed $t$, {\proofsize
$$\sum_{x_{h}\in \cX_{i, h}} \sum_{b_{1:h-1}}\indicDleK \prod_{k=1}^{h-1} \muphi_k(a_{k} \vert x_k, b_{1:k}) \prod_{k \in [h-1]} \pi_{i,k}^t(b_{k} \vert x_k) \cdot p^{\pi^t_{-i}}_{1:h}(x_{i,h}) \le 1.
$$} so we have
{\proofsize \begin{equation*}
    \begin{gathered}
        \max_{\muphi \in \Phi_i^K} \sum_{x_{h}\in \cX_{i, h}} \sum_{b_{1:h-1}}\indicDleK \prod_{k=1}^{h-1} \muphi_k(a_{k} \vert x_k, b_{1:k}) \sum_{t = 1}^T \prod_{k \in [h-1]} \pi_{i,k}^t(b_{k} \vert x_k) \cdot p^{\pi^t_{-i}}_{1:h}(x_{i,h}) \le T.
    \end{gathered}
\end{equation*} }
Moreover, by the previous bound \eqref{equation:swap-mu-over-pi-bound},
{\proofsize  \begin{align*}
    \sum_{x_{h}\in \cX_{i, h}} \sum_{b_{1:h-1}} \indicDleK  \frac{\prod_{k=1}^{h-1} \muphi_k(a_{k} \vert x_k, b_{1:k})}{ \prod_{k\in \wb{W}} \pi^{\star,h}_{i, k}(a_{k} \vert x_k )} \le X_{i,h}{h-1\choose K\wedge H -1}A_i^{K \wedge h}.
\end{align*} }
Using these two inequalities, we can get that 
{\proofsize  \begin{align*}
    &{\rm II}_h \le H^2\eta \paren{ \lambda+1 }T+ \frac{CH^2  \eta\log(X_i A_i/p)}{ \lambda}X_{i,h}{h-1\choose K\wedge H -1}A_i^{K \wedge h}.
\end{align*} }
Taking summation over $h\in[H]$, we have
{\proofsize  \begin{align*}
    &\sum_{h=1}^H   \wt{\rm REGRET}_{h}^{T, \swap} = \sum_{h=1}^H({\rm I}_{h}+{\rm II}_{h})\\
    \le~& H^3 \eta \paren{ \lambda+1 }T + \paren{ \frac{2 A_i\log (8X_i A_i/p)}{\eta} + \frac{CH^2 \eta\log(X_i A_i/p)}{ \lambda}}\sum_{h=1}^H X_{i,h}  {h-1\choose (K-1) \wedge (h-1)}A_i^{K \wedge h}\\
    \le ~& H^3 \eta \paren{ \lambda+1 } T + \paren{ \frac{2 A_i\log(8X_i A_i/p)}{\eta} + \frac{CH^2 \eta\log(X_i A_i/p)}{ \lambda}} X_i {H-1 \choose K\wedge H -1}A_i^{K\wedge H },
\end{align*} }
for all $\lambda\in (0,1].$ Choosing $\lambda = 1$, we have,
{\proofsize  \begin{align*}
    \sum_{h=1}^H   \wt{\rm REGRET}_{h}^{T, \swap} \le
    2H^3\eta T + \paren{ \frac{2 A_i \log(8X_i A_i/p) }{\eta} +  CH^2 \eta\log(X_i A_i/p) } X_i {H-1 \choose K\wedge H -1}A_i^{K\wedge H }.
\end{align*} }
Then, choosing $\eta = \sqrt{{H \choose K\wedge H }A_i^{K \wedge
 H +1} X_i \iota / (H^3T) }$ and using ${H-1 \choose K\wedge H -1} \le {H \choose K\wedge H}$ we have
{\proofsize  \begin{align*}
     \sum_{h=1}^H   \wt{\rm REGRET}_{h}^{T, \swap} &\le  2{\sqrt{H^3 {H \choose K\wedge H } A_i^{K \wedge
     H +1} X_i T \iota }}\\
    + &\mc{O}  \paren{ {H \choose K\wedge H }  A_i^{K \wedge H} X_i   \iota \sqrt{\frac{H {H \choose K\wedge H }  A_i^{K \wedge H +1 } X_i \iota }{T}}}
\end{align*} }
with probability at least $1-p/8$, where $\iota = \log (8X_i A_i/p)$ is a log factor.
\end{proof}

\subsubsection{Proof of Lemma \ref{lemma:bound-bias1^I}: Bound on ${\rm BIAS}_{1, h}^{T, \swap}$}
\label{appendix:proof-bound-bias1^I}
\begin{proof}
We can rewrite ${\rm BIAS}_{1, h}^{T, \swap}$ as 
{\proofsize  \begin{align*}
    &{\rm BIAS}_{1, h}^{T, \swap}\\
    =~& \max_{\phi \in \Phi_i^K}  \sum_{x_{h}\in \cX_{i, h}} \sum_{b_{1:h-1}} \indicDleK \frac{\prod_{k=1}^{h-1} \muphi_k(a_{k} \vert x_k, b_{1:k})}{\prod_{k \in \wb{W} \cup \set{h} } \pi^{\star,h}_{i, k}(a_{k} \vert x_k )} \\
    &\times \sum_{t = 1}^T \underbrace{ \prod_{k \in \wb{W} \cup \set{h} } \pi^{\star,h}_{i, k}(a_{k} \vert x_k ) \prod_{k = 1}^{h-1} \pi_i^t (b_{k}|x_{k})
    \Big\langle \pi_{i, h}^t(\cdot \vert x_{i, h}), \L_{i,h}^{t}(x_{i, h},\cdot)- \wt L^t_{(x_{h}, b_{1: h-1})} \Big\rangle}_{\defeq \wt\Delta_t^{x_{h}, b_{1:h-1}}}.
\end{align*}
}Here, for fixed $x_h \in \cX_i$ and $b_{1:h-1}$, let $W = \set{k\in [h-1]: b_k\not= a_k}$ and $\wb{W} = \fill(W, (h-1)\wedge(K-1))$. Observe that the random variable $\wt\Delta_t^{x_{h}, b_{1: h-1}}$ satisfy the following:
\begin{itemize}
    \item By the definition of $\wt L$ in Algorithm \ref{algorithm:kcfr-estimator-i}, we can rewrite $\wt \Delta_t^{x_{h}, b_{1:h-1}}$ as
    {\proofsize  \begin{align*}
        &\wt \Delta_t^{x_{h}, b_{1:h-1}} = \prod_{k \in \wb{W} \cup \set{h} } \pi^{\star,h}_{i, k}(a_{k} \vert x_k ) \prod_{k = 1}^{h-1} \pi_i^t (b_{k}|x_{k}) \Big\langle \pi_{i, h}^t(\cdot \vert x_{i, h}), \L_{i,h}^{t}(x_{i, h},\cdot) \Big\rangle \\
        &- \prod_{k \in \wb{W}} \pi_{i,k}^t(b_{k} \vert x_k) \Big\langle \pi_{i, h}^t(\cdot \vert x_{i, h}),  \indic{(x_h, \cdot) = (x_h^{t,(h,\wb{W})}, a_{h}^{t, (h, \wb{W})}) } \cdot \paren{H-h+1 - \sum_{h' = h}^H r_{i, h'}^{t, (h, \wb{W})}} \Big\rangle;
    \end{align*} }
    \item $\wt\Delta_t^{x_{h}, b_{1:h-1}}\le \Big\langle \pi_{i, h}^t(\cdot \vert x_{i, h}), \L_{i,h}^{t}(x_{i, h},\cdot) \Big\rangle \le H$;
    \item $\E[\wt \Delta_t^{x_{h}, b_{1:h-1}}|\cF_{t-1}]=0$, where $\cF_{t-1}$ is the $\sigma$-algebra containing all information after $\pi^t$ is sampled. This also can be seen from the  unbiasedness of $\wt{\L}$;
    \item The conditional variance $\E[(\wt\Delta_t^{x_{h}, b_{1:h-1}})^2|\cF_{t-1}]$ can be bounded as
  {\proofsize  \begin{align*}
    & ~\quad \E\brac{ \paren{\wt\Delta_t^{x_{h}, b_{1:h-1}}}^2 \Big| \cF_{t-1}} \\
    &\le \E\brac{ \paren{ H-h+1 - \sum_{h'=h}^H r_{i,h'}^{t, (h,\wb{W})} }^2 \paren{\prod_{k \in \wb{W}} \pi_{i,k}^t(b_{k} \vert x_k)  \Big\langle \pi_{i, h}^t(\cdot \vert x_{i, h}), \indic{(x_h, \cdot) = (x_h^{t,(h,\wb{W})}, a_{h}^{t, (h, \wb{W})}) } \Big\rangle}^2\Big| \cF_{t-1} }\\
    & \overset{(i)}{\le} H^2 \prod_{k \in \wb{W}} \pi_{i,k}^t(b_{k} \vert x_k) \cdot \E\brac{ \Big\langle \pi_{i, h}^t(\cdot \vert x_{i, h}), \indic{(x_h, \cdot) = (x_h^{t,(h,\wb{W})}, a_{h}^{t, (h, \wb{W})}) } \Big\rangle \Big| \cF_{t-1}} \\
    & = H^2 \prod_{k \in \wb{W}} \pi_{i,k}^t(b_{k} \vert x_k) \cdot \E\brac{ \sum_{a\in \cA_i} \pi_i^t(a|x_k) \indic{(x_h, a) = (x_h^{t,(h,\wb{W})}, a_{h}^{t, (h, \wb{W})})} \Big\vert \cF_{t-1}} \\
    & = H^2 \prod_{k \in \wb{W}} \pi_{i,k}^t(b_{k} \vert x_k) \cdot \sum_{a\in \cA_i} \pi_i^t(a|x_k) \P^{((\pi^{\star, h}_{i,k})_{k\in \wb{W}\cup \set{h}}(\pi^{t}_{i,k})_{k\in[h-1]\setminus \wb{W}} ) \times \pi_{-i}^t}(x_h^{t,(h,\wb{W})} = x_h , a_h^{t, (h, \wb{W})} = a)
    \\ 
    & = H^2 \prod_{k \in \wb{W}} \pi_{i,k}^t(b_{k} \vert x_k) \cdot \sum_{a\in \cA_i} \pi_i^t(a|x_k) \paren{\prod_{k \in \wb{W} } \pi_{i,k}^{\star,h}(a_{k} \vert x_k) \cdot\prod_{k \in [h-1]\setminus \wb{W}} \pi_{i,k}^{t}(a_{k} \vert x_k)\cdot \pi_{i,h}^{\star,h}(a|x_h) p^{\pi^t_{-i}}_{1:h}(x_{h})} \\ 
     &\overset{(ii)}{\le} H^2  \prod_{k \in [h-1]} \pi_{i,k}^t(b_{k} \vert x_k) p^{\pi^t_{-i}}_{1:h}(x_{i,h})\prod_{k \in \wb{W} \cup \set{h} } \pi_{i,k}^{\star,h}(a_{k} \vert x_k)   .
  \end{align*} }
  Here, (i) uses $\Big\langle \pi_{i, h}^t(\cdot \vert x_{i, h}), \indic{(x_h, \cdot) = (x_h^{t,(h,\wb{W})}, a_{h}^{t, (h, \wb{W})}) } \Big\rangle \in [0,1]$; (ii) is because $\prod_{k \in \wb{W}} \pi_{i,k}^t(b_{k} \vert x_k) \cdot \prod_{k \in [h-1]\setminus \wb{W}} \pi_{i,k}^{t}(a_{k} \vert x_k) = \prod_{k \in [h-1]} \pi_{i,k}^t(b_{k} \vert x_k)$. 
\end{itemize}
Therefore, we can apply Freedman's inequality and union bound to get that  for any fixed $\lambda \in (0, 1/H]$, with probability at least $1-p/8$, the following holds simultaneously for all $(h,x_{i, h},  b_{1:h-1})$:

{\proofsize  \begin{align*}
    \sum_{t = 1}^T \wt\Delta_t^{x_h, b_{1:h-1}} \le&~ \lambda H^2 \prod_{k \in \wb{W} \cup \set{h} } \pi_{i,k}^{\star,h}(a_{k} \vert x_k) \sum_{t=1}^T \prod_{k \in [h-1]} \pi_{i,k}^t(b_{k} \vert x_k) p^{\pi^t_{-i}}_{1:h}(x_{i,h})
    + \frac{C\log(X_i A_i/p)}{\lambda},
\end{align*} }
where $C> 0$ is some absolute constant. Plugging this bound into ${\rm BIAS}_{1, h}^{T, \swap}$ yields that, for all $h\in[H]$,
{\proofsize \begin{equation*}\label{equation:bias1h-after-freedman}
    \begin{aligned}
    &~{\rm BIAS}_{1, h}^{T, \swap} \le \max_{\muphi \in \Phi_i^K} \sum_{x_{h}\in \cX_{i, h}} \sum_{b_{1:h-1}} \indicDleK  \frac{\prod_{k=1}^{h-1} \muphi_k(a_{k} \vert x_k, b_{1:k})}{ \prod_{k\in \wb{W} \cup \set{h}} \pi^{\star,h}_{i, k}(a_{k} \vert x_k )}\\
     &\times \brac{
    \lambda H^2 \prod_{k \in \wb{W} \cup \set{h} } \pi_{i,k}^{\star,h}(a_{k} \vert x_k) \sum_{t=1}^T \prod_{k \in [h-1]} \pi_{i,k}^t(b_{k} \vert x_k) p^{\pi^t_{-i}}_{1:h}(x_{i,h})
    + \frac{C \log(X_i A_i/p)}{\lambda}}\\
    \le~ & \lambda H^2  \max_{\muphi \in \Phi_i^K} \sum_{x_{h}\in \cX_{i, h}} \sum_{b_{1:h-1}}\indicDleK \prod_{k=1}^{h-1} \muphi_k(a_{k} \vert x_k, b_{1:k}) \sum_{t = 1}^T \prod_{k \in [h-1]} \pi_{i,k}^t(b_{k} \vert x_k) p^{\pi^t_{-i}}_{1:h}(x_{i,h}) \\
    &+ \frac{C \log(X_i A_i/p)}{\lambda} \max_{\muphi \in \Phi_i^K} \sum_{x_{h}\in \cX_{i, h}} \sum_{b_{1:h-1}} \indicDleK  \frac{\prod_{k=1}^{h-1} \muphi_k(a_{k} \vert x_k, b_{1:k})}{ \prod_{k\in \wb{W} \cup \set{h}} \pi^{\star,h}_{i, k}(a_{k} \vert x_k )}.
\end{aligned}
\end{equation*}
}By Corollary \ref{lemma:sum=1}, for fixed $t$, {\proofsize
$$\sum_{x_{h}\in \cX_{i, h}} \sum_{b_{1:h-1}}\indicDleK \prod_{k=1}^{h-1} \muphi_k(a_{k} \vert x_k, b_{1:k}) \prod_{k \in [h-1]} \pi_{i,k}^t(b_{k} \vert x_k) p^{\pi^t_{-i}}_{1:h}(x_{i,h}) \le 1.
$$}So we have
{\proofsize \begin{equation*}
    \begin{gathered}
        \lambda H^2  \max_{\muphi \in \Phi_i^K} \sum_{x_{h}\in \cX_{i, h}} \sum_{b_{1:h-1}}\indicDleK \prod_{k=1}^{h-1} \muphi_k(a_{k} \vert x_k, b_{1:k}) \sum_{t = 1}^T \prod_{k \in [h-1]} \pi_{i,k}^t(b_{k} \vert x_k) p^{\pi^t_{-i}}_{1:h}(x_{i,h}) \le \lambda H^2  T.
    \end{gathered}
\end{equation*} 
}Moreover, by the inequality (\ref{equation:swap-mu-over-pi-bound}) in the proof of Lemma \ref{lemma:bound-regret^I}, we have
{\proofsize  \begin{align*}
    \sum_{x_{h}\in \cX_{i, h}} \sum_{b_{1:h-1}} \indicDleK  \frac{\prod_{k=1}^{h-1} \muphi_k(a_{k} \vert x_k, b_{1:k})}{ \prod_{k\in \wb{W} \cup \set{h}} \pi^{\star,h}_{i, k}(a_{k} \vert x_k )} \le  X_{i,h} {h-1 \choose (K-1) \wedge (h-1)}A_i^{K \wedge h}.
\end{align*} }
Plugging these bounds into ${\rm BIAS}_{1, h}^{T, \swap}$ yields that,  with probability at least $1-p/8$, for all $h\in [H]$,
{\proofsize  \begin{align*}
    {\rm BIAS}_{1, h}^{T, \swap} \le \lambda H^2 T + \frac{C \log(X_iA_i/p)}{\lambda} X_{i,h}  {h-1 \choose (K-1) \wedge (h-1)}A_i^{ K \wedge h}.
\end{align*} }
Taking summation over $h\in[H]$, we have
{\proofsize  \begin{align*}
    \sum_{h=1}^H   {\rm BIAS}_{1, h}^{T, \swap} \le&~ \lambda H^3 T + \frac{C \log(X_iA_i/p)}{\lambda}\sum_{h=1}^H X_{i,h}  {h-1\choose (K-1) \wedge (h-1)}A_i^{K \wedge h}\\
    \le &~ \lambda H^3   T + \frac{C \log(X_iA_i/p)}{\lambda}X_i {H-1 \choose K\wedge H -1}A_i^{K\wedge H }\\
    \le &~\lambda H^3   T + \frac{C \log(X_iA_i/p)}{\lambda}X_i {H \choose K\wedge H }A_i^{K\wedge H },
\end{align*} }
for all $\lambda\in (0,1/H].$ 
Choose $\lambda = \min\set{\frac{1}{H}, \sqrt{\frac{CX_i  {H \choose K\wedge H}A_i^{K\wedge H } \log(X_iA_i/p)}{H^3 T}}}$, we obtain the bound 
{\proofsize  \begin{align*}
       \sum_{t=1}^T {\rm BIAS}_{1, h}^{T, \swap} \le \mc{O}\paren{\sqrt{H^3 {H \choose K\wedge H} A_i^{K\wedge H} X_iT \iota}+ H {H \choose K\wedge H }A_i^{K\wedge H } X_i\iota},
\end{align*}
}where $\iota = \log(8X_iA_i/p)$ is a log factor.
\end{proof}

\subsubsection{Proof of Lemma \ref{lemma:bound-bias2^I}: Bound on ${\rm BIAS}_{2, h}^{T, \swap}$}
\label{appendix:proof-bound-bias2^I}
\begin{proof}
For fixed $x_h$ and $b_{1:h-1}$, let $W = \set{k\in [h-1]: b_k\not= a_k}$ and $\wb{W} = \fill(W, (K-1) \wedge (h-1))$
We can rewrite ${\rm BIAS}_{2, h}^{T, \swap}$ as 
{\proofsize  \begin{align*}
    {\rm BIAS}_{2, h}^{T, \swap}
    =~&\max_{\muphi\in \Phi_i^K} \sum_{x_{h}\in \cX_{i, h}} \sum_{b_{1:h-1}} \indicDleK \prod_{k=1}^{h-1} \muphi_k(a_{k} \vert x_k, b_{1:k}) \\
    &\times \max_\vphi\sum_{t = 1}^T \prod_{k = 1}^{h-1} \pi_i^t (b_{k}|x_{k})
    \Big\langle \vphi \circ \pi_{i, h}^t(\cdot \vert x_{i, h}), \wt L^t_{(x_{h}, b_{1: h-1})}-\L_{i,h}^{t}(x_{i, h},\cdot) \Big\rangle \\
    =& \max_{\muphi\in \Phi_i^K} \sum_{x_{h}\in \cX_{i, h}} \sum_{b_{1:h-1}} \indicDleK \prod_{k=1}^{h-1} \muphi_k(a_{k} \vert x_k, b_{1:k}) \\
    &~\times \max_\vphi \sum_{t = 1}^T \prod_{k = 1}^{h-1} \pi_i^t (b_{k}|x_{k}) \sum_{b_h} \pi_{i,h}^t(b_h|x_h)
    \paren{ \wt L^t_{(x_{h}, b_{1: h-1})}(\vphi(b_h))-\L_{i,h}^{t}(x_{i, h},\vphi(b_h))} \\
    \overset{}{\le}& \max_{\muphi\in \Phi_i^K} \sum_{x_{h}\in \cX_i } \sum_{b_{1:h-1}}  \indicDleK  \prod_{k=1}^{h-1} \muphi_k(a_{k} \vert x_k, b_{1:k}) \\
    &~\times \max_{\phi_h'} \sum_{t = 1}^T \prod_{k = 1}^{h-1} \pi_i^t (b_{k}|x_{k}) \sum_{b_h,a_h} \pi_{i,h}^t(b_h|x_h) \phi_h'(a_h|x_h,b_{1:h})
    \paren{ \wt L^t_{(x_{h}, b_{1: h-1})}(a_h)-\L_{i,h}^{t}(x_{i, h},a_h)} \\
    \overset{(i)}{=}& \max_{\muphi\in \Phi_i^K}\sum_{x_{i, h}, a_h} \sum_{b_{1:h-1}}\sum_{b_h} \indicDleK  \prod_{k=1}^{h} \muphi_k(a_{k} \vert x_k, b_{1:k})\\
    &~\times \sum_{t = 1}^T  \prod_{k = 1}^{h} \pi_i^t (b_{k}|x_{k}) \paren{ \wt L^t_{(x_{h}, b_{1: h-1})}(a_h)-\L_{i,h}^{t}(x_{i, h},a_h)} \\
    =& \max_{\muphi\in \Phi_i^K}\sum_{x_{i, h}, a_h} \sum_{b_{1:h-1}}\sum_{b_h} \indicDleK  \frac{\prod_{k=1}^{h} \muphi_k(a_{k} \vert x_k, b_{1:k})}{\prod_{ k \in \wb{W}\cup \set{h} } \pi^{\star,h}_{i, k}(a_{k} \vert x_k )} \\
    &~\times \sum_{t = 1}^T   \underbrace{\prod_{k = 1}^{h} \pi_i^t (b_{k}|x_{k}) \paren{ \wt L^t_{(x_{h}, b_{1: h-1})}(a_h)-\L_{i,h}^{t}(x_{i, h},a_h)}  \prod_{k \in \wb{W}\cup \set{h} } \pi^{\star,h}_{i, k}(a_{k} \vert x_k )}_{\defeq \wt\Delta_t^{x_{h}, b_{1:h},a_h}}.
\end{align*} 
}Here, (i) comes from the fact that  the inner max over $\phi_h'$ and the outer max over $\phi_{1:h-1}$ are separable and thus can be merged into a single max over $\phi_{1:h}$. 

Observe that the random variable $\wt\Delta_t^{x_{i,h}, b_{1:h}, a_h}$ satisfy the following:
\begin{itemize}
    \item By the definition of $\wt \piL$ in Algorithm \ref{algorithm:kcfr-estimator-i}, we can rewrite $\wt \Delta_t^{x_{h}, b_{1:h}, a_h}$ as
    {\proofsize  \begin{align*}
        &~~\wt \Delta_t^{x_{h}, b_{1:h}, a_h} \\
        &=  \prod_{k \in \wb{W} \cup \{h\}} \pi_{i,k}^t(b_{k} \vert x_k) \indic{(x_h, a_h) = (x_h^{t,(h,\wb{W})}, a_{h}^{t, (h, \wb{W})}) } \cdot \paren{H-h+1 - \sum_{h' = h}^H r_{i, h'}^{t, (h, \wb{W})}} \\
         &~~~- \prod_{k \in [h]} \pi_{i,k}^t(b_{k} \vert x_k)    \prod_{k \in \wb{W}\cup \set{h} } \pi^{\star,h}_{i, k}(a_{k} \vert x_k )L_{i,h}^{t}(x_{i,h}, a_h).
    \end{align*} }
    \item $\wt\Delta_t^{x_{h}, b_{1:h}, a_h}\in [-H, H]$.
    
    \item $\E[\wt \Delta_t^{x_{h}, b_{1:h}, a_h }|\cF_{t-1}]=0$, where $\cF_{t-1}$ is the $\sigma$-algebra containing all information after $\pi^t$ is sampled.
    
    \item The conditional variance $\E[(\wt\Delta_t^{x_{h}, b_{1:h}, a_h})^2|\cF_{t-1}]$ can be bounded as
  {\proofsize  \begin{align*}
    & ~\quad \E\brac{ \paren{\wt\Delta_t^{x_{h}, b_{1:h-1}, a_h}}^2 \Big| \cF_{t-1}} \\
    &\le \E\brac{ \paren{ H-h+1 - \sum_{h'=h}^H r_{i,h'}^{t, (h,\wb{W})} }^2 \paren{ \prod_{k \in \wb{W} \cup \{h\}} \pi_{i,k}^t(b_{k} \vert x_k) \indic{(x_h, a_h) = (x_h^{t,(h,\wb{W})}, a_{h}^{t, (h, \wb{W})}) } }^2\Big| \cF_{t-1} }\\
    & \le H^2 \prod_{k \in \wb{W} \cup \set{h}} \pi_{i,k}^t(b_{k} \vert x_k)   \P^{((\pi^{\star, h}_{i,k})_{k\in \wb{W}\cup \set{h}}(\pi^{t}_{i,k})_{k\in[h-1]\setminus \wb{W}} ) \times \pi_{-i}^t}\paren{ (x_{h}^{t, (h,\wb{W})}, a_{h}^{t,(h, \wb{W})})=(x_{h}, a_h) } \\
    & = H^2 \prod_{k \in \wb{W} \cup \set{h}} \pi_{i,k}^t(b_{k} \vert x_k)  \prod_{k \in \wb{W}  } \pi_{i,k}^{\star,h}(a_{k} \vert x_k) \cdot\prod_{k \in [h-1]\setminus \wb{W}} \pi_{i,k}^{t}(a_{k} \vert x_k)\cdot \pi_{i,h}^{\star, h}(a_h | x_h)  p^{\pi^t_{-i}}_{1:h}(x_{i,h}) \\
     &\overset{(i)}{=} H^2 \prod_{k \in [h]} \pi_{i,k}^t(b_{k} \vert x_k) p^{\pi^t_{-i}}_{1:h}(x_{i,h})\prod_{k \in \wb{W}\cup \set{h} } \pi_{i,k}^{\star,h}(a_{k} \vert x_k) .
  \end{align*} }
  Here, (i) is because $\prod_{k \in \wb{W} \cup \set{h}} \pi_{i,k}^t(b_{k} \vert x_k) \cdot \prod_{k \in [h-1]\setminus \wb{W}} \pi_{i,k}^{t}(a_{k} \vert x_k) = \prod_{k \in [h]} \pi_{i,k}^t(b_{k} \vert x_k)$ .
\end{itemize}
Therefore, we can apply Freedman's inequality and union bound to get that  for any fixed $\lambda \in (0, 1/H]$, with probability at least $1-p/8$, the following holds simultaneously for all $(h,x_{i, h},  b_{1:h}, a_h)$: 

{\proofsize  \begin{align*}
    \sum_{t = 1}^T \wt\Delta_t^{x_h, b_{1:h}, a_h} \le& ~\lambda H^2 \prod_{k \in \wb{W}\cup \set{h} } \pi_{i,k}^{\star,h}(a_{k} \vert x_k) \sum_{t=1}^T \prod_{k \in [h]} \pi_{i,k}^t(b_{k} \vert x_k) p^{\pi^t_{-i}}_{1:h}(x_{i,h})
    + \frac{C \log(X_i A_i/p)}{\lambda},
\end{align*} 
} where $C> 0$ is some absolute constant. Plugging this bound into ${\rm BIAS}_{2, h}^{T, \swap}$ yields that, for all $h\in[H]$ and $\muphi \in \Phi_i^K$,
{\proofsize \begin{equation*}
    \begin{aligned}
    &~{\rm BIAS}_{2, h}^{T, \swap} \le  \max_{\muphi\in\Phi_i^K}\sum_{x_{h}, a_h} \sum_{b_{1:h}} \indicDleK  \frac{\prod_{k=1}^{h} \muphi_k(a_{k} \vert x_k, b_{1:k})}{ \prod_{k\in \wb{W}\cup \set{h}} \pi^{\star,h}_{i, k}(a_{k} \vert x_k )}\\
     &\times \brac{
    \lambda H^2  \prod_{k \in \wb{W}\cup \set{h} } \pi_{i,k}^{\star,h}(a_{k} \vert x_k) \sum_{t=1}^T \prod_{k \in [h]} \pi_{i,k}^t(b_{k} \vert x_k) p^{\pi^t_{-i}}_{1:h}(x_{i,h})
    + \frac{C \log(X_i A_i/p)}{\lambda}}\\
    \le~ & \max_{\muphi\in\Phi_i^K} \lambda H^2 \sum_{x_{h}, a_h} \sum_{b_{1:h}}\indicDleK \prod_{k=1}^{h} \muphi_k(a_{k} \vert x_k, b_{1:k}) \sum_{t = 1}^T \prod_{k \in [h]} \pi_{i,k}^t(b_{k} \vert x_k) p^{\pi^t_{-i}}_{1:h}(x_{i,h}) \\
    &+\max_{\muphi\in\Phi_i^K} \frac{C \log(X_i A_i/p)}{\lambda}\sum_{x_{h},a_h} \sum_{b_{1:h}} \indicDleK  \frac{\prod_{k=1}^{h} \muphi_k(a_{k} \vert x_k, b_{1:k})}{ \prod_{k\in \wb{W}\cup \set{h}} \pi^{\star,h}_{i, k}(a_{k} \vert x_k )}.
\end{aligned}
\end{equation*} }
By Corollary \ref{lemma:sum=1}, for fixed $t$, {\proofsize
$$\sum_{x_{h}\in \cX_{i, h}} \sum_{b_{1:h-1}}\indicDleK \prod_{k=1}^{h-1} \muphi_k(a_{k} \vert x_k, b_{1:k}) \prod_{k \in [h-1]} \pi_{i,k}^t(b_{k} \vert x_k)p^{\pi^t_{-i}}_{1:h}(x_{i,h}) \le 1.
$$} so we have
{\proofsize
\begin{align*}
    &~\sum_{x_{h}, a_h} \sum_{b_{1:h}}\indicDleK \prod_{k=1}^{h} \muphi_k(a_{k} \vert x_k, b_{1:k}) \sum_{t = 1}^T \prod_{k \in [h]} \pi_{i,k}^t(b_{k} \vert x_k) p^{\pi^t_{-i}}_{1:h}(x_{i,h}) \\
    &~ =\sum_{x_{h}} \sum_{b_{1:h-1}}\indicDleK \prod_{k=1}^{h-1} \muphi_k(a_{k} \vert x_k, b_{1:k}) \sum_{t = 1}^T \prod_{k \in [h-1]} \pi_{i,k}^t(b_{k} \vert x_k) p^{\pi^t_{-i}}_{1:h}(x_{i,h})\le T.
\end{align*}
} Moreover, by the inequality (\ref{equation:swap-mu-over-pi-bound}) in the proof of Lemma \ref{lemma:bound-regret^I}, we have
{\proofsize  \begin{align*}
    &\sum_{x_{h},a_h} \sum_{b_{1:h}} \indicDleK  \frac{\prod_{k=1}^{h} \muphi_k(a_{k} \vert x_k, b_{1:k})}{ \prod_{k\in \wb{W}\cup \set{h}} \pi^{\star,h}_{i, k}(a_{k} \vert x_k )}\\
    \overset{(i)}{=}&A_i \sum_{x_{h}\in \cX_{i, h}} \sum_{b_{1:h-1}} \indicDleK  \frac{\prod_{k=1}^{h-1} \muphi_k(a_{k} \vert x_k, b_{1:k})}{ \prod_{k\in \wb{W}\cup \set{h}} \pi^{\star,h}_{i, k}(a_{k} \vert x_k )}\\
    \overset{ }{\le}&  X_{i,h} {h-1\choose K\wedge H -1}A_i^{K\wedge h+1}.
\end{align*} }

 Plugging these bounds into ${\rm BIAS}_{2, h}^{T, \swap}$ yields that,  with probability at least $1-p/8$, for all $h\in [H]$,
{\proofsize  \begin{align*}
    {\rm BIAS}_{2, h}^{T, \swap} \le \lambda H^2 T + \frac{C \log(X_iA_i/p)}{\lambda} X_{i,h}  {h-1\choose K\wedge H -1} A_i^{K\wedge h+1}.
\end{align*} }

Taking summation over $h\in[H]$, we have
{\proofsize  \begin{align*}
    \sum_{h=1}^H   {\rm BIAS}_{2, h}^{T, \swap} \le&~ \lambda H^3  T + \frac{C \log(X_iA_i/p)}{\lambda}\sum_{h=1}^H X_{i,h}  {h-1\choose (K-1) \wedge (h-1)}A_i^{K\wedge h+1}\\
    \le &~ \lambda H^3  T + \frac{C \log(X_iA_i/p)}{\lambda}X_i {H-1 \choose K\wedge H -1}A_i^{K\wedge H+1}\\
    \le & ~ \lambda H^3  T + \frac{C \log(X_iA_i/p)}{\lambda}X_i {H \choose K\wedge H }A_i^{K\wedge H+1},
\end{align*} }
for all $\lambda\in (0,1/H].$ 
Choose $\lambda = \min\set{\frac{1}{H}, \sqrt{\frac{CX_i  {H \choose K\wedge H }A_i^{K\wedge H+1} \log(X_iA_i/p)}{H^3 T}}}$, we obtain the bound 
{\proofsize  \begin{align*}
       \sum_{t=1}^T {\rm BIAS}_{2, h}^{T, \swap} \le \mc{O}\paren{\sqrt{H^3 {H \choose K\wedge H} A_i^{K\wedge H+1} X_iT \iota}+ H {H \choose K\wedge H }A_i^{K\wedge H+1} X_i\iota},
\end{align*} }
where $\iota = \log(8X_iA_i/p)$ is a log factor.
\end{proof}

\subsection{Proof of Lemma \ref{lemma:bound-G2-bandit}}
\label{appendix:bound-G2-bandit}
\begin{proof}
Throughout the proof, for $x_h$ and $b_{1:h-1}$, let $W = \set{k\in [h-1]: b_k\not= a_k}$ and $h'$ is the maximal element in $W$, so we have $h' = \tau_K$. Recall that $G_{h}^{T, \ext}(x_h; \muphi)$ (eq. (\ref{equation:local-external-regret})) is defined as 
$$
G_{h}^{T, \ext}(x_h; \muphi) \equiv~ \sum_{b_{1:h-1}} \indicDeqK \prod_{k=1}^{h-1} \muphi_k(a_k \vert x_k, b_{1:k \wedge \tau_K}) \what{R}^{T, \ext}_{(x_h, b_{1:\tau_K}), x_h}.
$$
For each $h\in [H]$ and $(x_{h}, b_{1: h'}) \in \omegaii$, we have we have 
{\proofsize  \begin{align*}
    \hat R_{(x_{h}, b_{1,h'})}^{T, \rm ext} =  &\max_a \sum_{t = 1}^T \prod_{k = 1}^{h'} \pi_i^t (b_{k}|x_{k}) \Big(  \< \pi_{i, h}^t(\cdot \vert x_{i, h}), L_{i,h}^{t}(x_{i, h},\cdot) \>- \L_{i,h}^{t}(x_{i, h}, a) \Big) \\
    \le & \max_a \sum_{t = 1}^T \prod_{k = 1}^{h'} \pi_i^t (b_{k}|x_{k})  
    \paren{\Big\langle \pi_{i, h}^t(\cdot \vert x_{i, h})  , \wt L_{(x_{h}, b_{1: h'})}^{t} \Big\rangle- \wt L_{(x_{h}, b_{1: h'})}^{t}(a)}\\
    &+ \sum_{t = 1}^T \prod_{k = 1}^{h'} \pi_i^t (b_{k}|x_{k})
    \Big\langle\pi_{i, h}^t(\cdot \vert x_{i, h}),  \L_{i,h}^{t}(x_{i, h},\cdot) -\wt L^t_{(x_{h}, b_{1: h'})} \Big\rangle \\
    &+ \max_a \sum_{t = 1}^T \prod_{k = 1}^{h'} \pi_i^t (b_{k}|x_{k})
    \paren{ \wt L^t_{(x_{h}, b_{1: h'})}(a)- \L_{i,h}^{t}(x_{i, h},a) }.
\end{align*} }
Substituting this into $\max_{\muphi \in \Phi_i^K}\sum_{x_{h}\in \cX_{i, h}}G_{h}^{T, \ext}(x_{h};\muphi)$ yields that
{\scriptsize
{\proofsize  \begin{align*}
    & \max_{\muphi \in \Phi_i^K}\sum_{x_{h}\in \cX_{i, h}}G_{h}^{T, \ext}(x_{h};\muphi)\\
    =~& \max_{\muphi \in \Phi_i^K} \sum_{x_{h}\in \cX_{i, h}} \sum_{b_{1:h-1}} \indicDeqK \prod_{k=1}^{h-1} \muphi_k(a_k \vert x_k, b_{1:k \wedge \tau_K}) \what{R}^{T, \ext}_{(x_h, b_{1:\tau_K}), x_h}\\
    \le~& \underbrace{\max_{\muphi \in \Phi_i^K} \sum_{x_{h}\in \cX_{i, h}} \sum_{b_{1:h-1}} \indicDeqK \prod_{k=1}^{h-1} \muphi_k(a_{k} \vert x_k, b_{1: k \wedge \tau_K })  \max_a \sum_{t = 1}^T \prod_{k = 1}^{\tau_K} \pi_i^t (b_{k}|x_{k})  
    \paren{\Big\langle \pi_{i, h}^t(\cdot \vert x_{i, h})  , \wt L_{(x_{h}, b_{1: \tau_K})}^{t} \Big\rangle- \wt L_{(x_{h}, b_{1: \tau_K})}^{t}(a)}}_{\defeq \wt{\rm REGRET}_h^{T, \ext}}\\
    &+  \underbrace{ \max_{\muphi \in \Phi_i^K} \sum_{x_{h}\in \cX_{i, h}} \sum_{b_{1:h-1}} \indicDeqK \prod_{k=1}^{h-1} \muphi_k(a_{k} \vert x_k, b_{1: k \wedge \tau_K}) \sum_{t = 1}^T \prod_{k = 1}^{\tau_K} \pi_i^t (b_{k}|x_{k})
    \Big\langle\pi_{i, h}^t(\cdot \vert x_{i, h}),  \L_{i,h}^{t}(x_{i, h},\cdot) -\wt L^t_{(x_{h}, b_{1: \tau_K})} \Big\rangle}_{\defeq {\rm BIAS}_{1, h}^{T, \ext}}  \\
    &+ \underbrace{ \max_{\muphi \in \Phi_i^K} \sum_{x_{h}\in \cX_{i, h}} \sum_{b_{1:h-1}} \indicDeqK \prod_{k=1}^{h-1} \muphi_k(a_{k} \vert x_k, b_{1: k \wedge \tau_K}) \max_a \sum_{t = 1}^T \prod_{k = 1}^{\tau_K} \pi_i^t (b_{k}|x_{k})
    \paren{ \wt L^t_{(x_{h}, b_{1: \tau_K})}(a)- \L_{i,h}^{t}(x_{i, h},a) }}_{\defeq {\rm BIAS}_{2, h}^{T, \ext}}.
\end{align*} }
}The bounds for $\sum_{h=1}^H \wt{\rm REGRET}_h^{T, \ext}$, $\sum_{h=1}^H {\rm BIAS}_{1, h}^{T, \ext}$, and $\sum_{h=1}^H {\rm BIAS}_{2, h}^{T, \ext}$  are given in the following three Lemmas (proofs deferred to Appendix~\ref{appendix:proof-bound-regret^II}, \ref{appendix:proof-bound-bias1^II}, and \ref{appendix:proof-bound-bias2^II}) respectively.


\begin{lemma}[Bound on $\wt{\rm REGRET}_h^{T, \ext}$]\label{lemma:bound-regret^II}
If we choose learning rates as 
$$
\eta_{x_h} = \sqrt{{H \choose K \wedge H } X_i A_i^{K\wedge H + 1}  \log (8X_iA_i/p) / (H^3T) }
$$ for all $x_h\in \cX_i$ (same with \eqref{equation:learningrate-bandit}).
 Then with probability at least $1-p/4$, we have
{\proofsize  \begin{align*}
     \sum_{h=1}^H   \wt{\rm REGRET}_{h}^{T, \ext} &\le \sqrt{H^3 {H \choose K\wedge H }A_i^{K \wedge H + 1 
  } X_i T \iota }\\
    + &\mc{O}  \paren{ {H \choose K\wedge H }  A_i^{K \wedge H + 1 } X_i   \iota \sqrt{\frac{H {H \choose K\wedge H }  A_i^{K \wedge H +1 } X_i \iota }{T}}},
\end{align*} }
where $\iota = \log (8X_i A_i/p)$ is a log factor.
\end{lemma}

\begin{lemma}[Bound on ${\rm BIAS}_{1, h}^{T, \ext}$]\label{lemma:bound-bias1^II}
With probability at least $1-p/8$, we have
{\proofsize  \begin{align*}
    \sum_{h=1} ^H {\rm BIAS}_{1, h}^{T, \ext}\le \mc{O}\paren{\sqrt{H^3 {H \choose K\wedge H } A_i^{K\wedge H+1} X_iT \iota}+ H {H \choose K\wedge H }A_i^{K\wedge H+1 } X_i\iota}, 
\end{align*} }
where $\iota = \log(8X_iA_i/p)$ is a log factor.
\end{lemma}

\begin{lemma}[Bound on ${\rm BIAS}_{2, h}^{T, \ext}$]\label{lemma:bound-bias2^II}
With probability at least $1-p/8$, we have
{\proofsize  \begin{align*}
    \sum_{h=1} ^H {\rm BIAS}_{2, h}^{T, \ext}\le \mc{O}\paren{\sqrt{H^3 {H \choose K\wedge H } A_i^{K\wedge H} X_iT \iota}+ H {H \choose K\wedge H }A_i^{K\wedge H } X_i\iota}, 
\end{align*} }
where $\iota = \log(8X_iA_i/p)$ is a log factor.

\end{lemma}

Combining Lemma~\ref{lemma:bound-regret^II},~\ref{lemma:bound-bias1^II}, and~\ref{lemma:bound-bias2^II}, we have with probability at least $1-p/2$ that

 {\proofsize  \begin{align*}
     \sum_{h=1}^H \max_{\phi \in \Phi_i^K}\sum_{x_{h}\in \cX_{i, h}}G_{h}^{T, \ext}(x_{h};\muphi) &\le  \sum_{h=1}^H \wt{\rm REGRET}_h^{T, \ext} +  \sum_{h=1}^H {\rm BIAS}_{1, h}^{T, \ext} +  \sum_{h=1}^H {\rm BIAS}_{2, h}^{T, \ext}\\
     &\le \sqrt{H^3 {H \choose K\wedge H }A_i^{K\wedge H   +1} X_i T \iota }\\
    &+ \mc{O}  \paren{ {H \choose K\wedge H }  A_i^{K\wedge H  + 1 } X_i   \iota \sqrt{\frac{H {H \choose K\wedge H }  A_i^{K\wedge H   +1 } X_i \iota }{T}}}\\
    &+  \mc{O}  \paren{ H {H \choose K\wedge H }A_i^{K\wedge H + 1} X_i\iota}.
 \end{align*} }

\end{proof}

\subsubsection{Proof of Lemma \ref{lemma:bound-regret^II}: Bound on $\wt{\rm REGRET}_h^{T, \ext}$}
\label{appendix:proof-bound-regret^II}
\begin{proof}
Recall that $\wt{\rm REGRET}_h^{T, \ext}$ is defined as 
{\proofsize
\begin{align*}
    \max_{\muphi \in \Phi_i^K} \sum_{x_{h}\in \cX_{i, h}} \sum_{b_{1:h-1}} \indicDeqK \prod_{k=1}^{h-1} \muphi_k(a_{k} \vert x_k, b_{1: k \wedge \tau_K }) \max_a \sum_{t = 1}^T \prod_{k = 1}^{\tau_K} \pi_i^t (b_{k}|x_{k})  
    \paren{\Big\langle \pi_{i, h}^t(\cdot \vert x_{i, h})  , \wt L_{(x_{h}, b_{1: \tau_K})}^{t} \Big\rangle- \wt L_{(x_{h}, b_{1: \tau_K})}^{t}(a)}.
\end{align*}
}We first apply regret minimization lemma (Lemma \ref{lemma:regret-with-time-selection-bandit}) on $\regmin_{x_h}$ to give an upper bound of 
$$
\max_a \sum_{t = 1}^T \prod_{k = 1}^{\tau_K} \pi_i^t (b_{k}|x_{k})  
    \paren{\Big\langle \pi_{i, h}^t(\cdot \vert x_{i, h})  , \wt L_{(x_{h}, b_{1: \tau_K})}^{t} \Big\rangle- \wt L_{(x_{h}, b_{1: \tau_K})}^{t}(a)}.
$$
For $\regmin_{x_h}$, Algorithm \ref{algorithm:kefr-sampled-bandit} gives that $M^t_{b_{1:h'}} = \prod_{k = 1}^{h'} \pi_i^t (b_{k}|x_{k})$ and $w_{b_{1:h'}} = \prod_{k \in  W \cup \set{h'+1,\dots,h}} \pi^{\star, h}_{i,k}(a_k \vert x_k) $ for any $b_{1:h'} \in \omegaii(x_h)$,  where $W = \set{k\in [h-1]: b_k\not= a_k}$. Letting  $W(h) = W \cup \set{h'+1,\dots,h}$, $\eta$ be the learning rate of $\regmin_{x_h}$ and $\wb L = H$. The assumptions in Lemma \ref{lemma:regret-with-time-selection-bandit} are verified in Section \ref{appendix:proof-bound-regret^I}. So by Lemma \ref{lemma:regret-with-time-selection-bandit}, with probability at least $1-p/8$, we have for all $x_h \in cX_i$,
{\proofsize  \begin{align*}
    &\max_a \sum_{t = 1}^T \prod_{k = 1}^{\tau_K} \pi_i^t (b_{k}|x_{k})  
    \paren{\Big\langle \pi_{i, h}^t(\cdot \vert x_{i, h})  , \wt L_{(x_{h}, b_{1: \tau_K})}^{t} \Big\rangle- \wt L_{(x_{h}, b_{1: \tau_K})}^{t}(a)} \\ 
    \le ~&\frac{2 \log (8X_iA_i/p)}{\eta w_{b_{1:h'}}} + \eta\sum_{t=1}^T M^t_{b_{1:h'}} \wb{L} \Big\langle \pi_i^t (\cdot|x_{k}), \wt L_{(x_{h}, b_{1: '})}^{t}(\cdot) \Big\rangle\\
    = ~&\frac{2\log (8X_iA_i/p)}{\eta\prod_{k\in \wh} \pi^{\star,h}_{i, k}(a_{k} \vert x_k )} + H \eta \prod_{k = 1}^{h'} \pi_i^t (b_{k}|x_{k}) \\
    &~~~ \times \sum_{t=1}^T \Big\langle \pi_{i, h}^t(\cdot \vert x_{i, h}), \frac{ \indic{(x_h, \cdot) = (x_h^{t,(h, h', W)}, a_{h}^{t, (h, h', W)}) }\sum_{h'' = h}^H \paren{1-r_{i, h''}^{t, (h, h', W)}}}{ \prod_{k\in \wh} \pi^{\star,h}_{i, k}(a_{k} \vert x_k ) \prod_{k \in [h']\setminus W} \pi_{i,k}^t(b_{k} \vert x_k) } \Big\rangle\\
    \le~&\frac{ 2\log (8X_iA_i/p) }{ \eta \prod_{k\in \wh} \pi^{\star,h}_{i, k}(a_{k} \vert x_k )} + H^2 \eta \sum_{t=1}^T \Big\langle \pi_i^t (\cdot|x_{k}), \frac{\prod_{k \in W} \pi_{i,k}^t(b_{k} \vert x_k) \indic{(x_h, a) = (x_h^{t,(h, h', W)}, a_{h}^{t, (h, h', W)}) }}{\prod_{k\in \wh} \pi^{\star,h}_{i, k}(a_{k} \vert x_k )} \Big\rangle.
\end{align*} 
}
Here, we use $(i)$ our choices of $M_{b_{1:h'}}^t$ and $w_{b_{1:h'}}$; $(ii) ~(|\cB^s|+|\cB^e|) |\Psi^e| \le A_i^{H+1} \le X_i A_i$ and $(iii)$  taking union bound over all infosets.
Plugging this into $\wt{\rm REGRET}_{h}^{T, \ext}$, we have 
{\proofsize  \begin{align*}
    &\wt{\rm REGRET}_{h}^{T, \ext}\\=~& \max_{\muphi \in \Phi_i^K} \sum_{x_{h}\in \cX_{i, h}} \sum_{b_{1:h-1}} \indicDeqK \prod_{k=1}^{h-1} \muphi_k(a_{k} \vert x_k, b_{1: k \wedge h'}) \\
    &\times \max_a \sum_{t = 1}^T \prod_{k = 1}^{\tau_K} \pi_i^t (b_{k}|x_{k})  
    \paren{\Big\langle \pi_{i, h}^t(\cdot \vert x_{i, h})  , \wt L_{(x_{h}, b_{1: \tau_K})}^{t} \Big\rangle- \wt L_{(x_{h}, b_{1: \tau_K})}^{t}(a)}\\
    \le~& \frac{2 \log (8X_iA_i/p)}{\eta} \max_{\muphi \in \Phi_i^K} \sum_{x_{h}\in \cX_{i, h}} \sum_{b_{1:h-1}} \indicDeqK  \frac{\prod_{k=1}^{h-1} \muphi_k(a_{k} \vert x_k, b_{1: k \wedge h'})}{\prod_{k\in \wh} \pi^{\star,h}_{i, k}(a_{k} \vert x_k )} \\
    &+  H^2 \eta \max_{\muphi \in \Phi_i^K} \sum_{x_{h}\in \cX_{i, h}} \sum_{b_{1:h-1}} \indicDeqK\frac{\prod_{k=1}^{h-1} \muphi_k(a_{k} \vert x_k, b_{1: k \wedge h'})}{\prod_{k\in \wh } \pi^{\star,h}_{i, k}(a_{k} \vert x_k )}\\
    &\times \sum_{t=1}^T \underbrace{\prod_{k \in {W}} \pi_{i,k}^t(b_{k} \vert x_k) \Big\langle \pi_{i, h}^t(\cdot \vert x_{i, h})  ,\indic{(x_h, \cdot) = (x_h^{t,(h, h', W)}, a_{h}^{t, (h, h', W)}) }\Big\rangle}_{\defeq\wb{\Delta}_{t}^{ (x_{h},b_{1:h'}) }} .
\end{align*} 
} Letting 
{\proofsize \begin{gather*}
    {\rm{I}}_h \defeq \frac{2 \log (8X_iA_i/p)}{\eta} \max_{\muphi \in \Phi_i^K} \sum_{x_{h}\in \cX_{i, h}} \sum_{b_{1:h-1}} \indicDeqK  \frac{\prod_{k=1}^{h-1} \muphi_k(a_{k} \vert x_k, b_{1: k \wedge h'})}{\prod_{k\in \wh} \pi^{\star,h}_{i, k}(a_{k} \vert x_k )};\\
    {\rm{II}}_h  \defeq H^2 \eta  \max_{\muphi \in \Phi_i^K} \sum_{x_{h}\in \cX_{i, h}} \sum_{b_{1:h-1}} \indicDeqK\frac{\prod_{k=1}^{h-1} \muphi_k(a_{k} \vert x_k, b_{1: k \wedge h'})}{\prod_{k\in \wh } \pi^{\star,h}_{i, k}(a_{k} \vert x_k )} \sum_{t=1}^T \wb{\Delta}_{t}^{ (x_{h},b_{1:h'}) }.
\end{gather*} 
}Using Lemma \ref{lemma:balancing}, we have
{\proofsize  \begin{align*}
    &\sum_{x_{h}\in \cX_{i, h}} \sum_{b_{1:h-1}} \indicDeqK  \frac{\prod_{k=1}^{h-1} \muphi_k(a_{k} \vert x_k, b_{1: k \wedge h'})}{ \prod_{k\in \wh} \pi^{\star,h}_{i, k}(a_{k} \vert x_k )}\\
    =& \sum_{x_{h}\in \cX_{i, h}} \sum_{b_{1:h-1}}\indicDeqK  \frac{\prod_{k=1}^{h-1} \muphi_k(a_{k} \vert x_k, b_{1: k \wedge h'})\prod_{k\in [h-1]\setminus W} \pi^{\star,h}_{i, k}(a_{k} \vert x_k )}{ \prod_{k\in [h]} \pi^{\star,h}_{i, k}(a_{k} \vert x_k )}\\
    \overset{(i)}{=}&\sum_{x_{h}\in \cX_{i, h}} \sum_{a_h}  \frac{\sum_{b_{1:h-1}} \indicDeqK \prod_{k=1}^{h-1} \muphi_k(a_{k} \vert x_k, b_{1: k \wedge h'})\prod_{k\in [h-1]\setminus W} \pi^{\star,h}_{i, k}(b_{k} \vert x_k ) \cdot \pi_{i,h}^{\star,h}(a_h \vert x_h)}{ \prod_{k\in [h]} \pi^{\star,h}_{i, k}(a_{k} \vert x_k )} \\
    \overset{(ii)}{=}&  \sum_{W \subset [h-1], |W|=  K   }A_{i}^{|W|}  \\
    & \sum_{x_h, a_h}  \frac{\sum_{b_{1:h-1}: W = \set{k\in[h-1]: a_k \not = b_k}}  \prod_{k=1}^{h-1} \muphi_k(a_{k} \vert x_k, b_{1: k \wedge h'})\prod_{k\in [h-1]\setminus W} \pi^{\star,h}_{i, k}(b_{k} \vert x_k )\prod_{k\in W} \pi^{\rm unif}_{i, k}(b_{k} \vert x_k ) \pi^{\star,h}_{i, h}(a_{h} \vert x_h )}{ \prod_{k\in [h]} \pi^{\star,h}_{i, k}( a_{k} \vert x_k )} \\
    \overset{(iii)}{\le }& \sum_{W \subset [h-1], |W|=  K   } A_{i}^{|W|} X_{i,h} A_i~~ \\
    \overset{}{=}& X_{i,h} {h -1 \choose K  }A_i^{K\wedge H + 1}.
\end{align*} 
}Here, (i) uses that $\pi_{i,h}^{\star,h}(\cdot|x_h)$ is uniform distribution on $\cA_i$ and that for $k\in [h-1] \setminus W $, we have $a_k=b_k$; (ii) follows from grouping $x_h$ and $b_{1:h}$ by $|W|$, where $\pi_{i, k}^{\rm unif}$ is the uniform distribution on $\cA_i$; (iii) uses Lemma \ref{lemma:balancing} and that the 
numerator is no more than the sequence-form probability of some policy. This policy can be understand as: 
\begin{itemize}
    \item Sample recommended action $b_k$ from $\pi_{i,k}^{\star, h}(\cdot | x_k)$ if step $k \in W$. Otherwise, sample recommended action $b_k$ from $\pi_{i,k}^{\rm unif}(\cdot | x_k)$.
    
    \item ``True'' actions are sampled from $\phi_k(a_k |x_k, b_{1:k})$ for $k \in [h-1]$. At step $h$, ``True'' action is sampled from $a_h$.
\end{itemize}
Consequently,
{\proofsize \begin{equation}\label{equation:ext-mu-over-pi-bound}
    \sum_{x_{h}\in \cX_{i, h}} \sum_{b_{1:h-1}} \indicDeqK  \frac{\prod_{k=1}^{h-1} \muphi_k(a_{k} \vert x_k, b_{1: k \wedge h'})}{ \prod_{k\in \wh} \pi^{\star,h}_{i, k}(a_{k} \vert x_k )} \le X_{i,h} {h -1 \choose K }A_i^{K\wedge H + 1}.
\end{equation}
}So we have$${\rm I}_h \le \frac{2 \log (8X_iA_i/p)}{\eta}X_{i,h} {h-1 \choose  K }A_i^{K\wedge H+1}.$$
To give an upper bound of ${\rm I}_h$, obvserve that the random variables $\wb{\Delta}_{t}^{ (x_{h},b_{1:h'}) }$ satisfy the following: 
\begin{itemize}
    \item $\wb{\Delta}_{t}^{ (x_{h},b_{1:h'}) }=\prod_{k \in W} \pi_{i,k}^t(b_{k} \vert x_k) \Big\langle \pi_{i, h}^t(\cdot \vert x_{i, h})  ,\indic{(x_h, \cdot) = (x_h^{t,(h, h', W)}, a_{h}^{t, (h, h', W)}) }\Big\rangle \in [0,1].$ 

    \item Let $\cF_{t-1}$ be the $\sigma$-algebra containing all information after $\pi^t$ is sampled, then
    {\proofsize  \begin{align*}
        &~\E\brac{\wb{\Delta}_{t}^{ (x_{h},b_{1:h'}) }\vert \cF_{t-1}} \\
        = & \prod_{k \in W} \pi_{i,k}^t(b_{k} \vert x_k) \E\brac{ \sum_a \pi_{i,h}^t(a|x_{i,h}) \indic{(x_h, a) = (x_h^{t,(h, h', W)}, a_{h}^{t, (h, h', W)})}\vert \cF_{t-1}}\\
        = & \prod_{k \in W} \pi_{i,k}^t(b_{k} \vert x_k) \sum_a \pi_{i,h}^t(a|x_{i,h}) \P^{((\pi^{\star, h}_{i,k})_{k\in W(h)}(\pi^{t}_{i,k})_{k\in[h']\setminus W} ) \times \pi_{-i}^t}(x_h^{t,(h, h', W)} = x_h, a_h^{t,(h, h', W)} = a )\\
        = & \prod_{k \in W} \pi_{i,k}^t(b_{k} \vert x_k) \sum_a \pi_{i,h}^t(a|x_{i,h}) \paren{\prod_{k \in W(h) } \pi_{i,k}^{\star,h}(a_{k} \vert x_k) \cdot\prod_{k \in [h']\setminus W} \pi_{i,k}^{t}(a_{k} \vert x_k)\cdot \pi_{i,h}^{\star,h}(a \vert x_h)  p^{\pi^t_{-i}}_{1:h}(x_{h}) }\\
        = &  \prod_{k \in [h']} \pi_{i,k}^t(b_{k} \vert x_k) \prod_{k \in \wh } \pi_{i,k}^{\star,h}(a_{k} \vert x_k)p^{\pi^t_{-i}}_{1:h}(x_{h}),
    \end{align*} }where the last equation is because $\prod_{k \in W} \pi_{i,k}^t(b_{k} \vert x_k) \cdot \prod_{k \in [h']\setminus W} \pi_{i,k}^{t}(a_{k} \vert x_k) = \prod_{k \in [h']} \pi_{i,k}^t(b_{k} \vert x_k) ;$
    \item The conditional variance $\E[(\wb\Delta_t^{ (x_{h},b_{1:h'}) })^2|\cF_{t-1}]$ can be bounded as
    {\proofsize  \begin{align*}
        & ~\quad \E\brac{ \paren{\wb\Delta_t^{ (x_{h},b_{1:h'}) }}^2 \Big| \cF_{t-1}} \le \E\brac{ \wb\Delta_t^{ (x_{h},b_{1:h'}) }  \Big| \cF_{t-1}}\\
    &\overset{}{=}   \prod_{k \in [h']} \pi_{i,k}^t(b_{k} \vert x_k) \prod_{k \in \wh } \pi_{i,k}^{\star,h}(a_{k} \vert x_k)p^{\pi^t_{-i}}_{1:h}(x_{h}).
    \end{align*} }
    Here, the inequality comes from that $\wb\Delta_t^{ (x_{h},b_{1:h'}) }\in [0,1]$.
\end{itemize}
Therefore, we can apply Freedman's inequality and union bound to get that  for any fixed $\lambda \in (0, 1]$, with probability at least $1-p/8$, the following holds simultaneously for all $(h,x_{i, h},  b_{1:h-1})$: 
{\proofsize  \begin{align*}
    \sum_{t = 1}^T \wb\Delta_t^{ (x_{h},b_{1:h'}) } \le& \paren{ \lambda+1 } \prod_{k \in [h']} \pi_{i,k}^t(b_{k} \vert x_k) \prod_{k \in \wh } \pi_{i,k}^{\star,h}(a_{k} \vert x_k)p^{\pi^t_{-i}}_{1:h}(x_{h})
    + \frac{C\log(X_i A_i/p)}{\lambda},
\end{align*}
}where $C> 0$ is some absolute constant. Plugging this bound into ${\rm II}_h$ yields that,
{\proofsize  \begin{align*}
    {\rm II}_h \le & H^2 \eta \max_{\muphi \in \Phi_i^K} \sum_{x_{h}\in \cX_{i, h}} \sum_{b_{1:h-1}} \indicDeqK\frac{\prod_{k=1}^{h-1} \muphi_k(a_{k} \vert x_k, b_{1: k \wedge h'})}{\prod_{k\in \wh} \pi^{\star,h}_{i, k}(a_{k} \vert x_k )}\\
    & \times \brac{\paren{ \lambda + 1 } \prod_{k \in [h']} \pi_{i,k}^t(b_{k} \vert x_k) \prod_{k \in \wh } \pi_{i,k}^{\star,h}(a_{k} \vert x_k)p^{\pi^t_{-i}}_{1:h}(x_{h})
    + \frac{C\log(X_i A_i/p)}{\lambda}}.
\end{align*} }
Note that for fixed $t$, by Corollary \ref{lemma:sum=1},
{\proofsize
$$\sum_{x_{h}\in \cX_{i, h}}\sum_{b_{1:h-1}}\indicDeqK \prod_{k=1}^{h-1} \muphi_k(a_{k} \vert x_k, b_{1:k\wedge h'}) \prod_{k \in [h']} \pi_{i,k}^t(b_{k} \vert x_k) p^{\pi^t_{-i}}_{1:h}(x_{i,h}) \le 1.$$
}So we have
{\proofsize \begin{equation*}
    \begin{gathered}
        \max_{\muphi \in \Phi_i^K} \sum_{x_{h}\in \cX_{i, h}} \sum_{b_{1:h-1}}\indicDeqK \prod_{k=1}^{h-1} \muphi_k(a_{k} \vert x_k, b_{1: k \wedge h'}) \sum_{t = 1}^T \prod_{k \in [h']} \pi_{i,k}^t(b_{k} \vert x_k) p^{\pi^t_{-i}}_{1:h}(x_{i,h}) \le T.
    \end{gathered}
\end{equation*}
}Moreover, by the previous bound \eqref{equation:ext-mu-over-pi-bound},
{\proofsize  \begin{align*}
    \sum_{x_{h}\in \cX_{i, h}} \sum_{b_{1:h-1}} \indicDeqK  \frac{\prod_{k=1}^{h-1} \muphi_k(a_{k} \vert x_k, b_{1: k \wedge h'})}{ \prod_{k\in W(h)} \pi^{\star,h}_{i, k}(a_{k} \vert x_k )} \le X_{i,h} {h -1 \choose K  }A_i^{K\wedge H  + 1}.
\end{align*} 
}Using these two inequalities, we can get that 
{\proofsize  \begin{align*}
    &{\rm II}_h \le  H^2  \eta \paren{ \lambda+1 }T+ \frac{CH^2  \eta\log(X_i A_i/p)}{\lambda}X_{i,h} {h -1 \choose K  }A_i^{K\wedge H + 1}.
\end{align*} }
Taking summation over $h\in[H]$, we have
{\proofsize  \begin{align*}
    &\sum_{h=1}^H   \wt{\rm REGRET}_{h}^{T, \ext} = \sum_{h=1}^H({\rm I}_{h}+{\rm II}_{h})\\
    \le&~ H^3 \eta \paren{ \lambda+1 }T + \paren{ \frac{ 2\log(8X_i A_i/p)}{\eta} + \frac{CH^2 \eta\log(X_i A_i/p)}{\lambda}}\sum_{h=1}^H X_{i,h} {h -1 \choose K }A_i^{K\wedge H + 1}\\
    \le &~  H^3 \eta \paren{ \lambda+1 } T + \paren{ \frac{ 2\log(8X_i A_i/p)}{\eta} + \frac{CH^2 \eta\log(X_i A_i/p)}{\lambda}} X_{i } {H -1 \choose K\wedge H  }A_i^{K\wedge H + 1} \\
    \le&~ H^3 \eta \paren{ \lambda+1 } T + \paren{ \frac{ 2\log(8X_i A_i/p)}{\eta} + \frac{CH^2 \eta\log(X_i A_i/p)}{\lambda}} X_{i } {H  \choose K\wedge H  }A_i^{K\wedge H + 1} ,
\end{align*} }
for all $\lambda\in (0,1].$ Choosing $\lambda = 1$, we have,
{\proofsize  \begin{align*}
    \sum_{h=1}^H   \wt{\rm REGRET}_{h}^{T, \ext} \le
    2H^3\eta T + \paren{ \frac{2 \log (8X_iA_i/p)}{\eta} + CH^2 \eta\log(X_i A_i/p) } X_i {H-1 \choose  K\wedge H }A_i^{K\wedge H +1 }.
\end{align*} }
Then, choosing $\eta = \sqrt{{H \choose K\wedge H }A_i^{K \wedge H  +1} X_i \iota / (H^3T) }$, we have
{\proofsize  \begin{align*}
     \sum_{h=1}^H   \wt{\rm REGRET}_{h}^{T, \ext} &\le \sqrt{H^3 {H \choose K\wedge H }A_i^{K\wedge H+ 1 
  } X_i T \iota }\\
    + &\mc{O}  \paren{ {H \choose K\wedge H }  A_i^{K\wedge H + 1 } X_i   \iota \sqrt{\frac{H {H \choose K\wedge H }  A_i^{K\wedge H  +1 } X_i \iota }{T}}},
\end{align*} }
where $\iota = \log (8X_i A_i /p)$ is a log factor.
\end{proof}
 
\subsubsection{Proof of Lemma \ref{lemma:bound-bias1^II}: Bound on ${\rm BIAS}_{1, h}^{T, \ext}$}
\label{appendix:proof-bound-bias1^II}
\begin{proof}
For fixed $x_h$ and $b_{1:h-1}$, let $W = \set{k\in [h-1]: b_k\not= a_k}$ and $h'$ is the maximal element in $W$, so we have $h' = \tau_K$. We define $\wh$ as the set $W \cup \set{h'+1,\dots,h}$. We can rewrite ${\rm BIAS}_{1, h}^{T, \ext}$ as 
{\proofsize  \begin{align*}
    &{\rm BIAS}_{1, h}^{T, \ext}\\
    =~& \max_{\phi \in \Phi_i^K}  \sum_{x_{h}\in \cX_{i, h}} \sum_{b_{1:h-1}} \indicDeqK \frac{\prod_{k=1}^{h-1} \muphi_k(a_{k} \vert x_k, b_{ 1: k\wedge h'})}{\prod_{k \in \wh } \pi^{\star,h}_{i, k}(a_{k} \vert x_k )} \\
    & \times \sum_{t = 1}^T \underbrace{ \prod_{k \in \wh } \pi^{\star,h}_{i, k}(a_{k} \vert x_k ) \prod_{k = 1}^{h'} \pi_{i,k}^t (b_{k}|x_{k})
    \Big\langle\pi_{i, h}^t(\cdot \vert x_{i, h}),  \L_{i,h}^{t}(x_{i, h},\cdot) -\wt L^t_{(x_{h}, b_{1: h'})} \Big\rangle}_{\defeq \wt\Delta_t^{(x_{h},b_{1:h'}) }}.
\end{align*} 
}Observe that the random variable $\wt\Delta_t^{(x_{h},b_{1:h'}) }$ satisfy the following: 
\begin{itemize}
    \item By the definition of $\wt L$, we can rewrite $\wt \Delta_t^{(x_{h},b_{1:h'}) }$ as
    {\proofsize  \begin{align*}
        &\wt \Delta_t^{(x_{h},b_{1:h'}) } = \prod_{k = 1}^{h'} \pi_{i,k}^t (b_{k}|x_{k}) \Big\langle \pi_{i,h}^t (\cdot |x_{h}), \ L_{(x_{h},b_{1:h'}) }^{t}  \Big\rangle \prod_{k \in \wh } \pi^{\star,h}_{i, k}(a_{k} \vert x_k )\\
        &- \prod_{k \in W} \pi_{i,k}^t(b_{k} \vert x_k) \Big\langle \pi_{i,h}^t (\cdot |x_{h}),  \indic{(x_h, \cdot) = (x_h^{t,(h ,h', W)}, a_{h}^{t, (h ,h', W)}) } \cdot \paren{H-h+1 - \sum_{h'' = h}^H r_{i,h''}^{t, (h ,h', W)}} \Big\rangle;
    \end{align*} }
    \item $\wt\Delta_t^{(x_{h},b_{1:h'}) }\le \Big\langle \pi_i^t (\cdot |x_{h}), \ L_{(x_{h},b_{1:h'}) }^{t}  \Big\rangle \le H$ ;
    \item $\E[\wt \Delta_t^{(x_{h},b_{1:h'}) }|\cF_{t-1}]=0$, where $\cF_{t-1}$ is the $\sigma$-algebra containing all information after $\pi^t$ is sampled. This also can be seen from the unbiasedness of $\wt{\piL}$;
    \item The conditional variance $\E[(\wt\Delta_t^{(x_{h},b_{1:h'}) })^2|\cF_{t-1}]$ can be bounded as
  {\proofsize  \begin{align*}
    & ~\quad \E\brac{ \paren{\wt\Delta_t^{(x_{h},b_{1:h'}) }}^2 \Big| \cF_{t-1}} \\
    &\le \E\brac{ \paren{ H-h+1 - \sum_{h''=h}^H r_{i,h''}^{t, (h ,h', W)} }^2 \paren{\Big\langle \pi_{i,h}^t(\cdot | x_h), \prod_{k \in W} \pi_{i,k}^t(b_{k} \vert x_k) \indic{(x_h, \cdot) = (x_h^{t,(h ,h', W)}, a_{h}^{t, (h ,h', W)}) } \Big\rangle}^2\Big| \cF_{t-1} }\\
    & \overset{(i)}{\le} H^2 \prod_{k \in W} \pi_{i,k}^t(b_{k} \vert x_k) \E\brac{  \Big\langle \pi_{i,h}^t(\cdot | x_h),  \indic{(x_h, \cdot) = (x_h^{t,(h ,h', W)}, a_{h}^{t, (h ,h', W)}) } \Big\rangle \Big| \cF_{t-1} }\\
    & = H^2 \prod_{k \in W} \pi_{i,k}^t(b_{k} \vert x_k) \E\brac{ \sum_{a\in\cA_i} \pi_{i,h}^t(a | x_h),  \indic{(x_h, a) = (x_h^{t,(h ,h', W)}, a_{h}^{t, (h ,h', W)}) } \Big\rangle \Big| \cF_{t-1} }\\
    & = H^2 \prod_{k \in W} \pi_{i,k}^t(b_{k} \vert x_k) \cdot \sum_{ a \in \cA_i}   \pi_{i,h}^t(a | x_h) \P^{((\pi^{\star, h}_{i,k})_{k\in \wh}(\pi^{t}_{i,k})_{k\in[h']\setminus W} ) \times \pi_{-i}^t}\paren{ (x_{h}^{t, (h ,h', W)}, a_{h}^{t,(h ,h', W)})=(x_{h}, a) } \\
    & = H^2 \prod_{k \in W} \pi_{i,k}^t(b_{k} \vert x_k) \cdot \sum_{ a \in \cA_i}   \pi_{i,h}^t(a | x_h)   \prod_{k \in W(h) } \pi_{i,k}^{\star,h}(a_{k} \vert x_k) \cdot\prod_{k \in [h']\setminus W} \pi_{i,k}^{t}(a_{k} \vert x_k)\cdot p^{\pi^t_{-i}}_{1:h}(x_{i,h}) \\
     &\overset{(ii)}{\le} H^2 \prod_{k \in [h']} \pi_{i,k}^t(b_{k} \vert x_k) \cdot \prod_{k \in \wh } \pi_{i,k}^{\star,h}(a_{k} \vert x_k) \cdot  p^{\pi^t_{-i}}_{1:h}(x_{i,h}) .
  \end{align*} }
  Here, (i) uses $\Big\langle \pi_{i,h}^t(\cdot | x_h),  \indic{(x_h, \cdot) = (x_h^{t,(h ,h', W)}, a_{h}^{t, (h ,h', W)}) } \Big\rangle\in [0,1]$; (ii) is because $$\prod_{k \in W} \pi_{i,k}^t(b_{k} \vert x_k) \prod_{k \in [h']\setminus W} \pi_{i,k}^{t}(a_{k} \vert x_k) = \prod_{k \in [h']} \pi_{i,k}^t(b_{k} \vert x_k).  $$
\end{itemize}
Therefore, we can apply Freedman's inequality and union bound to get that  for any fixed $\lambda \in (0, 1/H]$, with probability at least $1-p/8$, the following holds simultaneously for all $(h,x_{i, h},  b_{1:h-1})$: 

{\proofsize  \begin{align*}
    \sum_{t = 1}^T \wt\Delta_t^{(x_{h},b_{1:h'}) } \le& \lambda H^2 \prod_{k \in \wh } \pi_{i,k}^{\star,h}(a_{k} \vert x_k) \sum_{t=1}^T \prod_{k \in [h']} \pi_{i,k}^t(b_{k} \vert x_k) p^{\pi^t_{-i}}_{1:h}(x_{i,h})
    + \frac{C\log(X_i A_i/p)}{\lambda},
\end{align*} 
}where $C > 0$ is some absolute constant. Plugging this bound into ${\rm BIAS}_{1, h}^{T, \ext}$ yields that, for all $h\in[H]$,
{\proofsize \begin{equation*}
    \begin{aligned}
    &~{\rm BIAS}_{1, h}^{T, \ext} \le \max_{\muphi \in \Phi_i^K} \sum_{x_{h}\in \cX_{i, h}} \sum_{b_{1:h-1}} \indicDeqK  \frac{\prod_{k=1}^{h-1} \muphi_k(a_{k} \vert x_k, b_{ 1: k\wedge h'})}{ \prod_{k\in \wh} \pi^{\star,h}_{i, k}(a_{k} \vert x_k )}\\
     & \times \brac{
    \lambda H^2 \prod_{k \in \wh } \pi_{i,k}^{\star,h}(a_{k} \vert x_k) \sum_{t=1}^T \prod_{k \in [h']} \pi_{i,k}^t(b_{k} \vert x_k) p^{\pi^t_{-i}}_{1:h}(x_{i,h})
    + \frac{C \log(X_i A_i/p)}{\lambda}}\\
    \le~ & \lambda H^2  \max_{\muphi \in \Phi_i^K} \sum_{x_{h}\in \cX_{i, h}} \sum_{b_{1:h-1}}\indicDeqK \prod_{k=1}^{h-1} \muphi_k(a_{k} \vert x_k, b_{ 1: k\wedge h'}) \sum_{t = 1}^T \prod_{k \in [h']} \pi_{i,k}^t(b_{k} \vert x_k) p^{\pi^t_{-i}}_{1:h}(x_{i,h}) \\
    &+ \frac{C \log(X_i A_i/p)}{\lambda} \max_{\muphi \in \Phi_i^K} \sum_{x_{h}\in \cX_{i, h}} \sum_{b_{1:h-1}} \indicDeqK  \frac{\prod_{k=1}^{h-1} \muphi_k(a_{k} \vert x_k, b_{ 1: k\wedge h'})}{ \prod_{k\in \wh} \pi^{\star,h}_{i, k}(a_{k} \vert x_k )}.
\end{aligned}
\end{equation*}
}Note that for fixed $t$, by Corollary \ref{lemma:sum=1},
{\proofsize
$$\sum_{x_{h}\in \cX_{i, h}}\sum_{b_{1:h-1}}\indicDeqK \prod_{k=1}^{h-1} \muphi_k(a_{k} \vert x_k, b_{1:k\wedge h'}) \prod_{k \in [h']} \pi_{i,k}^t(b_{k} \vert x_k) p^{\pi^t_{-i}}_{1:h}(x_{i,h}) \le 1.$$
}So we have
{\proofsize \begin{equation*}
    \begin{gathered}
        \lambda H^2  \max_{\muphi \in \Phi_i^K} \sum_{x_{h}\in \cX_{i, h}} \sum_{b_{1:h-1}}\indicDeqK \prod_{k=1}^{h-1} \muphi_k(a_{k} \vert x_k, b_{ 1: k\wedge h'}) \sum_{t = 1}^T \prod_{k \in [h']} \pi_{i,k}^t(b_{k} \vert x_k) p^{\pi^t_{-i}}_{1:h}(x_{i,h})     \le \lambda H^2  T.
    \end{gathered}
\end{equation*} 
}Moreover, by the inequality (\ref{equation:ext-mu-over-pi-bound}) in the proof of Lemma \ref{lemma:bound-regret^II}, we have
{\proofsize  \begin{align*}
    \sum_{x_{h}\in \cX_{i, h}} \sum_{b_{1:h-1}} \indicDeqK  \frac{\prod_{k=1}^{h-1} \muphi_k(a_{k} \vert x_k, b_{ 1: k\wedge h'})}{ \prod_{k\in \wh} \pi^{\star,h}_{i, k}(a_{k} \vert x_k )} \le X_{i,h} {h -1 \choose K  }A_i^{K\wedge H + 1}.
\end{align*} }
Plugging these bounds into ${\rm BIAS}_{1, h}^{T, \ext}$ yields that,  with probability at least $1-p/8$, for all $h\in [H]$,
{\proofsize  \begin{align*}
    {\rm BIAS}_{1, h}^{T, \ext} \le \lambda H^2 T + \frac{C \log(X_iA_i/p)}{\lambda} X_{i,h} {h -1 \choose K  }A_i^{K\wedge H + 1}.
\end{align*} }
Taking summation over $h\in[H]$, we have
{\proofsize  \begin{align*}
    \sum_{h=1}^H   {\rm BIAS}_{1, h}^{T, \ext} \le& \lambda H^3 T + \frac{C \log(X_iA_i/p)}{\lambda}\sum_{h=1}^H X_{i,h} {h -1 \choose K  }A_i^{K\wedge H + 1}\\
    \le & \lambda H^3   T + \frac{C \log(X_iA_i/p)}{\lambda} X_{i } {H  \choose K\wedge H  }A_i^{K\wedge H + 1},
\end{align*} 
}for all $\lambda\in (0,1/H].$ 
Choose $\lambda = \min\set{\frac{1}{H}, \sqrt{\frac{C X_{i,h} {H  \choose K\wedge H  }A_i^{K\wedge H + 1} \log(X_iA_i/p)}{H^3 T}}}$, we obtain
{\proofsize  \begin{align*}
       \sum_{t=1}^T {\rm BIAS}_{1, h}^{T, \ext} \le \mc{O}\paren{\sqrt{H^3 {H \choose K\wedge H } A_i^{K\wedge H+1} X_iT \iota}+ H {H \choose K\wedge H }A_i^{K\wedge H+1 } X_i\iota},
\end{align*} }
where $\iota = \log(8X_iA_i/p)$ is a log factor.
\end{proof}

\subsubsection{Proof of Lemma \ref{lemma:bound-bias2^II}: Bound on ${\rm BIAS}_{2, h}^{T, \ext}$}
\label{appendix:proof-bound-bias2^II}
\begin{proof}
For fixed $x_h$ and $b_{1:h-1}$, let $W = \set{k\in [h-1]: b_k\not= a_k}$ and $h'$ is the maximal element in $W$, so we have $h' = \tau_K$. We define $\wh$ as the set $W \cup \set{h'+1,\dots,h}$.
We can rewrite ${\rm BIAS}_{2, h}^{T, \ext}$ as 
{\proofsize  \begin{align*}
    {\rm BIAS}_{2, h}^{T, \ext}
    =~&\max_{\muphi\in \Phi_i^K} \sum_{x_{h}\in \cX_{i, h}} \sum_{b_{1:h-1}} \indicDeqK \prod_{k=1}^{h-1} \muphi_k(a_{k} \vert x_k, b_{ 1: k\wedge h'}) \\
    & \times \max_a \sum_{t = 1}^T \prod_{k = 1}^{h'} \pi_i^t (b_{k}|x_{k})
    \paren{ \wt L^t_{(x_{h}, b_{1: \tau_K})}(a)- \L_{i,h}^{t}(x_{i, h},a) } \\
    \le & \max_{\muphi\in \Phi_i^K} \sum_{x_{h}\in \cX_{i, h}} \sum_{b_{1:h-1}} \indicDeqK   \prod_{k=1}^{h-1} \muphi_k(a_{k} \vert x_k, b_{ 1: k\wedge h'})  \\
    &\times\max_{\phi_h'} \sum_{a_h} \phi_h'(a_h|x_h, b_{1:k \wedge h'}) \sum_{t = 1}^T \prod_{k = 1}^{h'} \pi_i^t (b_{k}|x_{k}) \paren{ \wt L_{(x_{h},b_{1:h'}) }^t(a_h) -  L_{(x_{h},b_{1:h'}) }^{t}(a_h) } \\
    =~& \max_{\muphi\in \Phi_i^K}\sum_{x_{i, h}, a_h} \sum_{b_{1:h-1}}  \indicDeqK  \prod_{k=1}^{h} \muphi_k(a_{k} \vert x_k, b_{ 1: k\wedge h'}) \\
    &\times \sum_{t = 1}^T    \prod_{k = 1}^{h'} \pi_i^t (b_{k}|x_{k}) \paren{ \wt L_{(x_{h},b_{1:h'}) }^t(a_h) -  L_{(x_{h},b_{1:h'}) }^{t}(a_h) } \\
    =~& \max_{\muphi\in \Phi_i^K}\sum_{x_{i, h}, a_h} \sum_{b_{1:h-1}}  \indicDeqK  \frac{\prod_{k=1}^{h} \muphi_k(a_{k} \vert x_k, b_{ 1: k\wedge h'})}{\prod_{ k \in \wh } \pi^{\star,h}_{i, k}(a_{k} \vert x_k )} \\
    &~ \times \sum_{t = 1}^T   \underbrace{  \prod_{k = 1}^{h'} \pi_i^t (b_{k}|x_{k}) \paren{ \wt L_{(x_{h},b_{1:h'}) }^t(a_h) -  L_{(x_{h},b_{1:h'}) }^{t}(a_h) } \prod_{k \in \wh } \pi^{\star,h}_{i, k}(a_{k} \vert x_k )}_{\defeq \wt\Delta_t^{x_{h}, b_{1:h'}, a_h}}
\end{align*} 
}Here, (i) comes from the fact that  the inner max over $\phi_h'$ and the outer max over $\phi_{1:h-1}$ are separable and thus can be merged into a single max over $\phi_{1:h}$.
 
Observe that the random variable $\wt\Delta_t^{x_{h}, b_{1:h'}, a_h}$ satisfy the following:
\begin{itemize}[topsep=0pt, itemsep=0pt]
    \item By the definition of $\wt \piL$, we can rewrite $\wt \Delta_t^{x_{h}, b_{1:h'}, a_h}$ as
    {\proofsize  \begin{align*}
        &\wt \Delta_t^{x_{h}, b_{1:h'}, a_h} \\
        =&   \prod_{k \in W} \pi_{i,k}^t(b_{k} \vert x_k) \indic{(x_h, a_h) = (x_h^{t,(h ,h', W)}, a_{h}^{t, (h ,h', W)}) } \cdot \paren{H-h+1 - \sum_{h'' = h}^H r_{i,h''}^{t, (h ,h', W)}} \\
        &~~ -   \prod_{k \in [h']} \pi_{i,k}^t(b_{k} \vert x_k) L_{(x_{h},b_{1:h'}) }^{t}(a_h)   \prod_{k \in \wh} \pi^{\star,h}_{i, k}(a_{k} \vert x_k ).
    \end{align*} }
    \item $\wt\Delta_t^{x_{h}, b_{1:h'}, a_h}\le H$.
    
    \item $\E[\wt \Delta_t^{x_{h}, b_{1:h'}, a_h }|\cF_{t-1}]=0$, where $\cF_{t-1}$ is the $\sigma$-algebra containing all information after $\pi^t$ is sampled.
    
    \item The conditional variance $\E[(\wt\Delta_t^{x_{h}, b_{1:h'}, a_h })^2|\cF_{t-1}]$ can be bounded as
  {\proofsize  \begin{align*}
    & ~\quad \E\brac{ \paren{\wt\Delta_t^{x_{h}, b_{1:h'}, a_h}}^2 \Big| \cF_{t-1}} \\
    &\le \E\brac{ \paren{ H-h+1 - \sum_{h''=h}^H r_{i,h''}^{t, (h ,h', W)} }^2 \paren{ \prod_{k \in W} \pi_{i,k}^t(b_{k} \vert x_k) \indic{(x_h, a_h) = (x_h^{t,(h ,h', W)}, a_{h}^{t, (h ,h', W)}) } }^2\Big| \cF_{t-1} }\\
    & \le H^2 \prod_{k \in W} \pi_{i,k}^t(b_{k} \vert x_k)   \P^{((\pi^{\star, h}_{i,k})_{k\in \wh}(\pi^{t}_{i,k})_{k\in[h']\setminus W} ) \times \pi_{-i}^t}\paren{ (x_{h}^{t, (h ,h', W)}, a_{h}^{t,(h ,h', W)})=(x_{h}, a_h) } \\
    & = H^2 \prod_{k \in W} \pi_{i,k}^t(b_{k} \vert x_k)  \prod_{k \in W  } \pi_{i,k}^{\star,h}(a_{k} \vert x_k) \cdot\prod_{k \in [h']\setminus W} \pi_{i,k}^{t}(a_{k} \vert x_k)\cdot \pi_{i,h}^{\star, h}(a_h | x_h)  p^{\pi^t_{-i}}_{1:h}(x_{i,h}) \\
     &\overset{(i)}{=} H^2 \prod_{k \in [h']} \pi_{i,k}^t(b_{k} \vert x_k)  p^{\pi^t_{-i}}_{1:h}(x_{i,h})\prod_{k \in \wh} \pi_{i,k}^{\star,h}(a_{k} \vert x_k) .
  \end{align*} }
  Here, (i) is because $\prod_{k \in [h']\setminus W} \pi_{i,k}^{t}(a_{k} \vert x_k)  = \prod_{k \in [h']\setminus W} \pi_{i,k}^{t}(b_{k} \vert x_k) $ .
\end{itemize}
Therefore, we can apply Freedman's inequality and union bound to get that  for any fixed $\lambda \in (0, 1/H]$, with probability at least $1-p/8$, the following holds simultaneously for all $(h,x_{i, h},  b_{1:h}, a_h)$: 

{\proofsize  \begin{align*}
    \sum_{t = 1}^T \wt\Delta_t^{x_h, b_{1:h'}, a_h} \le& \lambda H^2 \prod_{k \in \wh } \pi_{i,k}^{\star,h}(a_{k} \vert x_k) \sum_{t=1}^T \prod_{k \in [h']} \pi_{i,k}^t(b_{k} \vert x_k)  p^{\pi^t_{-i}}_{1:h}(x_{i,h})
    + \frac{C \log(X_i A_i/p)}{\lambda},
\end{align*} 
}where $C> 0$ is some absolute constant. Plugging this bound into ${\rm BIAS}_{2, h}^{T, \ext}$ yields that, for all $h\in[H]$ and $\muphi \in \Phi_i^K$,

{\proofsize \begin{equation*}
    \begin{aligned}
    &~{\rm BIAS}_{2, h}^{T, \ext} \le  \max_{\muphi\in\Phi_i^K}\sum_{x_{h}, a_h} \sum_{b_{1:h-1}} \indicDeqK  \frac{\prod_{k=1}^{h} \muphi_k(a_{k} \vert x_k, b_{ 1: k\wedge h'})}{ \prod_{k\in \wh} \pi^{\star,h}_{i, k}(a_{k} \vert x_k )}\cdot\\
     &\brac{
    \lambda H^2  \prod_{k \in \wh } \pi_{i,k}^{\star,h}(a_{k} \vert x_k) \sum_{t=1}^T \prod_{k \in [h']} \pi_{i,k}^t(b_{k} \vert x_k)  p^{\pi^t_{-i}}_{1:h}(x_{i,h})
    + \frac{C \log(X_i A_i/p)}{\lambda}}\\
    \le~ & \max_{\muphi\in\Phi_i^K} \lambda H^2 \sum_{x_{h}, a_h} \sum_{b_{1:h-1}}\indicDeqK \prod_{k=1}^{h} \muphi_k(a_{k} \vert x_k, b_{ 1: k\wedge h'}) \sum_{t = 1}^T \prod_{k \in [h']} \pi_{i,k}^t(b_{k} \vert x_k)  p^{\pi^t_{-i}}_{1:h}(x_{i,h}) \\
    &+\max_{\muphi\in\Phi_i^K} \frac{C \log(X_i A_i/p)}{\lambda}\sum_{x_{h},a_h} \sum_{b_{1:h-1}} \indicDeqK  \frac{\prod_{k=1}^{h} \muphi_k(a_{k} \vert x_k, b_{ 1: k\wedge h'})}{ \prod_{k\in \wh} \pi^{\star,h}_{i, k}(a_{k} \vert x_k )}.
\end{aligned}
\end{equation*} 
}Note that for fixed $t$, by Corollary \ref{lemma:sum=1},
{\proofsize
$$\sum_{x_{h}\in \cX_{i, h}}\sum_{b_{1:h-1}}\indicDeqK \prod_{k=1}^{h-1} \muphi_k(a_{k} \vert x_k, b_{1:k\wedge h'}) \prod_{k \in [h']} \pi_{i,k}^t(b_{k} \vert x_k)  p^{\pi^t_{-i}}_{1:h}(x_{i,h}) \le 1.$$
}So we have
$$
\sum_{x_{h}, a_h} \sum_{b_{1:h-1}}\indicDeqK \prod_{k=1}^{h} \muphi_k(a_{k} \vert x_k, b_{ 1: k\wedge h'}) \sum_{t = 1}^T \prod_{k \in [h']} \pi_{i,k}^t(b_{k} \vert x_k)  p^{\pi^t_{-i}}_{1:h}(x_{i,h}) \le T.
$$
Moreover, by the inequality (\ref{equation:ext-mu-over-pi-bound}) in the proof of Lemma \ref{lemma:bound-regret^II}, we have
{\proofsize  \begin{align*}
    &\sum_{x_{h},a_h} \sum_{b_{1:h-1}} \indicDeqK  \frac{\prod_{k=1}^{h} \muphi_k(a_{k} \vert x_k, b_{ 1: k\wedge h'})}{ \prod_{k\in \wh} \pi^{\star,h}_{i, k}(a_{k} \vert x_k )}\\
    &= \sum_{x_{h} } \sum_{b_{1:h-1}} \indicDeqK  \frac{\prod_{k=1}^{h-1} \muphi_k(a_{k} \vert x_k, b_{ 1: k\wedge h'})}{ \prod_{k\in \wh} \pi^{\star,h}_{i, k}(a_{k} \vert x_k )}\\
    &\le X_{i,h} {h -1 \choose K  }A_i^{K\wedge H}
\end{align*} 
}Plugging these bounds into ${\rm BIAS}_{2, h}^{T, \ext}$ yields that,  with probability at least $1-p/8$, for all $h\in [H]$,
{\proofsize  \begin{align*}
    {\rm BIAS}_{2, h}^{T, \ext} \le \lambda H^2 T + \frac{C \log(X_iA_i/p)}{\lambda} X_{i,h}  {h-1\choose K} A_i^{K\wedge H}.
\end{align*} 
}Taking summation over $h\in[H]$, we have
{\proofsize  \begin{align*}
    \sum_{h=1}^H   {\rm BIAS}_{2, h}^{T, \ext} \le& \lambda H^3  T + \frac{C \log(X_iA_i/p)}{\lambda}\sum_{h=1}^H X_{i,h}  {h-1\choose K} A_i^{K\wedge H}\\
    \le & \lambda H^3  T + \frac{C \log(X_iA_i/p)}{\lambda}X_i {H \choose K}A_i^{K\wedge H},
\end{align*} }
for all $\lambda\in (0,1/H].$ 
Choose $\lambda = \min\set{\frac{1}{H}, \sqrt{\frac{CX_i  {H \choose K\wedge H }A_i^{K\wedge H} \log(X_iA_i/p)}{H^3 T}}}$, we obtain the bound 
{\proofsize  \begin{align*}
       \sum_{t=1}^T {\rm BIAS}_{2, h}^{T, \ext} \le \mc{O}\paren{\sqrt{H^3 {H \choose K\wedge H } A_i^{K\wedge H} X_iT \iota}+ H {H \choose K\wedge H }A_i^{K\wedge H } X_i\iota},
\end{align*} }
where $\iota = \log(8X_iA_i/p)$ is a log factor.
\end{proof}